\documentclass[letterpaper,11pt]{article}
\usepackage[margin=1in]{geometry}  

\usepackage{bbm}
\usepackage{graphicx}
\usepackage{amsmath,amssymb,amsthm,amsfonts}

\usepackage{paralist}
\usepackage{bm}
\usepackage{xspace}
\usepackage{url}
\usepackage{prettyref}
\usepackage{boxedminipage}
\usepackage{wrapfig}
\usepackage{ifthen}
\usepackage{color}
\usepackage{xspace}

\usepackage{amsmath,amsthm,amsfonts,amssymb}
\usepackage{mathtools}
\usepackage{graphicx}

\usepackage{nicefrac}

\newtheorem*{definition*}{Definition}

\DeclareMathOperator*{\argmax}{arg\,max}
\DeclareMathOperator*{\argmin}{arg\,min}
\usepackage{subcaption}

\usepackage[utf8]{inputenc}

\usepackage{xcolor}
\definecolor{expert}{HTML}{008000}
\definecolor{error}{HTML}{f96565}

\usepackage{color-edits}
\addauthor{sw}{blue}

\usepackage{thmtools}
\usepackage{thm-restate}

\usepackage{tikz}
\usetikzlibrary{arrows,calc} 
\newcommand{\tikzAngleOfLine}{\tikz@AngleOfLine}
\def\tikz@AngleOfLine(#1)(#2)#3{%
\pgfmathanglebetweenpoints{%
\pgfpointanchor{#1}{center}}{%
\pgfpointanchor{#2}{center}}
\pgfmathsetmacro{#3}{\pgfmathresult}%
}

\declaretheoremstyle[
    headfont=\normalfont\bfseries, 
    bodyfont = \normalfont\itshape]{mystyle}

\usepackage[linesnumbered,algoruled,boxed,lined,noend]{algorithm2e}

\usepackage{listings}
\usepackage{amsmath}
\usepackage{amsthm}
\usepackage{tikz}
\usepackage{caption}
\usepackage{mdwmath}
\usepackage{multirow}
\usepackage{mdwtab}
\usepackage{eqparbox}
\usepackage{multicol}
\usepackage{amsfonts}
\usepackage{tikz}
\usepackage{multirow,bigstrut,threeparttable}
\usepackage{amsthm}
\usepackage{bbm}
\usepackage{epstopdf}
\usepackage{mdwmath}
\usepackage{mdwtab}
\usepackage{eqparbox}
\usetikzlibrary{topaths,calc}
\usepackage{latexsym}
\usepackage{cite}
\usepackage{amssymb}
\usepackage{bm}
\usepackage{amssymb}
\usepackage{graphicx}
\usepackage{mathrsfs}
\usepackage{epsfig}
\usepackage{psfrag}
\usepackage{setspace}
\usepackage[
            CJKbookmarks=true,
            bookmarksnumbered=true,
            bookmarksopen=true,
            colorlinks=true,
            citecolor=red,
            linkcolor=blue,
            anchorcolor=red,
            urlcolor=blue
            ]{hyperref}
\usepackage[linesnumbered,algoruled,boxed,lined]{algorithm2e}
\usepackage{algpseudocode}
\usepackage{stfloats}
\RequirePackage[numbers]{natbib}

\usepackage{comment}
\usepackage{mathtools}
\usepackage{blkarray}
\usepackage{multirow,bigdelim,dcolumn,booktabs}

\usepackage{xparse}
\usepackage{tikz}
\usetikzlibrary{calc}
\usetikzlibrary{decorations.pathreplacing,matrix,positioning}

\usepackage[T1]{fontenc}
\usepackage[utf8]{inputenc}
\usepackage{mathtools}
\usepackage{blkarray, bigstrut}
\usepackage{gauss}

\newcommand*{\BraceAmplitude}{0.4em}%
\newcommand*{\VerticalOffset}{0.5ex}%
\newcommand*{\HorizontalOffset}{0.0em}%
\newcommand*{\blocktextwid}{3.0cm}%
\NewDocumentCommand{\InsertLeftBrace}{%
	O{} 
	O{\HorizontalOffset,\VerticalOffset} 
	O{\blocktextwid} 
	m   
	m   
	m   
}{%
	\begin{tikzpicture}[overlay,remember picture]
	\coordinate (Brace Top)    at ($(#4.north) + (#2)$);
	\coordinate (Brace Bottom) at ($(#5.south) + (#2)$);
	\draw [decoration={brace, amplitude=\BraceAmplitude}, decorate, thick, draw=black, #1]
	(Brace Bottom) -- (Brace Top) 
	node [pos=0.5, anchor=east, align=left, text width=#3, color=black, xshift=\BraceAmplitude] {#6};
	\end{tikzpicture}%
}%
\NewDocumentCommand{\InsertRightBrace}{%
	O{} 
	O{\HorizontalOffset,\VerticalOffset} 
	O{\blocktextwid} 
	m   
	m   
	m   
}{%
	\begin{tikzpicture}[overlay,remember picture]
	\coordinate (Brace Top)    at ($(#4.north) + (#2)$);
	\coordinate (Brace Bottom) at ($(#5.south) + (#2)$);
	\draw [decoration={brace, amplitude=\BraceAmplitude}, decorate, thick, draw=black, #1]
	(Brace Top) -- (Brace Bottom) 
	node [pos=0.5, anchor=west, align=left, text width=#3, color=black, xshift=\BraceAmplitude] {#6};
	\end{tikzpicture}%
}%
\NewDocumentCommand{\InsertTopBrace}{%
	O{} 
	O{\HorizontalOffset,\VerticalOffset} 
	O{\blocktextwid} 
	m   
	m   
	m   
}{%
	\begin{tikzpicture}[overlay,remember picture]
	\coordinate (Brace Top)    at ($(#4.west) + (#2)$);
	\coordinate (Brace Bottom) at ($(#5.east) + (#2)$);
	\draw [decoration={brace, amplitude=\BraceAmplitude}, decorate, thick, draw=black, #1]
	(Brace Top) -- (Brace Bottom) 
	node [pos=0.5, anchor=south, align=left, text width=#3, color=black, xshift=\BraceAmplitude] {#6};
	\end{tikzpicture}%
}%

\usetikzlibrary{patterns}

\definecolor{cof}{RGB}{219,144,71}
\definecolor{pur}{RGB}{186,146,162}
\definecolor{greeo}{RGB}{91,173,69}
\definecolor{greet}{RGB}{52,111,72}

\def\1{\mathbbm{1}}

\newenvironment{keywords}
{\bgroup\leftskip 20pt\rightskip 20pt \small\noindent{\bfseries
Keywords:} \ignorespaces}%
{\par\egroup\vskip 0.25ex}
\newlength\aftertitskip     \newlength\beforetitskip
\newlength\interauthorskip  \newlength\aftermaketitskip

\usepackage{xspace}

\newcommand{\stepa}[1]{\overset{\rm (a)}{#1}}
\newcommand{\stepb}[1]{\overset{\rm (b)}{#1}}
\newcommand{\stepc}[1]{\overset{\rm (c)}{#1}}

\definecolor{myblue}{rgb}{.8, .8, 1}
\definecolor{mathblue}{rgb}{0.2472, 0.24, 0.6} 
\definecolor{mathred}{rgb}{0.6, 0.24, 0.442893}
\definecolor{mathyellow}{rgb}{0.6, 0.547014, 0.24}

\usepackage{dsfont}

\usepackage{tikz}
\usetikzlibrary{decorations.pathmorphing}
\usepackage{pgfplots}
\usetikzlibrary{shapes.misc}
\usetikzlibrary{arrows.meta}
\usetikzlibrary{patterns}

\usepackage[capitalize,noabbrev]{cleveref}
\crefname{lemma}{Lemma}{Lemmas}
\crefname{proposition}{Proposition}{Propositions}
\crefname{remark}{Remark}{Remarks}
\crefname{corollary}{Corollary}{Corollaries}
\crefname{definition}{Definition}{Definitions}
\crefname{conjecture}{Conjecture}{Conjectures}
\crefname{figure}{Fig.}{Figures}
\crefname{assumption}{Assumption}{Assumptions}

\theoremstyle{plain}
\newtheorem{theorem}{Theorem}[section]

\newtheorem{lemma}[theorem]{Lemma}
\newtheorem{corollary}[theorem]{Corollary}
\theoremstyle{definition}

\newtheorem{assumption}[theorem]{Assumption}
\theoremstyle{remark}
\newtheorem{remark}[theorem]{Remark}
\newtheorem{abstraction}{Abstraction}

\def\ECRepeatTheorems{%
\theoremstyle{THkey}%
\newtheorem{repeattheorem}{Theorem}
\newtheorem{repeatlemma}{Lemma}

\theoremstyle{EXkey}

}

\ECRepeatTheorems 

\newcommand {\Bcal} {{\mathcal{B}}}
\newcommand {\Ocal} {{\mathcal{O}}}
\newcommand {\Fcal} {{\mathcal{F}}}
\newcommand {\widetildeO} {{\widetilde{O}}}

\newcommand{\indic}{\mathds{1}}
\newcommand{\indt}{\mathrel{\text{\scalebox{1.07}{$\perp\mkern-10mu\perp$}}}}
\newcommand{\Tr}{\mathsf{tr}}
\newcommand{\Var}{\mathsf{Var}}

\newcommand{\btheta}{{\boldsymbol{\theta}}}
\newcommand{\bhtheta}{{\boldsymbol{\widehat\theta}}}
\newcommand{\bx}{{\boldsymbol{x}}}
\newcommand{\bA}{{\boldsymbol{A}}}
\newcommand{\bD}{{\boldsymbol{D}}}
\newcommand{\bZ}{{\boldsymbol{Z}}}
\newcommand{\bI}{{\boldsymbol{I}}}
\newcommand{\bV}{{\boldsymbol{V}}}
\newcommand{\bz}{{\boldsymbol{z}}}

\newcommand{\hr}{\widehat{r}}
\newcommand{\br}{\overline{r}}

\newcommand{\hb}{\widehat{b}}

\newcommand {\E} {{\mathbb{E}}}
\newcommand {\R} {{\mathbb{R}}}

\newcommand{\Reglte}{{R}}

\newcommand\numberthis{\addtocounter{equation}
{1}\tag{\theequation}}

\DeclarePairedDelimiter{\abs}{\lvert}{\rvert}
\DeclarePairedDelimiter{\nor}{\|}{\|}
\DeclarePairedDelimiter{\parr}{(}{)}
\DeclarePairedDelimiter{\parq}{[}{]}

\DeclarePairedDelimiter{\bra}{\lbrace}{\rbrace}

\begin{document}

\title{Joint Value Estimation and Bidding in Repeated First-Price Auctions}

\author{Yuxiao Wen, Yanjun Han, Zhengyuan Zhou\thanks{Yuxiao Wen is with the Courant Institute of Mathematical Sciences, New York University, email: \url{yuxiaowen@nyu.edu}. Yanjun Han is with the Courant Institute of Mathematical Sciences and the Center for Data Science, New York University, email: \url{yanjunhan@nyu.edu}. Zhengyuan Zhou is with the Stern School of Business, New York University, and Arena Technologies, email: \url{zz26@stern.nyu.edu}.}}

\maketitle

\begin{abstract}%
We study regret minimization in repeated first-price auctions (FPAs), where a bidder observes only the realized outcome after each auction---win or loss. This setup reflects practical scenarios in online display advertising where the actual value of an impression depends on the difference between two potential outcomes, such as clicks or conversion rates, when the auction is won versus lost. We incorporate causal inference into this framework and analyze the challenging case where only the treatment effect admits a simple dependence on observable features. Under this framework, we propose algorithms that jointly estimate private values and optimize bidding strategies under two different feedback types on the highest other bid (HOB): the full-information feedback where the HOB is always revealed, and the binary feedback where the bidder only observes the win-loss indicator. Under both cases, our algorithms are shown to achieve near-optimal regret bounds. Notably, our framework enjoys a unique feature that the treatments are actively chosen, and hence eliminates the need for the overlap condition commonly required in causal inference.
\end{abstract}


\section{Introduction}
\label{sec:intro}
Recent years have seen an industry-wide shift from second-price to first-price auctions (FPAs) in online advertising for display ads, beginning with platforms like OpenX, AppNexus, and Index Exchange \citep{Exchange}, and extending to Google AdX \citep{Google} and AdSense \citep{wong2021moving}, the world's largest ad exchanges. Unlike second-price auctions, FPAs are non-truthful, introducing critical challenges for real-time algorithmic bidding. One of these challenges is to learn the distribution of the highest other bids (HOBs) for developing the optimal bid shading strategy. To address this, \citep{balseiro2023contextual,han2020learning,han2024optimal} introduced the framework of regret minimization in repeated FPAs, developing algorithms that achieve near-optimal regrets under various feedback mechanisms. This framework has since been extended to accommodate many other practical considerations, such as latency \citep{zhang2021meow}, additional side information \citep{zhang2022leveraging}, budget constraints \citep{wang2023learning}, return on investment (ROI) constraints \citep{aggarwal2024no}, non-stationarity \citep{hu2025learning}, to name a few. We refer to a recent survey \citep{aggarwal2024auto} for an overview of these developments. 

\begin{figure*}[ht]
\begin{center}
\centerline{\includegraphics[width=0.8\columnwidth]{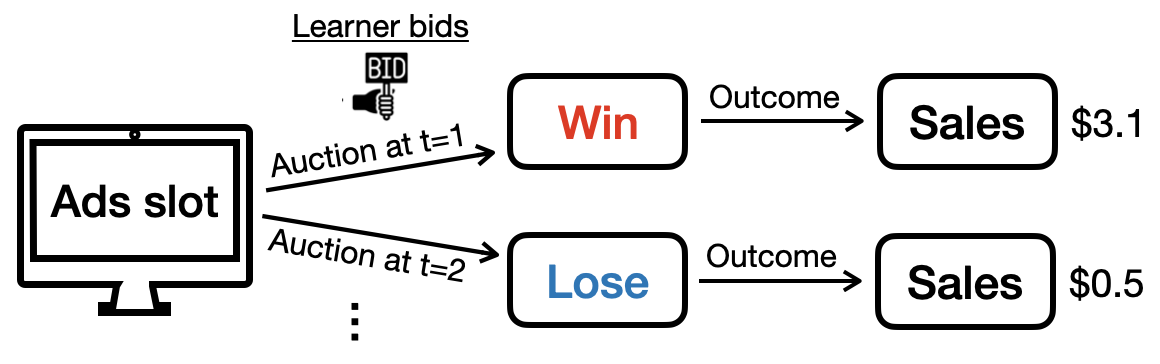}}
\vskip -0.1in
\caption{The learner repeatedly bids for advertisement slots and observes the (random) outcome sales after winning or losing the auctions.}
\label{fig:FPA}
\end{center}
\vskip -0.2in
\end{figure*}

Meanwhile, another important challenge in FPAs is the estimation of the actual value of an ad impression (i.e. advertising slot). In the above line of literature, this task of value estimation is assumed to be independent of bidding, so that a black-box access to these value estimates can be assumed. In practice, however, this value depends on bidding outcomes as it is the difference between two potential outcomes, commonly clicks or conversion rates, when the auction is won versus lost. As shown in \cref{fig:FPA}, an advertiser can observe the outcome of winning the advertising slot when he/she wins the auction, and the outcome of doing nothing when losing the auction; the value of the advertising slot is the difference between these two potential outcomes. Consequently, as \citep{waisman2024online} observes, the problem of value estimation can be framed as a specific instance of causal inference, where the target quantity is the treatment effect, and the treatment is determined by the bidding outcome. This introduces the following challenges:
\begin{itemize}
    \item First, as in all causal inference problems, only one potential outcome is observable for each subject, making direct estimation of the treatment effect inherently difficult. 
    \item Second, unique to bidding in FPAs, the treatment (winning or losing the auction) is determined by the bidding strategy, which, in turn, depends on past outcomes and the current estimate of the actual value. This interdependence complicates both the estimation and the bidding process, and this is a structure overlooked by the existing literature. 
\end{itemize}

To address the aforementioned challenges, practitioners often use randomized controlled experiments for value estimation, or pretend that value estimation and bidding were independent processes so as to decouple these tasks. However, as \citep{waisman2024online} shows, the first approach can be very cost inefficient, while the latter approach may introduce significant biases (e.g. significantly overestimate/underestimate some values) into the final bidding outcome. 

In this paper, we provide a first systematic study for joint value estimation and bidding in FPAs, by formulating a specific causal inference problem and applying a bandit framework to it. Specifically, we make the following contributions:
\begin{itemize}
    \item \textbf{Joint value estimation and bidding:} We formulate the problem of joint value estimation and bidding (henceforth abbreviated as JVEB) in repeated FPAs, by incorporating a causal inference problem into the bandit learning framework. We represent the value as the difference between two potential outcomes observed upon winning and losing, respectively, and introduce a realistic model where the treatment effect takes a simpler form compared to the potential outcomes.  

    \item \textbf{General framework of HOB estimation:} We introduce a flexible black-box framework for HOB estimation that captures a broad class of estimation procedures and supports final bidding guarantees derived solely from this abstraction. This framework cleanly separates HOB estimation from the core tasks of JVEB, and it accommodates diverse feedback settings, including full-information HOB feedback, where HOB estimation is inherently decoupled, and binary HOB feedback, where the remaining interdependence can be carefully managed.

    \item \textbf{Regret guarantees:} Within our JVEB framework, we design a bandit learning algorithm that attains regret $\widetilde{O}(\sqrt{\Delta dT})$ against an oracle who makes optimal dynamic bids. Here $d$ is the feature dimension, and $\Delta$ is a parameter arising from our black-box HOB estimation framework. Under full-information HOB feedback, this bound specializes to $\widetilde{O}(\sqrt{dT})$ regret under i.i.d. HOBs and $\widetilde{O}(d\sqrt{T})$ regret for linear contextual HOBs. Under binary HOB feedback, it yields $\widetilde{O}(d^{\frac{1}{3}}T^{\frac{2}{3}})$ regret. 

    \item \textbf{Algorithmic insights:} We provide several core algorithmic insights for the JVEB problem. For value estimation, building on classical causal inference tools such as the inverse propensity weighting (IPW) estimator, we employ a carefully designed weighted least squares (WLS) procedure to tightly control estimation variance. For bidding, beyond standard bandit techniques like arm elimination, we introduce a new decision rule, the ``better of two UCBs,'' which plays a central role in reducing variance and achieving a small regret. For the joint task, we develop techniques to handle potential interdependencies among value estimation, HOB estimation, and bidding. The resulting algorithm combines causal inference and bandit learning in a deliberate way that avoids costly randomized experiments and unrealistic overlap assumptions.
\end{itemize}

\subsection{Related Work}

\paragraph{Bidding in FPAs.} FPAs, also known as descending-price or Dutch auctions, have a long history. Early study of FPAs typically took a game-theoretic view dating back to the foundational work of \citep{vickrey1961counterspeculation, myerson1981optimal}, and there has been significant progress in understanding the computational complexity of Bayesian Nash equilibrium for online FPAs \citep{wang2020bayesian,filos2021complexity,bichler2021learning,bichler2023computing,chen2023complexity,filos2024computation}. A different line of study takes the perspective of online learning where a single bidder would like to learn the ``environment'', mainly the behaviors of other bidders in the literature. Dating back to \citep{kleinberg2003value} in a different pricing context, this perspective has inspired a flurry of work \citep{han2020learning,zhang2022leveraging,balseiro2019learning,balseiro2023contextual,badanidiyuru2023learning,wang2023learning,han2024optimal,aggarwal2024no, hu2025learning, han2026semi} which establishes sublinear regrets under various feedback structures and environments in FPAs; we refer to the survey \citep{aggarwal2024auto} for an overview. However, this line of work assumes a known valuation and mainly focuses on learning HOBs or handling other constraints (budget, side information, etc), and the problem of joint value estimation and bidding is explicitly listed as an open direction in \citep{han2024optimal}.

\paragraph{Unknown valuation.} There is also a rich line of literature on learning unknown valuations in auctions. In second-price auctions, \citep{weed2016online} studies a regret minimization setting where the bidder observes the valuation only if he/she wins the auction; this setting is then generalized to other auction formats \citep{feng2018learning} including FPAs \citep{achddou2021fast}. A recent work \citep{cesa2024role} provides a systematic study of feedback mechanisms (termed as \emph{transparency}) and unknown valuations in FPAs and provides tight regret bounds in various settings. However, all these studies do \emph{not} view the valuation of a bidder as a treatment effect, while an empirical study \citep{waisman2024online} stresses the importance of a causal inference modeling. Although \citep{waisman2024online} does not directly study the regret minimization problem, they identify a unique challenge in FPAs that the goals of regret minimization and causal inference are not well-aligned (as opposed to second-price auctions where these tasks are aligned). This challenge precisely motivates the problem formulation of joint value estimation and bidding in the current work. 

We provide a more detailed comparison with \citep{cesa2024role} as it is closest to ours. First, an obvious difference is that while the unknown valuation is the potential outcome of winning in \citep{cesa2024role}, in our work it is the difference between two potential outcomes. Our scenario is more challenging for two reasons: 1) the feedback structure of the valuation is more involved, as the bidder needs both wins and losses to observe both sides of the coin; 2) because there are \emph{two} potential outcomes in our model, the problem of \emph{directly} estimating the treatment effect without strong assumptions on the potential outcomes is technically more challenging and draws connections to the causal inference literature. Second, compared with \citep{cesa2024role}, our model involves features and thus competes with a stronger oracle than the best fixed bid. In addition, thanks to the presence of features, we may adapt the standard unconfoundedness assumption in causal inference and assume a conditional independence between valuations and HOBs given the features (cf. \cref{ass:noconfounding}). By contrast, \citep{cesa2024role} assumes a general dependence structure without features, a setting not directly comparable to ours. 

\paragraph{Causal inference.} The causal inference formulation of value estimation in FPAs in \citep{waisman2024online} assumes the unconfoundedness and overlap conditions; both conditions are standard in causal inference \citep{imbens2015causal}, and are essential for controlling the estimation error in policy learning \citep{athey2021policy,zhou2023offline}. There have also been some efforts in relaxing or removing the overlap condition, such as trimming \citep{yang2018asymptotic,branson2023causal}, policy shifts \citep{zhao2024positivity}, and bypassing the uniform requirements \citep{jin2025policy}. For example, the policy learning result in \citep{jin2025policy} only involves the overlap condition for the selected policy; by contrast, thanks to the special nature of the joint value estimation and bidding problem, we do not even require the overlap condition for the selected bid.

We also mention that assuming the treatment effect to be more structured and simpler than potential outcomes is a prevalent assumption in causal inference \citep{robins1994correcting,van2006statistical}. For CATE estimation, this structure also proves to be challenging and attracts quite a few theoretical interests, e.g. \citep{chernozhukov2018double,cui2020value,gao2020minimax,nie2021quasi,foster2023orthogonal,kennedy2023towards,kennedy2024minimax}. This work also joins this line of literature and assumes only a linear model for the treatment effect in online bidding problem.

\subsection{Notations}
\label{sec:notation}
For positive integers $n\in\mathbb{N}$, we let $[n] = \{1,2,\dots,n\}$. For an event $E$, we defined the indicator function $\indic[E]$ to be 1 if the event occurs and 0 otherwise. The standard asymptotic notations $O(\cdot), \Omega(\cdot),$ and $\Theta(\cdot)$ are used to suppress constant factors, and $\widetilde{O}(\cdot), \widetilde\Omega(\cdot),$ and $\widetilde\Theta(\cdot)$ to suppress poly-logarithmic factors. For a random variable $X$, $\E[X]$ and $\Var(X)$ denote its mean and variance respectively. Throughout, we use the filtration $\Fcal_t$ to denote all available history up to the beginning of round $t$, and write $\E_t[X] = \E[X|\Fcal_t]$ for simplicity. For a vector $x\in \R^d$ and a PSD matrix $A\in \R^{d\times d}$, let $\|x\|_A = \sqrt{x^\top Ax}$. For a cumulative distribution function $G$ on $[0,1]$, let its inverse be $G^{-1}(z) = \inf\bra{x\in[0,1]: G(x)\ge z}$.

\subsection{Organization}
The rest of the paper is organized as follows. We first formally state our problem and present the main results in Section \ref{sec:formulation_and_results}. Then Section \ref{sec:linear_te} focuses on the full-information HOB feedback and introduces the core ideas of variance reduction. Section \ref{sec:binary_hob} extends our bidding algorithms and ideas to the binary HOB feedback with more involved analysis, due to the fundamentally more difficult feedback. Numerical experiments are provided in Section \ref{sec:numericals} to validate our theoretical findings, and Section \ref{sec:conclusion} concludes this paper. The technical proofs and extended discussions are deferred to the appendices.

\section{Problem Formulation and Main Results}
\label{sec:formulation_and_results}

In this section, we formulate the JVEB problem in FPAs. \Cref{sec:formulation} lays out the causal inference and bandit learning components that define JVEB. \Cref{sec:HOB} introduces a flexible black-box framework for HOB estimation, a necessary ingredient but not the central challenge of JVEB. \Cref{sec:main_results_full} presents the regret guarantees under the various forms of HOB feedback. These results are summarized in \cref{tab:regrets}.

\begin{table}%
	\caption{Minimax Regret Characterizations under Two HOB Feedbacks}
	\label{tab:regrets}
	\begin{minipage}{\columnwidth}
		\begin{center}
			\begin{tabular}{cccc}
				\toprule
				   & Plug-and-play bound & I.i.d. HOBs & Linear HOBs\\
                   \midrule
				Full-information HOB  & $\widetildeO(\sqrt{\Delta dT})$ & $\widetildeO(\sqrt{dT})$ & $\widetildeO(d\sqrt{T})$\\
                \midrule
			     Binary HOB  & $\widetildeO(\sqrt{\Delta dT} + \xi T)$ & $\widetildeO(d^{\frac{1}{3}}T^{\frac{2}{3}})$ & \\
				\bottomrule
			\end{tabular}
		\end{center}
		\bigskip\centering
		\footnotesize\emph{Note:} The quantities $\Delta$ and $\xi$ concern the performance of the HOB estimation in \cref{ass:est_oracle_bern} and \ref{ass:est_oracle_weaker_bern}. 
	\end{minipage}
\end{table}%

\subsection{JVEB Formulation}
\label{sec:formulation}

The JVEB problem considers a single bidder that jointly estimates her private values and bids against a population of other bidders in repeated FPAs over a time horizon $T$. At the beginning of each time $t\in [T]$, the bidder observes contextual information $\bx_t\in\R^d$. The bidder then submits a bid $b_t\in \R_+$, while the highest other bid (HOB) in that round is denoted by $M_t\in \R_+$. If $b_t\ge M_t$, the bidder wins the auction, and receives an outcome value $v_{t,1}\in \mathbb{R}$, representing the potential outcome associated with winning. If $b_t < M_t$, the bidder loses the auction, and receives an outcome value $v_{t,0}\in \mathbb{R}$ representing the potential outcome for losing. These potential outcomes may, for example, correspond to the final sales generated when an ad is shown versus when it is not. 

The final realized payoff for the learner is
\begin{equation}\label{eq:payoff_formulation}
r_t(b_t) = \indic[b_t\geq M_t](v_{t,1}-b_t) + \indic[b_t<M_t]v_{t,0} 
=\indic[b_t\geq M_t](v_{t,1}-v_{t,0}-b_t) + v_{t,0}. 
\end{equation}
Throughout, we adopt the standard scaling assumption that $v_{t,1}, v_{t,0}, b_t, M_t\in [0,1]$, and $\|\bx_t\|_2\le 1$. Since the dependence of the maximizer of \eqref{eq:payoff_formulation} on the potential outcomes $(v_{t,1}, v_{t,0})$ only factors through the \emph{treatment effect} $v_{t,1}-v_{t,0}$, we only need to model the treatment effect. In this paper, we adopt a simple linear model for the treatment effects:
\begin{assumption}[Linear model]\label{assump:linear}
$\E[v_{t,1}- v_{t,0} \mid \bx_t] = \btheta_*^{\top}\bx_t$ for unknown $\btheta_*\in\R^d$ with $\|\btheta_*\|_2\leq 1$. 
\end{assumption}

Despite its simplicity, a linear model is a meaningful and practical starting point that captures the key algorithmic ideas. More importantly, \Cref{assump:linear} imposes a linear structure \emph{only} on the treatment effects, while the potential outcomes $(v_{t,1}, v_{t,0})$ may follow far more complex distributions. This aligns with standard practice in causal inference, where treatment effects often admit cleaner structure than the underlying potential outcomes \citep{kennedy2024minimax}. 

In addition to distributional properties, the dependence structure among different random variables is of utmost importance in causal inference. The next two assumptions concern the dependence between the treatment effect $v_{t,1}-v_{t,0}$ and the HOB $M_t$, and across different time $t$. 
\begin{assumption}[Unconfoundedness]\label{ass:noconfounding}
It holds that $v_{t,1}- v_{t,0} \indt M_t \mid \bx_t$. 
\end{assumption}

\begin{assumption}[Conditional independence]\label{ass:context}
Conditioned on the contexts $\{\bx_t\}_{t=1}^T$, the random vectors $\{(v_{t,1}, v_{t,0}, M_t)\}_{t=1}^T$ are conditionally independent over time. 
\end{assumption}

\cref{ass:noconfounding} mimics the standard nonconfoundedness assumption in causal inference \citep{imbens2015causal}. It only requires the conditional independence on the treatment effect but not on the potential outcomes, hence weaker than the counterpart adopted in \citep[Assumption 1]{waisman2024online} which pioneers the study of JVEB in auctions. Intuitively, the context $\bx_t$ should encode all necessary information for a given item to both the bidder and competitors. This assumption is also somewhat necessary: even if $v_{t,0}=0$ and $\bx_t=\varnothing$, dependence between $v_{t,1}$ and $M_t$ could lead to an $\Omega(T)$ regret (cf. \citep[Theorem 8]{han2024optimal} and \citep[Theorem 8]{cesa2024role}). 

\cref{ass:context} is a standard independence assumption over time, stating that that both $(v_{t,1},v_{t,0})$ and $M_t$ depend only on the current context $\bx_t$. While one could argue that the HOB $M_t$ should also depend on past history, modeling such dependence is typically intractable in practice. As a practical alternative, prior work on auctions, including \citep{badanidiyuru2023learning, han2024optimal}, adopts this independence assumption, and we follow the same approach in this paper.

Having set up the causal inference framework, we now turn to the bandit-learning component. Let $G_t$ denote the time-varying cumulative distribution function (CDF) of $M_t$ conditioned on the context $\bx_t$, i.e., $G_t(b) = \mathbb{P}(M_t\le b|\bx_t)$. Then the expected payoff for submitting a bid $b_t$ is
\begin{equation}\label{eq:expected_payoff_formulation}
\br_t(b_t) \coloneqq \E[r_t(b_t)] 
= G_t(b_t)(\btheta_*^\top \bx_t - b_t) + \E[v_{t,0}].
\end{equation}
We postpone the modeling of $G_t$ and its estimation, as well as the feedback models for the HOB, to the next section. Finally, we can adopt a bandit learning framework and define the bidder's regret by competing against a powerful dynamic oracle:
\begin{equation}\label{eq:reg_def_lte}
\Reglte(\pi) = \sup_{\{G_t,\bx_t\}_{t=1}^T, \btheta_{*}} \sum_{t=1}^T \parr*{\br_t\parr*{b_t^*} - \E[\br_t(b_t)]}
\end{equation}
where the bidder's policy $\pi=(b_t)_{t=1}^T$ is a sequence of bids, and $b_t^*$ is the optimal bid of the oracle who has full knowledge of the true parameters $\btheta_{*}$ and the CDF $G_t$:
\begin{equation}\label{eq:optimal-bid}
b_t^*  = \argmax_{b\in[0,1]}\br_t(b) = \argmax_{b\in[0,1]}G_t(b)\parr*{\btheta_{*}^{\top}\bx_t -b}.
\end{equation}
The bidder's target is to devise a bidding policy $\pi$ with the smallest possible regret $R(\pi)$, ideally satisfying $R(\pi) = o(T)$, indicating that the bidder asymptotically matches the oracle’s performance.

\subsection{HOB Modeling and Estimation}\label{sec:HOB}
The JVEB is already a nontrivial problem even when the HOBs (that is, $G_t$) are perfectly known to the bidder, since the bidder must still submit bids based on value estimates and update those estimates using bidding outcomes. In practice, however, the HOBs are never known exactly, so it remains necessary to model this component and, ideally, to propose a flexible framework for estimating the HOBs that can be applied to the JVEB in a black-box manner.

Throughout this paper, we make the following mild assumption on the CDF $G_t$. 

\begin{assumption}[Bounded density]\label{ass:Gt}
The CDF $G_t$ is $L$-Lipschitz, i.e., it admits a density $g_t$ with $\|g_t\|_\infty\le L$. 
\end{assumption}

The performance of HOB estimation depends on additional structural properties of $G_t$, as well as the HOB feedback model. In practice, there are two common types of HOB feedback: 
\begin{itemize}
    \item \textit{Full-information}: the HOB $M_t$ is always revealed to the bidder at the end of time $t$ \citep{han2020learning}. This is called the ``minimum-bid-to-win'' in the Google Ad Exchange.

    \item \textit{Binary}: the bidder only observes the auction outcome, i.e., the win-loss indicator $\indic[M_t\le b_t]$, at the end of time $t$ \citep{cesa2024role}. 
\end{itemize}
Both feedback models will be studied in this paper. For each HOB feedback model, rather than imposing explicit structural assumptions on $G_t$, we will adopt a flexible black-box framework and express the required HOB estimation guarantees in terms of estimation abstractions.

\subsubsection{Black-box Abstraction under Full-information Feedback.}\label{sec:blackbox_abstraction_fullinfo}
Under full-information HOB feedback, we state the HOB estimation guarantee in terms of the following black-box abstraction.  
\begin{abstraction}[Estimation oracle I]\label{ass:est_oracle_bern}
Based on past observations $\{\bx_s, M_s\}_{s<t}$, there exists a valid CDF estimator $\widehat{G}_t$ such that with probability at least $1-T^{-2}$, 
\begin{equation*}
\abs*{\widehat{G}_t(b) - G_t(b)} \le \delta_t\sqrt{G_t(b)(1-G_t(b))} + \delta_t^2
\end{equation*}
for some known sequence $\{\delta_t\}_{t\in [T]}$, every $b\in[0,1]$ and $t\in[T]$. 
\end{abstraction}
In other words, rather than imposing explicit structural assumptions on $G_t$ and constructing an explicit estimator $\widehat{G}_t$, \Cref{ass:est_oracle_bern} provides a black-box abstraction that can later be applied to our JVEB problem. For specific HOB models, \Cref{ass:est_oracle_bern} still needs to be verified with explicit oracle implementation; several examples are given in \Cref{sec:main_results_full}. Mathematically, \Cref{ass:est_oracle_bern} states a Bernstein-type concentration guarantee in which the estimation error becomes smaller when $G_t(b)$ is close to $0$ or $1$, a structure that proves useful in our final regret guarantee.

We also explain why such an abstraction is reasonable under full-information feedback. In fact, the HOB estimation problem is \emph{completely decoupled from} JVEB: regardless of the bidder’s strategy for JVEB, the observed trajectory $\{\bx_t, M_t\}_{t=1}^T$ remains unchanged, and this trajectory is the only information relevant for estimating the HOBs. Given this decoupled structure, it is natural to treat HOB estimation separately from JVEB and to introduce an abstraction for this component.


\subsubsection{Black-box Abstraction under Binary Feedback.} 

Under binary HOB feedback, we apply a slightly different abstraction for HOB estimation. 
\begin{abstraction}[Estimation oracle II]\label{ass:est_oracle_weaker_bern}
Based on past observations $\{\bx_s, b_s, \indic[M_s\le b_s]\}_{s<t}$, there exists a valid CDF estimator $\widehat{G}_t$ such that with probability at least $1-T^{-2}$, 
\begin{equation*}
\abs*{\widehat{G}_t(b) - G_t(b)} \le \delta_t(b; b^{t-1}) \sqrt{G_t(b)(1-G_t(b))} + \delta_t(b; b^{t-1})^2 + \xi
\end{equation*}
holds for every $b\in [0,1]$, every $t\in[T]$, some known functions $\{\delta_t(\cdot)\}_{t\in [T]}$ and bias parameter $\xi\ge 0$. 

\end{abstraction}

The above abstraction highlights several generalizations. First, the error parameter $\delta_t(b;b^{t-1})$ of $\widehat{G}_t(b)$ now depends on the bid $b$ and the historic bidding trajectory. It reflects the nature under binary HOB feedback that past bids $b^{t-1}$ affect the observations that are used to estimate the HOB distribution $G_t$. Second, while the upper bound in \Cref{ass:est_oracle_bern} mainly represents the stochastic term in Bernstein's inequality, in \Cref{ass:est_oracle_weaker_bern} we allow an additional bias term $\xi\ge 0$ that arises naturally from the bias-variance trade-off in nonparametric regressions with binary outcomes. Finally, we note that unlike the full-information feedback where the HOB estimation problem is completely decoupled from JVEB, the bidding strategy affects the HOB estimation under binary feedback. However, our master algorithm in \Cref{sec:binary_hob} still manages to decouple this dependence, so that the overall bidding algorithm works under this abstraction. Similar to \Cref{ass:est_oracle_bern}, we also realize an estimation oracle satisfying \Cref{ass:est_oracle_weaker_bern} in a real example in \Cref{sec:main_results_binary}.

\subsection{Main Results}
We are now ready to present our main regret bounds under both abstractions, and show how these HOB estimation oracles are implemented in specific examples. 

\subsubsection{Full-information HOB Feedback.}\label{sec:main_results_full}
As discussed above, when the HOB is always revealed, the problems of JVEB and HOB estimation are completely decoupled. Given an estimation oracle satisfying \cref{ass:est_oracle_bern} and the knowledge of $\delta_t$ at time $t$, we develop a bidding algorithm for JVEB with the following regret guarantee. 



\begin{theorem}[Upper bound I]\label{thm:main-lte}
Suppose the bidder implements Abstraction \ref{ass:est_oracle_bern}. Under Assumptions \ref{assump:linear}--\ref{ass:Gt}, there is an algorithm $\pi_{\mathrm{lte}}$ (with the knowledge of the confidence bound $\delta_t$ at time $t$ and the knowledge of $L$) that achieves the expected regret
\begin{equation*}
\Reglte(\pi_{\mathrm{lte}}) = \widetilde{O}\parr*{\sqrt{\Delta dT}}, 
\end{equation*}
where $\Delta := 1 + \sum_{t=1}^T \delta_t^2$. 
\end{theorem}

At a high level, \cref{thm:main-lte} is a ``plug-and-play'' bound which translates any HOB estimation guarantee in the form of \cref{ass:est_oracle_bern} to a final regret guarantee. We list a few examples of estimation oracles with specific values of $\Delta$ and final regret guarantee. 

\paragraph{I.i.d. HOBs.} When $M_t$ is i.i.d., Bernstein's concentration (\cref{lem:bernstein}) shows that \cref{ass:est_oracle_bern} holds for the empirical CDF with $\delta_t=\widetilde{O}(1/\sqrt{t})$. Then $\Delta=\widetilde{O}(1)$, and \cref{thm:main-lte} achieves a regret upper bound of $\widetilde{O}(\sqrt{dT})$ for i.i.d. HOBs. This regret is near-optimal in view of the lower bound in \Cref{thm:lower-bound} below. 

\paragraph{Linear HOBs.} In the linear model $M_t = \varphi_*^\top \bx_t + \eta_t$ where $M_t$ depends linearly on $\bx_t$ \citep{badanidiyuru2023learning}, we assume that $\varphi_*\in\R^d$ is an unknown parameter with $\|\varphi_*\|_2\le 1$ and impose the following assumptions on the i.i.d. noise $\eta_t$ with CDF $Q$: 
\begin{assumption}\label{ass:flat_noise_tail}
There exist constants $\sigma, a_0, a_1, a_2>0$, such that the following holds:
\begin{enumerate}
    \item[1)] $\eta_t$ is zero-mean and $\sigma^2$-sub-Gaussian;

    \item[2)] If $Q(v)\in [0,a_0]\cup[1-a_0,1]$, then $Q'(v)\le a_1\sqrt{Q(v)(1-Q(v))}$ for all $v\in \R$;

    \item[3)] The second derivative is bounded, i.e. $|Q''(v)| \le a_2$ for all $v\in \R$. 
\end{enumerate}
\end{assumption}

Note that \cref{ass:flat_noise_tail} is weaker than \cite{badanidiyuru2023learning} and does not require the log-concavity there, and is satisfied for Gaussian or truncated Gaussian $\eta_t$ (by the well-known Mills ratio results \cite[Eq (9)]{gordon1941values}). In this case the oracle in \cref{ass:est_oracle_bern} also exists: 

\begin{lemma}[Linear HOB Estimation]\label{lem:linear_hob_error}
Let Assumptions \ref{ass:noconfounding}--\ref{ass:flat_noise_tail} hold. Under full-information HOB feedback, there exists a CDF estimator $\widehat{G}_t$ that satisfies \cref{ass:est_oracle_bern} with the HOB error parameter $\Delta=\widetildeO(d)$.
\end{lemma}

Therefore, \cref{thm:main-lte} achieves a regret upper bound of $\widetilde{O}(d\sqrt{T})$ when the HOBs admit such a linear model. This recovers the regret upper bound of \cite{badanidiyuru2023learning} under weaker assumptions. 

\paragraph{Known HOBs.} When $G_t$ is perfectly known, we have $\delta_t\equiv 0$ and thus $\Delta = 1$. In this case, \Cref{thm:main-lte} shows an upper bound of $\widetilde{O}(\sqrt{dT})$, which is near-optimal by the following matching lower bound. 
\begin{theorem}[Lower bound]\label{thm:lower-bound}
Under Assumptions \ref{assump:linear}--\ref{ass:Gt}, even if $G_t$ is perfectly known, the following lower bounds hold: when $T\ge d^2$,
\begin{equation*}
\inf_\pi \Reglte(\pi) = \Omega(\sqrt{dT}).
\end{equation*}
\end{theorem}

\subsubsection{Main Results under Binary HOB Feedback.}
\label{sec:main_results_binary}

Under the binary HOB feedback, the HOB estimation is no longer a standalone problem, and we face a more difficult ``joint value estimation, HOB estimation, and bidding'' problem. Nevertheless, we show that the estimation of $G_t$ can still be abstracted from the JVEB problem in the form of \cref{ass:est_oracle_weaker_bern}, and establish the following regret guarantee. 

\begin{theorem}[Upper bound II]\label{thm:main-lte-binary}
Suppose there is some $T_0$ such that for every $t\ge T_0$, the bidder implements Abstraction \ref{ass:est_oracle_weaker_bern} with $\|\delta_t\|_\infty \le c$ for a small constant $c$ depending only on $L$.
Under Assumptions \ref{assump:linear}--\ref{ass:Gt}, there is an algorithm $\pi_{\mathrm{bin}}$ that achieves the expected regret
\begin{equation*}
\Reglte(\pi_{\mathrm{bin}}) = \widetilde{O}\parr*{T_0 + \sqrt{\Delta dT} + \xi T}, 
\end{equation*}
where $\Delta = 1+\sup_{b^T}\sum_{t=1}^T\delta_t(b_t;b^{t-1})^2$. 
\end{theorem}

Once again, \Cref{thm:main-lte-binary} establishes a ``plug-and-play'' regret bound. Compared with \Cref{thm:main-lte}, the new regret bound involves a burn-in period $T_0$, a worst-case definition of $\Delta$ over all possible bid trajectories $b^T$, and the additional bias parameter $\xi$ in \Cref{ass:est_oracle_weaker_bern}, due to the additional complication caused by binary HOB feedback. We give an example of how \Cref{thm:main-lte-binary} is applied. 

\paragraph{I.i.d. HOBs.} For i.i.d. HOBs, we develop an implementation of \cref{ass:est_oracle_weaker_bern} that leads to following regret guarantee: 


\begin{theorem}[Upper bound III]\label{thm:main-lte-binary-iid}
Consider the binary HOB feedback and i.i.d. HOBs. Under Assumptions \ref{assump:linear}--\ref{ass:Gt}, there is an oracle implementation of \cref{ass:est_oracle_weaker_bern} with 
\begin{equation*}
T_0 = \widetilde{O}(T^{2/3}), \quad \Delta = \widetilde{O}(d^{-1/3}T^{1/3}), \quad \xi = \widetilde{O}(d^{1/3}T^{-1/3}). 
\end{equation*}
Consequently, there is an algorithm $\pi_{\mathrm{bin}}$ achieving the expected regret 
\begin{equation*}
\Reglte(\pi_{\mathrm{bin}}) = \widetilde{O}\parr*{d^{\frac{1}{3}}T^{\frac{2}{3}}}. 
\end{equation*}
\end{theorem}
Even in the special case with no contexts, $v_{t,0}\equiv 0$, and $v_{t,1}\equiv v$, a lower bound $\Omega(T^{\frac{2}{3}})$ was shown in \cite{balseiro2023contextual} under the binary HOB feedback. It is rather straightforward to develop a complementary lower bound $\Omega(d^{\frac{1}{3}}T^{\frac{2}{3}})$ by splitting the horizon into $d$ independent sub-problems and embedding a lower bound instance for each sub-problem (c.f. Appendix \ref{app:lin_lower_bound}). Therefore, our black-box abstraction still achieves an optimal regret.

\section{Full-information HOB Feedback}
\label{sec:linear_te}
In this section, we introduce our bidding algorithm that achieves the regret upper bound in \cref{thm:main-lte}, assuming access to an estimation oracle satisfying \cref{ass:est_oracle_bern}. Throughout the algorithms in this section, we use $\Bcal = \bra*{\frac{j}{\sqrt{T}}: j\in[\lfloor\sqrt{T}\rfloor]}$ to denote a discretization of the continuous bid space $[0,1]$ to avoid potentially inefficient optimization problems; this discretization only incurs an additional regret of $O(\sqrt{T})$.

\subsection{Truncated IPW Estimator}\label{sec:truncated_IPW}
In order to estimate the unknown parameter $\btheta_*$ in the linear model, the first question is to find a sound estimator for the mean treatment effect $\E[v_{t,1}-v_{t,0}]$ without modeling the potential outcomes $v_{t,1}$ and $v_{t,0}$. When bidding price $b$ at time $t$, if the HOB distribution $G_t$ \emph{were} perfectly known, a natural estimator would be the inverse propensity weighting (IPW) estimator:
\begin{equation}\label{eq:weighted_estimator_te}
\widehat{e}_t(b) = \frac{\indic[b\geq M_t]v_{t,1}}{G_t(b)} - \frac{\indic[b< M_t]v_{t,0}}{1-G_t(b)}. 
\end{equation}
However, since we only have access to an estimator $\widehat{G}_t$ of $G_t$, we will apply a truncated IPW estimator defined as
\begin{equation}\label{eq:truncated_estimator_te}
\widetilde{e}_t(b) = \frac{\indic[b\geq M_t]}{\max\bra*{\delta_t^2 , \widehat{G}_t(b)}}v_{t,1} - \frac{\indic[b< M_t]}{\max\bra*{\delta_t^2, 1-\widehat{G}_t(b)}}v_{t,0},
\end{equation}
where $\delta_t$ is the concentration parameter in \cref{ass:est_oracle_bern} which is assumed to be known. Unlike $\widehat{e}_t(b)$, the truncated estimator $\widetilde{e}_t(b)$ is biased. The following lemma summarizes the bias-variance properties of $\widetilde{e}_t(b)$. 
\begin{lemma}\label{lem:bias-variance}
Suppose $|\widehat{G}_t(b) - G_t(b)|\le \delta_t\sqrt{G_t(b)(1-G_t(b))} + \delta_t^2$. Then there exist absolute constants $c, c'>0$ such that
\begin{equation}\label{eq:sigma_t}
\begin{cases}
    \abs*{\E[\widetilde{e}_t(b)] - \btheta_*^\top \bx_t} \le c\delta_t\sigma_t(b) \\
    \Var(\widetilde{e}_t(b)) \le c'\sigma_t(b)^2
\end{cases} \qquad \text{where }\sigma_t(b) \coloneqq \parr*{ \widehat{G}_t(b)(1-\widehat{G}_t(b))}^{-\frac{1}{2}}.
\end{equation}
\end{lemma}

\cref{lem:bias-variance} essentially says that the quantity $\widetilde{e}_t(b_t)$, computable at the end of time $t$, approximately satisfies a linear model
\begin{equation}\label{eq:linear-model}
\widetilde{e}_t(b_t) = \btheta_*^\top \bx_t + \varepsilon_t, \quad \text{with } \E[\varepsilon_t] \approx 0, \Var(\epsilon_t) \le c'\sigma_t(b_t)^2. 
\end{equation}
Henceforth, we will use $\sigma_t(b_t)^2$ in \eqref{eq:sigma_t} as the variance proxy for the quantity $\widetilde{e}_t(b_t)$. Since $\widehat{G}_t(b_t)$ could be close to $0$ or $1$ (i.e. no overlap condition), the quantity $\sigma_t(b_t)$ can be prohibitively large. 

\subsection{Base Algorithm}
\label{sec:base_lin_te}

\begin{algorithm}[!t]\caption{LIN-TE (Base algorithm for linear treatment effects)}
\label{alg:base_alg_te}
\textbf{Input:} Time indices $\Phi_t\subseteq [t-1]$, subspace $B_t\subseteq \Bcal$, estimated CDF $\widehat{G}_t$, time horizon $T$.

$\sigma_\tau \gets \sigma_\tau(b_\tau)$ in \eqref{eq:sigma_t} for $\tau\in \Phi_t$;

$\bA_{t}\leftarrow \bI + \sum_{\tau\in \Phi_t} \sigma_\tau^{-2}\bx_\tau \bx_\tau^{\top}$;

$\bz_{t} \leftarrow \sum_{\tau\in\Phi_t}\sigma_\tau^{-2}\bx_\tau \widetilde{e}_\tau$ where $\widetilde{e}_\tau = \widetilde{e}_\tau(b_\tau)$ is defined in \eqref{eq:truncated_estimator_te};

Compute estimator $\bhtheta_{t} \leftarrow \bA_{t}^{-1}\bz_{t}$.

Set $\gamma\gets 1 + c_1\log(2T)+c_2\sqrt{\sum_{\tau\in\Phi_t}\delta_\tau^2}$, with $\delta_\tau$ given in Assumption~\ref{ass:est_oracle_bern} and absolute constants $c_1=14,c_2=20$.

\For{$b\in B_t$}
{

Compute width $w_{t,0}(b) \leftarrow c_3L\gamma\|\bx_t\|_{\bA_t^{-1}}\parr*{\widehat{G}_t(b) + c_3L\parr*{\gamma\|\bx_t\|_{\bA_t^{-1}} + 4\delta_t}} + 4\delta_t$ with an absolute constant $c_3=120$. 

Compute width $w_{t,1}(b) \leftarrow c_3L\gamma\|\bx_t\|_{\bA_t^{-1}}\parr*{\parr{1-\widehat{G}_t(b)} + c_3L\parr*{\gamma\|\bx_t\|_{\bA_t^{-1}} + 4\delta_t}} + 4\delta_t$. 

Compute $u_{t,0}(b)\leftarrow \widehat{G}_t(b)\parr*{\bhtheta_{t}^{\top}\bx_t - b + \gamma\|\bx_t\|_{\bA_t^{-1}}} + 2\delta_t$.

Compute $u_{t,1}(b)\leftarrow \widehat{G}_t(b)\parr*{\bhtheta_{t}^{\top}\bx_t - b - \gamma\|\bx_t\|_{\bA_t^{-1}}} - \parr*{\bhtheta_t^{\top}\bx_t - \gamma\|\bx_t\|_{\bA_t^{-1}}} + 2\delta_t$.
}

Invoke \cref{alg:ucb_selection} with perturbation $c = \frac{1}{60L}$, CDF $\widehat{G}_t$, and value $\bhtheta_t^\top \bx_t$ to get $\parq{b_{\mathrm{left}},b_{\mathrm{right}}}$ and index $i\in\bra{0,1}$.
\end{algorithm}

Next we present the base algorithm LIN-TE in \cref{alg:base_alg_te} based on $\widetilde{e}_t(b)$. Based on the approximate linear model \eqref{eq:linear-model}, we apply linear regression of $\widetilde{e}_t(b_t)$ against $\bx_t$ to estimate $\btheta_*$, and use a weighted ridge regression (with weights inversely proportional to $\sigma_t(b_t)^2$) in Line 6 of \cref{alg:base_alg_te} to account for the heteroskedasticity. This choice of the weight is inspired by the form of the linear estimator with a smallest variance, and is important in reducing variance in the estimation of $\btheta_*$. 
The estimation performance of $\bhtheta_t$ is summarized in the following lemma. 

\begin{lemma}\label{lem:WLS}
Suppose $(v_{\tau,1},v_{\tau,0})_{\tau\in \Phi_t}$ are conditionally independent given $(\bx_\tau,b_\tau,M_\tau)_{\tau\in \Phi_t}$, then with probability at least $1-T^{-2}$, it holds that
\begin{equation*}
\abs*{\bhtheta^{\top}_t\bx_t - \btheta_*^{\top}\bx_t} \leq \gamma \|\bx_t\|_{\bA_t^{-1}}, 
\end{equation*}
where $\gamma$ is defined in Line 6 of \cref{alg:base_alg_te}. 
\end{lemma}

Starting from \cref{lem:WLS}, two more steps are necessary to turn a good estimator $\bhtheta_t$ of $\btheta_*$ into a good bidding strategy. First, the estimation guarantee in \cref{lem:WLS} requires a conditional independence condition, which is violated for adaptive bids; this dependence issue is discussed and mitigated via a master algorithm in \cref{sec:lte_dependence_issue}. Second, and more importantly, the bidder needs to translate an estimator $\bhtheta_t$ into a bid $b_t$. To this end, \cref{alg:base_alg_te} introduces two quantities $(u_{t,1}(b), u_{t,0}(b))$ for each bid $b$ (on Lines 9 and 10), related to the expected reward $\br_t(b)$ defined in \eqref{eq:expected_payoff_formulation}.

Specifically, the quantities $u_{t,0}(b)$ and $u_{t,1}(b)$ are two different upper confidence bounds (UCBs) of $\E_t[r_t(b)]$, up to additive constants independent of $b$. To see so, conditioned on the high-probability event in \cref{lem:WLS}, for every $b\in [0,1]$ we have
\begin{align*}
u_{t,0}(b) &\ge \br_t(b) - \E[v_{t,0}] = G_t(b)(\btheta_*^\top \bx_t - b)
\eqqcolon \br_{t,0}(b), \numberthis\label{eq:UCB0} \\
u_{t,1}(b) &\ge \br_t(b) - \E[v_{t,1}] = G_t(b)(\btheta_*^\top \bx_t - b) - \btheta_*^\top \bx_t  \\
&=  (1-G_t(b))(-\btheta_*^\top \bx_t) -G_t(b)b \eqqcolon \br_{t,1}(b). \numberthis\label{eq:UCB1}
\end{align*}
In addition, conditioned on the events in \cref{lem:WLS} and \cref{ass:est_oracle_bern}, the widths $w_{t,0}(b)$ and $w_{t,1}(b)$ defined in \cref{alg:base_alg_te} satisfy
\begin{equation*}
u_{t,0}(b) \le  \br_{t,0}(b) + w_{t,0}(b),\qquad u_{t,1}(b) \le  \br_{t,1}(b) + w_{t,1}(b).
\end{equation*}
The critical decision step would be choosing the ``better of two UCBs'' through invoking \cref{alg:ucb_selection} in \cref{sec:better-of-two-UCB}.

\subsection{``Better of Two UCBs''}\label{sec:better-of-two-UCB}

\begin{algorithm}[!t]\caption{UCB selection routine}
\label{alg:ucb_selection}
\textbf{Input:} perturbation $c>0$, estimated CDF $\widehat{G}_t$, estimated treatment value $\widehat{v}$.

Compute $b_{+} \gets \argmax_{b\in \Bcal} \widehat{G}_t(b)(\widehat{v} + 2c - b)$ and $b_{-} \gets \argmax_{b\in \Bcal} \widehat{G}_t(b)(\widehat{v} - 2c - b)$.

Define set $U \gets \bra{b\in \Bcal: \widehat{G}_t(b_-) - c \le \widehat{G}_t(b) \le \widehat{G}_t(b_+) + c}$.

Set $b_{\mathrm{left}}\gets \min U$ and $b_{\mathrm{right}}\gets \max U$.

Output interval $\parq{b_{\mathrm{left}}, b_{\mathrm{right}}}$ and UCB index $i=\indic[\widehat{G}_t(b_{\mathrm{left}}) \ge c]$.
\end{algorithm}






\begin{figure}[h!]
\centering

\begin{tikzpicture}[thick, >=stealth, scale=1.1]

\draw[->] (0,0) -- (5.3,0) node[right] {$\widehat{G}_t(b)$};
\draw[->] (0,-1.2) -- (5.2,-1.2) node[right] {$b$};
\draw[->] (0,0) -- (0,3) node[above] {UCB width};
\node[below] at (0,0) {$0$};
\node[left] at (0,2.7) {$1$};
\node[left] at (0,0.3) {$O(\delta_t)$};
\node[above, purple] at (4.6,-0.05) {\scriptsize $\widehat{G}_t(b_t^*)$};
\node[above, purple] at (4,-1.25) {\scriptsize $b_t^*$};
\draw[dashed] (0,2.7) -- (4.8,2.7);

\draw[blue, ultra thick]
  (0,0.3) -- (4,2.7) -- (5,2.7);

\draw[orange!90!brown, ultra thick]
  (0,2.7) -- (1,2.7) -- (5,0.3);

\node[blue] at (4.7,2.3) {$w_{t,0} \propto \widehat{G}_t$};
\node[orange!90!brown] at (5,1.1) {$w_{t,1} \propto 1-\widehat{G}_t$};

\draw[very thick, purple] (3.3,0) -- (3.3,-0.15);
\draw[very thick, purple] (4.9,0) -- (4.9,-0.15);
\draw[decorate,decoration={brace,amplitude=4pt},purple] (4.9,-0.15) -- (3.3,-0.15);
\node[purple, below] at (4.1,-0.2)
{\scriptsize $[\widehat{G}_t(b_{\mathrm{left}}),\,\widehat{G}_t(b_{\mathrm{right}})]$};

\node[cross out, draw=purple, thick, minimum size=3pt, inner sep=0pt] at (4.7,0) {};

\draw[very thick, purple] (1.8,-1.2) -- (1.8,-1.35);
\draw[very thick, purple] (4.7,-1.2) -- (4.7,-1.35);
\draw[decorate,decoration={brace,amplitude=4pt},purple] (4.7,-1.35) -- (1.8,-1.35);
\node[purple, below] at (3.25,-1.33)
{\scriptsize $[b_{\mathrm{left}},\,b_{\mathrm{right}}]$};


\draw[->,>=stealth,thick] (2.5,-1.15) -- (3,-0.05);

\node[cross out, draw=purple, thick, minimum size=3pt, inner sep=0pt] at (4,-1.2) {};
\node[left] at (2.75,-0.65) {\scriptsize $\widehat{G}_t$};


\end{tikzpicture}
    
\caption{\raggedright \textbf{The better choice of two UCBs.} 
We restrict our bid selection to the interval $[b_{\mathrm{left}},\,b_{\mathrm{right}}]$ returned by \cref{alg:ucb_selection}, which contains the hindsight optimal bid $b_t^*$ marked by \textcolor{purple}{\textbf{$\times$}}. Then we select the UCB with the \textbf{tighter} width $w_{t,i}$ over this bid interval by recognizing which endpoint ($0$ or $1$) the CDF interval $[\widehat{G}_t(b_{\mathrm{left}}),\,\widehat{G}_t(b_{\mathrm{right}})]$ is closer to. In this example, the CDF interval is closer to $1$, so the UCB index $i=1$ will be returned from \cref{alg:ucb_selection}, and we will use UCB $u_{t,1}$ (from width $w_{t,1}$) for further bid selection and elimination in \cref{alg:master_alg_te}. 
}
\label{fig:twoucb}
\end{figure}

This section describes the critical step in our decision process. The key structure is that $\br_{t,0}(b)$ and $\br_{t,1}(b)$, defined in \eqref{eq:UCB0} and \eqref{eq:UCB1} respectively, depend on the estimation error of $\btheta_*^\top \bx_t$ in different ways: the coefficient of $\btheta_*^\top \bx_t$ is $G_t(b)$ in $\br_{t,0}(b)$, and has magnitude $1-G_t(b)$ in $\br_{t,1}(b)$. Therefore, a possibly large estimation error of $\btheta_*^\top \bx_t$ when $G_t(b)$ is close to $0$ (resp. $1$) does not transform into a large estimation error in $\br_{t,0}(b)$ (resp. $\br_{t,1}(b)$); this is how the potentially unbounded variance $\sigma_t(b)^2$ in \eqref{eq:sigma_t} of the IPW estimator can be neutralized in the decision process. 

Motivated by this idea, our final bid is determined using the ``better of two UCBs'': we invoke a UCB selection subroutine to estimate whether $\widehat{G}_t(b_t^*)$ is closer to $0$ (resp. $1$) and focus only on the bids with $\widehat{G}_t(b)$ close to $0$ (resp. 1). 
Specifically, \cref{alg:ucb_selection} finds an interval $[b_{\mathrm{left}},b_{\mathrm{right}}]$ in the given bid space that satisfy two key properties. First, the interval contains the hindsight optimal bid $\argmax_{b^*\in \Bcal}\br_t(b^*)$ of the given bid space. Second, the corresponding estimated CDF over this interval $[\widehat{G}_t(b_{\mathrm{left}}), \widehat{G}_t(b_{\mathrm{right}})]$ is bounded away from either $0$ or $1$ by a constant $\frac{1}{30L}$.

\begin{lemma}[Good CDF interval]\label{lem:good_cdf_interval}
Let $c=\frac{1}{60L}$ and $[b_{\mathrm{left}},b_{\mathrm{right}}]$ be the interval found in \cref{alg:ucb_selection}. It holds that
\begin{enumerate}
    \item[(1)] $\argmax_{b^*\in \Bcal}\br_t(b^*) \in [b_{\mathrm{left}},b_{\mathrm{right}}]$;

    \item[(2)] $\widehat{G}_t(b_{\mathrm{right}}) - \widehat{G}_t(b_{\mathrm{left}}) \le 1 - 2c$ when $c^2 \ge \abs*{\widehat{v}-\btheta_*^\top \bx_t} + \|\widehat{G}_t - G_t\|_\infty$.\footnote{The numerical constants are not optimized.}
\end{enumerate}
\end{lemma}

The first claim indicates that it suffices to restrict our attention to the interval $[b_{\mathrm{left}},b_{\mathrm{right}}]$ found by \cref{alg:ucb_selection}, for it contains the optimal bid. The second claim then suggests that the indicator $i=\indic\parq*{\widehat{G}_t(b_{\mathrm{left}}) \ge c}$ can be used to select one of the two UCBs $u_{t,0}, u_{t,1}$. Indeed, $i=1$ implies that $\widehat{G}_t(b)\ge c$ for every bid $b\in [b_{\mathrm{left}},b_{\mathrm{right}}]$. Therefore, UCB $u_{t,1}$ is more favorable since its coefficient in front of $\btheta_*^\top \bx_t$ is approximately $(1-\widehat{G}_t(b))\le c^{-1}\widehat G_t(b)(1-\widehat G_t(b)) = (c\sigma_t(b)^2)^{-1}$. In other words, $u_{t,1}$ is the better of the two UCBs (up to a constant $c^{-1}$) with a tighter width $w_{t,1}$. Similarly, if $i=0$, it implies that $w_{t,0}(b)= O( \min\bra{w_{t,0}(b),w_{t,1}(b)})$ over the interval and UCB $u_{t,0}$ is the better one. Intuitively, this choice $i\in\bra{0,1}$ favors $\br_{t,0}$ for a small CDF, and favors $\br_{t,1}$ for a large CDF.

For this data-driven choice $i\in\bra{0,1}$, the next result shows that the resulting reward $\br_{t,i}(b)$ enjoys a small width $w_{t}(b) = \min\bra{w_{t,0}(b), w_{t,1}(b)}$. It thereby yields a small estimation-to-decision error in the subsequent bid selection in \cref{alg:master_alg_te}.
\begin{lemma}[Small width for selected UCB]\label{lem:TE_width}
Suppose the event in Lemma~\ref{lem:WLS} holds and $\widehat{G}_t$ satisfies \cref{ass:est_oracle_bern}. Let $\br_{t,0}$ and $\br_{t,1}$ be defined as in \eqref{eq:UCB0} and \eqref{eq:UCB1}. For the index $i\in\{0,1\}$ selected in \cref{alg:base_alg_te}, it holds that 
\begin{equation*}
u_{t,i}(b) - 2\min\bra{w_{t,0}(b), w_{t,1}(b)}\le \br_{t,i}(b)  \le u_{t,i}(b) 
\end{equation*}
for all $b\in [b_{\mathrm{left}}, b_{\mathrm{right}}]$. Here the interval $[b_{\mathrm{left}}, b_{\mathrm{right}}]$ is defined in \cref{alg:ucb_selection} and the widths $w_{t,1},w_{t,0}$ are defined on Line 8 of \cref{alg:base_alg_te}.
\end{lemma}

The main benefit is that the leading term of the confidence width $w_t(b)=\min\bra{w_{t,0}(b), w_{t,1}(b)}$ now scales with $\widehat{G}_t(b)(1-\widehat{G}_t(b))= 1/\sigma_t(b)^2$, so it helps neutralize possibly large $\sigma_t(b)$. Consequently, this adaptive selection of ``better of two UCBs'' is the key to eliminating the need for an overlap condition in our framework.

\subsection{Dependence Issue and Master Algorithm}\label{sec:lte_dependence_issue}

\begin{algorithm}[h!]\caption{SUP-LIN-TE (Master algorithm for FPA with linear treatment effects)}
\label{alg:master_alg_te}
\textbf{Input:} Time horizon $T$, Lipschitz constant $L$, estimation oracle $\Ocal$.

\textbf{Initialize:} set $S=\lceil \log_2 \sqrt{T} \rceil$, $\Phi^{(s)}_1=\varnothing$ for $s\in[S]$, and discretization $\Bcal = \bra*{\frac{j}{\sqrt{T}}: j\in [\lfloor \sqrt{T}\rfloor]}$.

\For{$t=1$ \KwTo $T$}
{   
Observe the context vector $\bx_t\in\R^d$.

Estimate the CDF and confidence level with $(\widehat{G}_t,\delta_t)\leftarrow \Ocal\parr*{\{M_\tau, \bx_\tau\}_{\tau<t}; \bx_t}$.

Initialize $B_1\leftarrow \Bcal$.

\For{$s=1$ \KwTo $S$}
{

Invoke Algorithm~\ref{alg:base_alg_te} with time indices $\Phi^{(s)}_t$, space $B_s$, and CDF $\widehat{G}_t$ to compute UCBs $\bra{u^{(s)}_{t,j}(b):j=0,1}_{b\in B_s}$, widths $\bra{w^{(s)}_{t,j}(b): j=0,1}_{b\in B_s}$, index $i_s\in\bra{0,1}$, and interval $\parq{b^{(s)}_{\mathrm{left}},b^{(s)}_{\mathrm{right}}}$.

Define $w_t^{(s)}(b) \gets \min\bra{w_{t,0}^{(s)}(b), w_{t,1}^{(s)}(b)}$ for $b\in B_s$.

\uIf{$\exists b\in B_s$ such that $w^{(s)}_t(b) > 2^{-s}$}{
Choose this $b_t\leftarrow b$.\label{line:master_te_select_bt}

Update: $\Phi^{(s)}_{t+1} \leftarrow \Phi^{(s)}_t\cup \{t\}$ and $\Phi^{(s')}_{t+1} \leftarrow \Phi^{(s')}_t$ for $s'\neq s$. Break the inner for loop.
}\uElseIf{$w^{(s)}_t(b)\leq \frac{1}{\sqrt{T}}$ for all $b\in B_s$}{
Choose $b_t \leftarrow \argmax_{b\in B_s}u^{(s)}_{t,i_s}(b)$ under the selected UCB $u_{i_s}$.
\label{line:master_te_exploit_bt}

Do not update: $\Phi^{(s')}_{t+1}\leftarrow \Phi^{(s')}_t$ for all $s'\in[S]$. Break the inner for loop.
}\Else{
We have $w^{(s)}_t(b)\leq 2^{-s}$ for all $b\in B_s$. \label{line:master_te_elim_step}

Eliminate bids: $\displaystyle B_{s+1}\leftarrow \bra*{b\in B_s : u^{(s)}_{t,i_s}(b) \geq \max_{b'\in B_s} u^{(s)}_{t,i_s}\parr*{b'} - 2\cdot 2^{-s}}\cap \parq*{b^{(s)}_{\mathrm{left}}, b^{(s)}_{\mathrm{right}}}$.
}
}

\tcc{Truncate the bid for nonzero propensity scores:}
Set $\gamma_t \gets 1 + c_1\log(2T) + c_2\sqrt{\sum_{\tau\in\Phi_t^{(s)}}\delta_\tau^2}$ where $s\in[S]$ is the stage $b_t$ is selected.

Let $z\gets \min\bra*{\gamma_t\sqrt{\frac{d}{T}}+4\delta_t, \frac{1}{2}}$, and bidder bids the truncated bid\label{line:master_bid_truncate}
$$b_t \gets \min\Big\{ \max\Big\{b_t, \widehat{G}_t^{-1}(z)\Big\}, \widehat{G}_t^{-1}(1-z)\Big\}.$$

Observe HOB $M_t$ and auction outcome $\indic[b_t\geq M_t]v_{t,1}$ and $\indic[b_t<M_t]v_{t,0}.$
}
\end{algorithm}

In this section we discuss why the condition of \cref{lem:WLS} may not hold because of a dependence issue, and how to mitigate it. Specifically, if $\Phi_t=[t-1]$ in \cref{alg:base_alg_te} and we select the bid based on UCBs $u_{t,0}(b)$ or $u_{t,1}(b)$, the selection of $b_t$ will depend on the realized outcomes
\begin{equation*}
\{\indic[b_s\geq M_s]v_{s,1}, \indic[b_s < M_s] v_{s,0}\}_{s<t};
\end{equation*}
therefore, conditioned on such a future bid $b_t$, the past potential outcomes $(v_{s,0},v_{s,1})_{s<t}$ may become dependent. To address this (subtle) dependence issue, we propose to use a master algorithm SUP-LIN-TE (\cref{alg:master_alg_te}). The idea behind this master algorithm is taken from \citep{auer2002using}, and has been subsequently used to address similar dependence issues in \citep{chu2011contextual,han2024optimal}. 

The purpose of this master algorithm is to partition the time indices $[t-1]$ into multiple stages, eliminate the suboptimal bids in a hierarchical manner, and select the bid $b_t$ only through the confidence widths $w_{t,0}(b), w_{t,1}(b)$ in LIN-TE, instead of the point estimates $u_{t,0}(b)$ or $u_{t,1}(b)$. Crucially, the widths $w_{t,0}(b), w_{t,1}(b)$ are computed based on the contexts $\bra{(\bx_\tau,b_\tau,M_\tau):\tau\in\Phi_t^{(s)}}$ on the current $s$-th stage, but do not involve the potential outcomes $\bra{(v_{\tau,1},v_{\tau,0}): \tau\in\Phi_t^{(s)}}$ on this stage. Here $\Phi_t^{(s)}\subseteq [t-1]$ denotes the partition belonging to the $s$-th stage. By contrast, the point estimates $u_{t,0}(b)$ and $u_{t,1}(b)$ (which depend on the potential outcomes $(v_{\tau,1},v_{\tau,0})$ on this stage) are only revealed when we decide to move to the next stage (Line 17). This information structure turns out to be the key idea to ensure the conditional independence in \cref{lem:WLS}: 

\begin{lemma}[Lemma 14 of \citep{auer2002using}]
\label{lem:ind_rewards}
Under \cref{ass:noconfounding} and  \ref{ass:context}, for every $s\in [S]$ and $t\in [T]$, $(v_{\tau,1},v_{\tau,0})_{\tau\in \Phi_t^{(s)}}$ are conditionally independent given $(\bx_\tau,b_\tau,M_\tau)_{\tau\in \Phi_t^{(s)}}$ in \cref{alg:master_alg_te}.
\end{lemma}

\begin{remark}
We remark that \cref{alg:master_alg_te} is applicable to either continuous or discrete bid space $\Bcal \subseteq [0,1]$. In particular, we can discretize $\Bcal$ a priori to any fine enough resolution that trades off the computational burden and the discretization error in the regret analysis. In addition, Line 13 of \cref{alg:master_alg_te} ensures that we never reach stage $s=S+1$ in the inner for loop. 
\end{remark}

\subsection{Regret Upper Bound}\label{sec:te_regret_analysis}
\begin{theorem}[Upper bound of \cref{alg:master_alg_te}]\label{thm:linear_te_reg_upper}
Let $\pi$ be the bidding strategy \cref{alg:master_alg_te}. Then under Assumptions \ref{assump:linear}--\ref{ass:Gt}, the policy $\pi$ achieves an expected regret
\begin{equation*}
\Reglte(\pi) = O\parr*{\sqrt{\Delta dT}\log^3 T}. 
\end{equation*}
\end{theorem}

In the sequel, we provide a proof sketch. WLOG we assume \cref{ass:est_oracle_bern} and the high-probability event in \cref{lem:WLS} holds almost surely. We also assume that $2\delta_t < 0.1$ for all $t\in[T]$, for there are at most $20\sum_{t=1}^T\delta_t \le \sqrt{\Delta T}$ times when this assumption is violated, contributing $O(\sqrt{\Delta T})$ regret. First, the following lemma proves that the bids in stage $s$ are only $O(2^{-s})$-suboptimal under $\br_t$.
\begin{lemma}\label{lem:master_confidence_bound_te}
In \cref{alg:master_alg_te}, for every $t\in[T]$ and $s\in[S]$, it holds that $\br_t(b_t^*)-\br_t(b) \leq 17\cdot 2^{-s}$ for all $b\in B_s$.
\end{lemma}
Let $b^{\circ}_t$ denote the selected bid before the truncation in Line 19 in \cref{alg:master_alg_te} and $b_t$ denote the final bid after truncation. Since \cref{ass:est_oracle_bern} implies $\|G_t-\widehat{G}_t\|_{\infty}\leq 2\delta_t$, by \cref{lem:truncation}, the truncation step does not hurt the reward too much:
\begin{equation}\label{eq:bounded_truncation_loss}
\br_t(b^{\circ}_t) - \br_t(b_t) \leq \gamma_t\sqrt{\frac{d}{T}} + 8\delta_t
\end{equation}
where $\gamma_t$ denotes the coefficient $\gamma$ define in Line 6 of \cref{alg:base_alg_te}. Then we have the following regret decomposition:
\begin{align*}
\Reglte(\pi) &= \sum_{t\in[T]} \E \parq{r_t(b^*_t)} - \E\parq{r_t(b_t)}
= \sum_{t\in[T]}\underbrace{\parr*{\br_t(b^*_t) - \br_t\parr*{b^{\circ}_t}}}_\text{(A)} + \underbrace{\parr*{\br_t\parr*{b^\circ_t} - \br_t(b_t)}}_\text{(B)}.
\end{align*}
By \eqref{eq:bounded_truncation_loss}, $\text{(B)}\leq \gamma_T\sqrt{dT} + 8\sum_{t=1}^T\delta_t = O\parr{\sqrt{\Delta dT}\log T}$ by Cauchy-Schwartz inequality. It remains to upper bound (A).

Denote by $\Phi^{(s)}\subseteq [T]$ the set of time indices belonging to stage $s\in [S]$, and $\Phi^{(S+1)}\subseteq [T]$ the set of time indices where all widths are below $\frac{1}{\sqrt{T}}$. A decomposition into different stages $s=1,\dots,S+1$ now gives
\begin{align*}
\sum_{t\in[T]}\parr*{\br_t\parr{b^*_t} - \br_t(b^\circ_t)} &= \sum_{t\in\Phi^{(S+1)}}\parr*{\br_t\parr{b^*_t} - \br_t(b^\circ_t)} + \sum_{s\in[S]}\sum_{t\in\Phi^{(s)}}\parr*{\br_t\parr{b^*_t} - \br_t(b^\circ_t)}\\
&\stepa{\leq} \frac{|\Phi^{(S+1)}|}{\sqrt{T}} + 17\sum_{s\in[S]}2^{-s}\abs{\Phi^{(s)}}\\
&\stepb{\leq} O(\sqrt{T}) + 17\sum_{s\in[S]}\sum_{t\in\Phi^{(s)}}w^{(s)}_t(b_t^\circ)\\
&\stepc{\leq} O(\sqrt{T}) + 17\sum_{s\in[S]}\sum_{t\in\Phi^{(s)}}w^{(s)}_t(b_t)
\end{align*}
where (a) is from \cref{lem:master_confidence_bound_te}, (b) uses $w^{(s)}_t(b_t^\circ) \geq 2^{-s}$ when $t\in\Phi^{(s)}$, and (c) uses $\widehat{G}_t(b_t)(1-\widehat{G}_t(b_t))\ge \widehat{G}_t(b_t^\circ)(1-\widehat{G}_t(b_t^\circ))$ to arrive at $w^{(s)}_t(b_t)\geq w^{(s)}_t(b_t^\circ)$ by the definition of $w_t^{(s)}$. Finally, applying the following elliptical potential lemma and $\sum_{t=1}^T\delta_t \le \sqrt{\Delta T}$ leads to the upper bound in \cref{thm:linear_te_reg_upper}.

\begin{lemma}\label{lem:elliptical-potential-TE}
For any $s\in[S]$, suppose $\delta_t<0.1$ for $t\in\Phi^{(s)}$ and we have
\[
\sum_{t\in\Phi^{(s)}}w^{(s)}_t(b_t) = O\parr*{\sqrt{\Delta dT}\log^2 T + d\Delta\log T + \sum_{t\in\Phi^{(s)}} \delta_t}.
\]
\end{lemma}

\section{Binary HOB Feedback}\label{sec:binary_hob}

In this section, we will address the more difficult ``joint value estimation, HOB estimation, and bidding problem'' under the binary HOB feedback. Recall that under the binary HOB feedback, the learner only observes the win-loss indicator $\indic[M_t\le b_t]$ about the HOB, as opposed to $M_t$ itself under the full-information feedback. This binary feedback crucially becomes dependent on the bidding process now, and the HOB estimation is no longer a standalone problem. Given this dependence between $M_t$ and $b_t$, the learner naturally needs to account for the exploration in learning $G_t$ in the bidding process.

\subsection{Analysis under Generic HOB Abstraction}
Similar to the full-information feedback, we describe our treatment effect estimator, base algorithm, and master algorithm separately.

\subsubsection{Modified IPW Estimator.}\label{sec:IPW_binary}
Given a CDF estimator $\widehat{G}_t$ satisfying \cref{ass:est_oracle_weaker_bern}, the truncated IPW estimator from \eqref{eq:truncated_estimator_te} is redefined as follows:
\begin{equation}\label{eq:truncated_estimator_te_weaker}
\widetilde{e}_t(b) = \frac{\indic[b\geq M_t]}{\max\bra*{\delta_t(b)^2 , \widehat{G}_t(b)}}v_{t,1} - \frac{\indic[b< M_t]}{\max\bra*{\delta_t(b)^2, 1-\widehat{G}_t(b)}}v_{t,0}.
\end{equation}
With a slight abuse of notation, we still use $\widetilde{e}_t$ to denote this estimator. The only difference is replacing the uniform error parameter $\delta_t$ by the bid-dependent counterpart $\delta_t(b)$. 

The bias and variance of this estimator can be controlled in a similar manner as in the full-information HOB feedback. 

\begin{lemma}[Adapted \cref{lem:bias-variance}]\label{lem:bias-variance_weaker}
Suppose the estimated CDF $\widehat{G}_t$ satisfies \cref{ass:est_oracle_weaker_bern}. Then there exist absolute constants $c, c'>0$ such that 
\begin{equation*}
\begin{cases}
    \abs*{\E[\widetilde{e}_t(b)] - \btheta_*^\top \bx_t} \le c(\delta_t(b)\sigma_t(b) + \xi \sigma_t(b)^2)\le c\parr*{\delta_t(b) + \xi}\sigma_t(b)^2 \\
    \Var(\widetilde{e}_t(b)) \le c'\max\{\xi\sigma_t(b)^4, \sigma_t(b)^2\} \le c'\sigma_t(b)^4.
\end{cases}
\end{equation*}
\end{lemma}
Recall that the variance proxy $\sigma_t(b)^2 = \parr*{\widehat{G}_t(b)\parr{1-\widehat{G}_t(b)}}^{-1}$ can be prohibitively large as we do not impose overlap conditions on the true CDF. In \Cref{lem:bias-variance_weaker}, both the bias and variance bounds naturally involve the bias parameter $\xi$ from HOB estimation, and the special case of $\xi=0$ recovers \Cref{lem:bias-variance}. The additional term $\xi \sigma_t(b)^4$ in the variance bound can be much worse than the previous term $\sigma_t(b)^2$ in \cref{lem:bias-variance_weaker}. Fortunately, using a weighted ridge regression with a different weight, this larger variance can still be mitigated in the final regret analysis.

\subsubsection{Base Algorithm Revisited.}

Next we revisit the base algorithm for weighted ridge regression with the new IPW estimator described by \cref{lem:bias-variance_weaker}. The differences between \cref{alg:base_alg_three} and \cref{alg:base_alg_te} are in \textcolor{blue}{blue}. To incorporate the potentially larger variance, the weights are changed from $\sigma_\tau^{-2}$ to $\sigma_\tau^{-4}$. The performance of the estimator $\bhtheta_t$ from this weighted ridge regression is summarized as follows.

\begin{lemma}[Adapted \cref{lem:WLS}]\label{lem:adapted_WLS}
Suppose $(v_{\tau,1},v_{\tau,0})_{\tau\in \Phi_t}$ are conditionally independent given $(\bx_\tau,b_\tau,M_\tau)_{\tau\in \Phi_t}$, then with probability at least $1-T^{-2}$, it holds that
\begin{align*}
\abs*{\bhtheta^{\top}_t\bx_t - \btheta_*^{\top}\bx_t} \leq \gamma \|\bx_t\|_{\bA_t^{-1}}, 
\end{align*}
where $\bA_t$ and $\gamma$ are defined in Line 3 and Line 6 of \cref{alg:base_alg_three} respectively. 
\end{lemma}

\begin{algorithm}[h!]\caption{Base algorithm under binary HOB feedback}
\label{alg:base_alg_three}
\textbf{Input:} Time indices $\Phi_t\subseteq [t-1]$, bid space $B_t$, estimated CDF $\widehat{G}_t$, time horizon $T$.

$\sigma_\tau \gets \sigma_\tau(b_\tau)$ in \eqref{eq:IPW_var_with_exploration} for $\tau\in \Phi_t$;

$\bA_{t}\leftarrow \bI + \sum_{\tau\in \Phi_t} $\textcolor{blue}{$\sigma_\tau^{-4}$}$ \bx_\tau \bx_\tau^{\top}$;

$\bz_{t} \leftarrow \sum_{\tau\in\Phi_t} $\textcolor{blue}{$\sigma_\tau^{-4}$}$ \bx_\tau \widetilde{e}_\tau$ where $\widetilde{e}_\tau = \widetilde{e}_\tau(b_\tau)$ is defined in \eqref{eq:IPW_var_with_exploration};

Compute estimator $\bhtheta_{t} \leftarrow \bA_{t}^{-1}\bz_{t}$.

Set $\gamma\gets 1 + c_1\log(2T)+c_2\sqrt{\color{blue}\sum_{\tau\in\Phi_t}(\delta_\tau^2 + \xi^2)}$, with $\delta_\tau=\delta_\tau(b_\tau;b^{\tau-1})$ given in \cref{ass:est_oracle_weaker_bern} and absolute constants $c_1=26,c_2=40$.

\For{$b\in B_t$}
{
Compute width $w_{t,0}(b) \leftarrow 120L\widehat{G}_t(b)\gamma\|\bx_t\|_{\bA_t^{-1}} + 4\delta_t(b) + 2\xi$. 

Compute width $w_{t,1}(b) \leftarrow 120L\parr{1-\widehat{G}_t(b)}\gamma\|\bx_t\|_{\bA_t^{-1}} + 4\delta_t(b) + 2\xi$. 

Compute $u_{t,0}(b)\leftarrow \widehat{G}_t(b)\parr*{\bhtheta_{t}^{\top}\bx_t - b + \gamma\|\bx_t\|_{\bA_t^{-1}}} + 2\delta_t(b) + \xi$.

Compute $u_{t,1}(b)\leftarrow \widehat{G}_t(b)\parr*{\bhtheta_{t}^{\top}\bx_t - b - \gamma\|\bx_t\|_{\bA_t^{-1}}} - \parr*{\bhtheta_t^{\top}\bx_t - \gamma\|\bx_t\|_{\bA_t^{-1}}} + 2\delta_t(b) + \xi$.
}

Invoke \cref{alg:ucb_selection} with perturbation $c=\frac{1}{60L}$, CDF $\widehat{G}_t$, and value $\bhtheta_t^\top \bx_t$ to get $\parq{b_{\mathrm{left}},b_{\mathrm{right}}}$ and index $i\in\bra{0,1}$. \label{line:find_ucb_interval}

\end{algorithm}

\subsubsection{Master Algorithm Revisited.}

\begin{algorithm}[h!]\caption{Master algorithm under binary HOB feedback}
\label{alg:joint_three_master}
\textbf{Input:} Time horizon $T$, Lipschitz constant $L$, HOB estimation oracle $\Ocal$.

\textbf{Initialize:} set $\Phi_1^{(s)} = \varnothing$ for $s\in[S]$ and discretize $\Bcal = \bra*{\frac{j}{\sqrt{T}}: j=0,1,\dots, \lfloor \sqrt{T}\rfloor}$.

\For{$t=1$ \KwTo $T$}{
Observe the context $\bx_t\in\R^d$.

Initialize $B_1\gets \Bcal$.

\For{$s=1$ \KwTo $S$}{
\textcolor{blue}{Get estimator $\widehat{G}_t^{(s)}$ using oracle $\Ocal(\bra{(\indic[M_\tau\le b_\tau], b_\tau, \bx_\tau):\tau\in \Phi_t^{(s)}}; \bx_t)$.} \label{line:stage_s_HOB_est}

Invoke \cref{alg:base_alg_te} with $\Phi_t^{(s)}$, $B_s$, $\widehat{G}_t^{(s)}$ to compute UCBs $\bra*{u^{(s)}_{t,j}(b):j=0,1}_{b\in B_s}$, widths $\bra*{w^{(s)}_{t,j}(b): j=0,1}_{b\in B_s}$, index $i_s\in\bra{0,1}$, interval $\parq{b^{(s)}_{\mathrm{left}}, b^{(s)}_{\mathrm{right}}}$.

Define $w_t^{(s)}(b) \gets \min\bra{w_{t,0}^{(s)}(b), w_{t,1}^{(s)}(b)}$ for $b\in B_s$.

{\color{blue}
\uIf{$\min_{b\in\Bcal}\max\bra{w_{t,0}^{(s)}(b), w_{t,1}^{(s)}(b)} > \frac{1}{60L}$}{
Sample $b_t \sim \mathrm{Bern}(\frac{1}{2})$.\label{line:master_explore} \qquad{\color{black}\tcp{Explore}}

Update: $\Phi_{t+1}^{(s)}\gets \Phi_t^{(s)}\cup\bra{t}$ and $\Phi^{(s')}_{t+1}\gets \Phi^{(s')}_t$ for $s\neq s'$. Break the inner for loop.
}
}
\uElseIf{$\exists b\in B_s$ such that $w^{(s)}_t(b)>2^{-s}$}{
Choose this $b_t\gets b$.\label{line:three_joint_bid_selection}

Update: $\Phi_{t+1}^{(s)}\gets \Phi_t^{(s)}\cup\bra{t}$ and $\Phi^{(s')}_{t+1}\gets \Phi^{(s')}_t$ for $s\neq s'$. Break the inner for loop.
}
\uElseIf{$w^{(s)}_t(b)\le \frac{1}{\sqrt{T}}$ for all $b\in B_s$}{
Choose $b_t\gets \argmax_{b\in B_s} u^{(s)}_{t,i_s}(b)$.

Do not update: $\Phi^{(s')}_{t+1}\gets \Phi^{(s')}_t$ for all $s'\in[S]$. Break the inner for loop.
}
\uElse{
We have $w^{(s)}_t(b)\le 2^{-s}$ for all $b\in B_s$.

Eliminate bids: $B_{s+1} \gets \bra*{b\in B_s: u^{(s)}_{t,i_s}(b) \ge \max_{b\in B_s}u^{(s)}_{t,i_s}\parr*{b} - 2\cdot 2^{-s}}\cap \parq*{b^{(s)}_{\mathrm{left}}, b^{(s)}_{\mathrm{right}}}$.
}
}
Bid $b_t$.

Observe HOB feedback $\indic[M_t\le b_t]$, auction outcome $\indic[M_t\le b_t]v_{t,1}$ and $\indic[M_t>b_t]v_{t,0}$.
}
\end{algorithm}

Recall that a master algorithm (\cref{alg:master_alg_te}) was used under the full-information HOB feedback to address a subtle dependence issue between the outcome observations and the value estimation. At a high level, it was because the value estimation in \cref{lem:adapted_WLS} requires conditional independence of the outcome observations. Now under the binary HOB feedback, we will also apply the master algorithm to HOB estimation to ensure that the HOB observations $\bra{\indic[M_\tau\le b_\tau]}_\tau$ are conditionally independent to allow flexible implementations of \cref{ass:est_oracle_weaker_bern}. Indeed, one typically requires such conditional independence to apply concentration results and derive estimation error bounds as in \cref{ass:est_oracle_weaker_bern}. This is captured by invoking HOB estimation in Line 7 of \cref{alg:joint_three_master} within the inner loop and only over the subset of time indices $\Phi_t^{(s)}$.
\begin{lemma}
\label{lem:ind_hobs}
Under \cref{ass:noconfounding} and  \ref{ass:context}, for every $s\in [S]$ and $t\in [T]$, $(v_{\tau,1},v_{\tau,0})_{\tau\in \Phi_t^{(s)}}$ are conditionally independent given $(\bx_\tau,b_\tau,M_\tau)_{\tau\in \Phi_t^{(s)}}$, and $(\indic[M_\tau\le b_\tau])_{\tau\in \Phi_t^{(s)}}$ are conditionally independent given $(\bx_\tau, b_\tau)_{\tau\in \Phi_t^{(s)}}$ in \cref{alg:joint_three_master}.
\end{lemma}

More importantly, there is an additional exploration (via randomized experiments) step from Line 10 to 12 in \cref{alg:joint_three_master}. As discussed in \cref{sec:IPW_binary}, the need of this exploration step ultimately arises as a consequence of the larger variance in the IPW estimator under the binary HOB feedback. Recall that in the ``Better of Two UCBs'' trick in \cref{sec:better-of-two-UCB}, \cref{lem:good_cdf_interval} requires the estimation error to be $O(1)$ to identify the better UCB. This is, at a high level, achieved by the truncation step in Line 20 of \cref{alg:master_alg_te} under the full-information HOB feedback: By ensuring a tiny but nonzero propensity score, we are able to bound the number of times the estimation error is beyond a constant level. This approach turns out to fail under the binary case because, due to the larger IPW variance in \cref{lem:bias-variance_weaker}, one would need to truncate more and thereby suffer a suboptimal regret. To work around this issue, we replace truncation by the exploration step in \cref{alg:joint_three_master}.

When we explore at time $t$, the bid is sampled uniformly such that the probability of winning and losing the auction is exactly $\frac{1}{2}$. Consequently, we are able to derive a more stable IPW estimator with a constant variance: with a slight abuse of notation,
\begin{equation}\label{eq:IPW_var_with_exploration}
(\widetilde{e}_t(b_t), \sigma_t(b_t)) = \begin{cases}
    (2(\indic[b\ge M_t] - \indic[b<M_t]), 2),\qquad \text{if $t$ is selected on Line 10 in \cref{alg:joint_three_master};}\\
    \text{as defined in \eqref{eq:truncated_estimator_te_weaker} and \eqref{eq:sigma_t}},\qquad \text{otherwise.}
\end{cases}
\end{equation}

\subsubsection{Regret Analysis}

Finally, we present the regret guarantee \cref{thm:main-lte-binary} for \cref{alg:joint_three_master} under the binary HOB feedback. The rest of this section is devoted to the proof of this result.

\begin{repeattheorem}[\cref{thm:main-lte-binary}]
Suppose there is some $T_0$ such that for every $t\ge T_0$, the bidder implements Abstraction \ref{ass:est_oracle_weaker_bern} with $\|\delta_t\|_\infty \le c$ for some constant $c$ depending on $L$.
Under Assumptions \ref{assump:linear}--\ref{ass:Gt}, there is an algorithm $\pi_{\mathrm{bin}}$ that achieves the expected regret
\begin{equation*}
\Reglte(\pi_{\mathrm{bin}}) = \widetilde{O}\parr*{T_0 + \sqrt{\Delta dT}  + \xi T}, 
\end{equation*}
where $\Delta = 1+\sup_{b^T} \sum_{t=1}^T\delta_t(b_t;b^{t-1})^2$. 
\end{repeattheorem}

WLOG we again assume \cref{ass:est_oracle_weaker_bern} and the high-probability event in \cref{lem:adapted_WLS} hold almost surely. Let $\hb^*_t = \argmin_{b\in\Bcal}\abs{b_t^* - b}$ denote the closest bid in the discretized space to the hindsight optimal.
\begin{lemma}[Adapted \cref{lem:master_confidence_bound_te}]\label{lem:adapted_master_confidence_bound_te}
In \cref{alg:joint_three_master}, for every $t\ge T_0$ and $s\in[S]$, it holds that $\br_t(b_t^*)-\br_t(b) \leq 17\cdot 2^{-s}$ for all $b\in B_s$.
\end{lemma}
Denote by $\Phi^{(s)}_1\subseteq [T]\backslash[T_0]$ the set of time indices belonging to stage $s\in [S]$ and $w_t^{(s)}(b_t) > 2^{-s}$ (i.e. selected by Line 13), $\Phi^{(s)}_0\subseteq [T]\backslash[T_0]$ the set of time indices belonging to stage $s\in [S]$ and for exploration (i.e. selected by Line 10), and $\Phi^{(S+1)}\subseteq [T]\backslash[T_0]$ the set of time indices where all widths are below $\frac{1}{\sqrt{T}}$. Then the regret can be decomposed as
\begin{align*}
\Reglte(\pi) &= \sum_{t\in[T]} \E \parq{r_t(b^*_t)} - \E\parq{r_t(b_t)}\\
&\stepa{\le} T_0 + 3L\sqrt{T} + \frac{|\Phi^{(S+1)}|}{\sqrt{T}} + \sum_{s\in[S]}|\Phi^{(s)}_0| + 17\sum_{s\in[S]}2^{-s}|\Phi^{(s)}_1|\\
&\stepb{\le} T_0 + O(\sqrt{T}) + \underbrace{\sum_{s\in[S]}|\Phi^{(s)}_0|}_{(\spadesuit)} + 17\underbrace{\sum_{s\in[S]}\sum_{t\in\Phi^{(s)}}w_t^{(s)}(b_t)}_{(\clubsuit)}
\end{align*}
where (a) follows from \cref{lem:adapted_master_confidence_bound_te}, and (b) uses $w_t^{(s)}(b_t) \ge 2^{-s}$ when $t\in\Phi^{(s)}_1$ in \cref{alg:joint_three_master}. Then the following potential lemma leads to $(\clubsuit) = \widetilde{O}(\sqrt{\Delta dT} + \xi T)$.
\begin{lemma}\label{lem:elliptical-potential-binary}
For any $s\in[S]$ in \cref{alg:joint_three_master}, we have
\[
\sum_{t\in\Phi^{(s)}_1}w^{(s)}_t(b_t) = O\parr*{\sqrt{\Delta dT}\log^2 T + d\Delta\log T + \sum_{t\in\Phi^{(s)}_1} \delta_t + \xi|\Phi^{(s)}_1|}.
\]
\end{lemma}
Recall that $\sum_{t=1}^T\delta_t \le \sqrt{\Delta T}$ by Cauchy-Schwartz inequality. To conclude the proof, it remains to bound the total number of exploration times $(\spadesuit)$ by the following lemma.
\begin{lemma}\label{lem:bound_exploration_binary}
In \cref{alg:joint_three_master}, it holds that
\[
\sum_{s\in[S]}|\Phi^{(s)}_0| = O\parr*{d\Delta\log^2 T + \xi T}.
\]
\end{lemma}

\subsection{Oracle Implementation for i.i.d. HOBs}\label{sec:implementation_iid_hob}

Moving forward, we focus on i.i.d. HOBs ($G_t\equiv G$) and present an implementation of \cref{ass:est_oracle_weaker_bern}. Despite the seemingly simple i.i.d. structure, this implementation is highly nontrivial for two reasons: (1) We need to derive bid-dependent error parameters $\bra{\delta_t}_t$ in \cref{ass:est_oracle_weaker_bern} that are \textit{self-normalizing} with respect to any bidding trajectory; otherwise the final quantity $\Delta$ may not lead to optimal regret bound or even bounded at all. (2) The output estimation $\widehat{G}_t$ needs to be a valid CDF, and in particular, monotone. This is nontrivial because, to derive bid-dependent error bounds based on the binary observations, the implementation will need to restrict its attention to \textit{local} regions in the bid space $[0,1]$, whereas monotonicity is a \textit{global} shape constraint. We will enforce monotonicity via a simple procedure. 

For the ease of notation in this section, we fix the discretization size $h\coloneqq d^{\frac{1}{3}}T^{-\frac{1}{3}}$ and an interval partition $I_1,\dots, I_J$ of $[0,1]$ such that $|I_j|\in [h,2h]$ and $J=\lfloor h^{-1}\rfloor$.

\subsubsection{A Binning Estimator.}
Suppose we are given a subset of indices $\Phi_t\subseteq [t-1]$ at time $t$. Define
\begin{equation}\label{eq:offline_G_est}
\begin{split}
g_j &\coloneqq \frac{1}{\abs{\bra{b_\tau: b_\tau\in I_j}}} \sum_{\tau\in\Phi_t}\indic[b_\tau\in I_j]\indic[M_\tau\le b_\tau],\\
\widehat{G}^{(0)}_{t}(b) &\coloneqq \sum_{j=1}^J \indic[b\in I_j]g_j.
\end{split}
\end{equation}
This $\widehat{G}^{(0)}_t$ is piece-wise constant over the interval partition and is computationally cheap as it only involves the empirical averages. Its estimation performance is summarized in the following lemma, where for notational simplicity we write $\delta_t(b)$ in place of $\delta_t(b;b^{t-1})$. 

\begin{lemma}[Estimation Error of \eqref{eq:offline_G_est}]\label{lem:hob_est_error_bound}
Suppose $h\in(0,1)$ and the intervals $\bra{I_j}_{j\in[J]}$ partition $[0,1]$ and $|I_j|\le 2h$ for every $j$. Also suppose $(\indic[M_\tau\le b_\tau])_{\tau\in\Phi_t}$ are conditionally independent given $(\bx_\tau, b_\tau)_{\tau\in\Phi_t}$. With probability $1-T^{-2}$, for every $b\in[0,1]$
\begin{equation*}
\abs*{\widehat{G}^{(0)}_t(b) - G(b)} \le 
\sqrt{G(b)(1-G(b))}\delta_t(b) + \delta_t(b)^2 + 3Lh,
\end{equation*}
where $\delta_t(b) = c\sqrt{\frac{\log(2T/h)}{N_t(b)}}$ for constant $c=8$. Here $I(b)$ denotes the interval $I_j\ni b$ and $N_t(b) = |\bra{\tau\in \Phi_t: b_\tau\in I(b)}|$ denotes the number of neighbor observations in $\Phi_t$.
\end{lemma}

To satisfies the condition $\|\delta_t\|_\infty = O(1)$ in \cref{thm:main-lte-binary}, it suffices to apply $\widetilde{O}(1)$ number of explorations in each interval $I_j$. Nonetheless, this empirical estimator $\widehat{G}^{(0)}_t$ in \eqref{eq:offline_G_est} is not necessarily a valid CDF, as it may not be monotone. We will use some more explorations in the next section to enforce monotonicity on $\widehat{G}^{(0)}_t$.

\subsubsection{Enforcing Monotonicity.}\label{sec:hob_monotonicity}
It would be relatively easier to construct a monotone estimator $\widehat{G}_t$ from the empirical averages in \eqref{eq:offline_G_est} if the RHS of the error width in \cref{lem:hob_est_error_bound} were \textit{known}. In this case, one can explicitly find a monotone estimator that shares the same error width up to a constant; see Figure \ref{fig:cdf_monotone1} for a simple recursive procedure for this.

\begin{figure}[h!]
\centering

\begin{tikzpicture}[scale=6,>=stealth]
\foreach \xx in {0.03,0.07,0.11}{
  \fill[gray] (\xx,0.26) circle (0.007);
}
\foreach \xx in {1.18,1.22,1.26}{
  \fill[gray] (\xx,0.63) circle (0.007);
}
  \foreach \xa/\xb/\band/\y/\cdf in {%
    0.14/0.34/0.08/0.27/0.23,
    0.34/0.54/0.14/0.36/0.29,
    0.54/0.74/0.22/0.38/0.29,
    0.74/0.94/0.14/0.5/0.43,
    0.94/1.14/0.1/0.6/0.55
  }{
    \pgfmathsetmacro{\halfband}{0.5*\band}
    \pgfmathsetmacro{\ymin}{max(0,\y-\halfband)}
    \pgfmathsetmacro{\ymax}{min(1,\y+\halfband)}

    \fill[blue!12] (\xa,\ymin) rectangle (\xb,\ymax);
    \draw[blue,dashed]    (\xa,\ymin) rectangle (\xb,\ymax);
    \draw[very thick] (\xa,\y) -- (\xb,\y);
    \draw[very thick,red] (\xa,\cdf) -- (\xb,\cdf);
  }

  \draw[->] (0,0) -- (1.4,0) node[below] {$x$};
  \draw[->] (0,0) -- (0,0.7) node[left] {$y$};
  \draw (0,0) -- (0,-0.02) node[below=2pt] {\small 0};
  \draw (1.32,0) -- (1.32,-0.02) node[below=2pt] {\small 1};
  \draw (0.54,0) -- (0.54,-0.02) node[below=2pt] {};
  \draw (0.74,0) -- (0.74,-0.02) node[below=2pt] {};
  \node[below] at (0.09,0) {\small $j=1$};
  \draw[->, red, thick] (0.2,-0.05) -- (0.4,-0.05);
  
  \foreach \x in {0.14,0.34,0.54,0.74,0.94,1.14}{\draw[dashed,gray!60] (\x,0) -- (\x,0.7);}

  \begin{scope}
    \draw[fill=white,draw=black] (0.72,0.04) rectangle (1.4,0.25);
    \draw[very thick] (0.76,0.2) -- (0.96,0.2);
    \node[anchor=west] at (0.98,0.2) {\small empirical mean};
    \draw[very thick,red] (0.76,0.14) -- (0.96,0.14);
    \node[anchor=west] at (0.98,0.14) {\small monotone value};
    \fill[blue!12] (0.76,0.06) rectangle (0.96,0.1);
    \draw[blue]    (0.76,0.06) rectangle (0.96,0.1);
    \node[anchor=west] at (0.98,0.08) {\small width};
  \end{scope}
\end{tikzpicture}

\caption{\raggedright \textbf{Enforcing monotonicity by width.} The procedure begins with the first (leftmost) interval and then iterates through the intervals $j=1,\dots,J$. For each interval, it sets the estimated value to be the maximum of the lower confidence bound and the value of the previous interval. Its validity is guaranteed when the true CDF $G$ indeed lies within the shaded regions.}
\label{fig:cdf_monotone1}
\end{figure}

The iterative procedure in Figure \ref{fig:cdf_monotone1} works as follows: starting from the leftmost interval $j=1$, the learner knows the CDF estimator $\widehat{G}_t^{(0)}(b)\equiv \widehat{G}_t^{(0)}(I_1)$, for $b\in I_1$ (since it is constant over each interval), and (hypothetically) its confidence width $\widehat{w}^{(0)}_{t,1}$. Then the learner sets the new monotone estimator to be $\widehat{G}_t(I_1) \coloneqq \widehat{G}_t^{(0)}(I_1) - \widehat{w}^{(0)}_{t,1}$. Note that the estimation error satisfies
\[
\abs*{\widehat{G}_t(I_1) - G(b)} \le \abs*{\widehat{G}^{(0)}_t(I_1) - G(b)} + \widehat{w}^{(0)}_{t,1} \le 2\widehat{w}^{(0)}_{t,1}
\]
for all $b\in I_1$. Then for the next interval $j=2$, set $\widehat{G}_t(I_2)\coloneqq \max\bra{\widehat{G}_t(I_1), \widehat{G}_t^{(0)}(I_2) - \widehat{w}^{(0)}_{t,2}}$, where the max guarantees monotonicity of the final estimator $\widehat{G}_t$. This iterative definition proceeds until $j=J$. The exact recursive procedure for enforcing monotonicity is slightly more involved to deal with the unknown confidence widths, and is deferred to Appendix \ref{sec:enforce_monotonicity_details}. 

\begin{lemma}[Monotone CDF Adjustment]\label{lem:monotone_cdf_adjustment_main}
Suppose the conditions in \cref{lem:hob_est_error_bound} hold. Then there is a monotone $\widehat{G}_t$ (output by \cref{alg:enforce_monotone}) that satisfies: For every $j\in[J]$ and $b\in I_j$, if $\delta_t(b) = O(\sqrt{h})$ and $G(b)\in[Lh,1-Lh]$, then
\begin{equation*}
\abs*{\widehat{G}_t(b) - G(b)} \le 
c_0\sqrt{G(b)(1-G(b))}\delta_t(b) + \delta_t(b)^2 + 3Lh,
\end{equation*}
for some constant $c_0$ that only depends on $L$.
\end{lemma}

\subsubsection{Initialization.}\label{sec:three_joint_init}

This section introduces an initialization to satisfy the two conditions in \cref{lem:monotone_cdf_adjustment_main} (to compensate that we do not know the widths of $\widehat{G}_t^{(0)}$). Recall that $h=d^{\frac{1}{3}}T^{-\frac{1}{3}}$ and we have a partition $\bra{I_j}_{j\in[J]}$ of $[0,1]$ with $|I_j|\in[h,2h]$.

\begin{algorithm}[h!]\caption{Initialization}
\label{alg:joint_three_init}
\textbf{Input:} Time horizon $T$, interval partition $\bra{I_j}_{j\in[J]}$, partition width $h$, Lipschitz constant $L$.

\textbf{Initialize:} set $n_0=16\log(2T)L^{-1}h^{-1}$.

\For{$j=1$ \KwTo $J$}{
\For{$t=(j-1)T_0$ \KwTo $jn_0$}{
Bid any $b_t\in I_j$.

Observe HOB feedback $\indic[M_t\le b_t]$.
}

Compute empirical average $g_j \gets \frac{1}{n_0}\sum_{t=(j-1)n_0}^{jn_0}\indic[M_t\le b_t]$.
}

Set CDF estimator $\widehat{G}_0(b) = \sum_{j=1}^J\indic[b\in I_j]g_j$.

Set $b_{\min} \gets \inf\bra{b\in[0,1]: \widehat{G}_0(b) \ge 9Lh}$ and $b_{\max} \gets \inf\bra{b\in[0,1]: \widehat{G}_0(b) \ge 1- 9Lh}$. \label{line:truncate_bid_space}

Set interval $\Bcal_0 \gets [b_{\min}, b_{\max}]\subseteq [0,1]$.

Output samples $\bra{b_t,\indic[M_t\le b_t], \indic[b_t\geq M_t]v_{t,1},\indic[b_t<M_t]v_{t,0}}_{t\in[Jn_0]}$ and bid space $\Bcal_0$.
\end{algorithm}

\cref{alg:joint_three_init} takes the first $\widetilde{O}(Jh^{-1}) = \widetilde{O}(d^{-\frac{2}{3}}T^{\frac{2}{3}})$ rounds as the initialization stage. During this stage, it samples $\widetilde{O}(h^{-1})=\widetilde{O}(d^{-\frac{1}{3}}T^{\frac{1}{3}})$ bids from each interval $I_j$. Then at the end of the initialization stage, we have $\delta_t(b) \le \sqrt{h}$ by definition in \cref{lem:hob_est_error_bound}, which satisfies the first condition in \cref{lem:monotone_cdf_adjustment_main}.

Next, we consider the other condition that $G(b)\in[Lh,1-Lh]$. A key observation is that, when $G(b)(1-G(b)) = O(h)$, the empirical average estimator $\widehat{G}_0$ defined in Line 8 approximates $G$ well up to $O(h)$ error by \cref{lem:hob_est_error_bound}. Therefore, \cref{alg:joint_three_init} is able to identify such regions by looking at $\widehat{G}_0$ in Line 8-11. Specifically, the returned interval $\Bcal_0$ satisfies:

\begin{lemma}[Truncated bid space]\label{lem:truncated_bid_space}
Consider the bid space $\Bcal_0= [b_{\min}, b_{\max}]$ in \cref{alg:joint_three_init}. With probability at least $1-T^{-2}$, it holds that
\begin{enumerate}
    \item $G(b)\in[Lh, 1-Lh]$ for every $b\in\Bcal_0$;

    \item $G(b)\in[0,25Lh]\cup [1-25Lh, 1]$ for $b\notin \Bcal_0$.
\end{enumerate}
\end{lemma}
By \cref{lem:truncated_bid_space}, we are guaranteed that $G(b)$ is bounded away from $0$ and $1$ over $\Bcal_0$, thereby satisfying the other condition of \cref{lem:monotone_cdf_adjustment_main}. This result also allows us to focus on the returned interval $\Bcal_0$ in HOB estimation throughout the remaining time horizon, since the change of the true CDF $G$ is too small outside of $\Bcal_0$ by the second claim.

\subsubsection{Overall Oracle Implementation.}

We are now in the position of piecing everything together for HOB estimation. For every time $t>Jn_0$ after the initialization stage in \cref{alg:joint_three_init}, we will implement \cref{ass:est_oracle_weaker_bern} by \cref{alg:hob_est}. At a high level, it sets a trivial estimation value for the bids outside of $\Bcal_0$ and the empirical average over the interval otherwise. Then it applies \cref{alg:enforce_monotone} to obtain the final monotone CDF estimator $\widehat{G}_t$ over the entire bid space $[0,1]$. The performance of the final estimator is summarized in \cref{lem:hob_est_error_bound_monotone}.

\begin{algorithm}[h!]\caption{HOB Estimation Oracle}
\label{alg:hob_est}
\textbf{Input:} Interval partition $\bra{I_j}_{j\in[J]}$, time indices $\Phi\supseteq [Jn_0]$, initialized interval $\Bcal_0=[b_{\min}, b_{\max}]$, Lipschitz constant $L$.

\For{interval index $j=1$ \KwTo $J$}{
\uIf{$I_j\not\subseteq \Bcal_0$}{
Set $\widehat{g}_j \gets \begin{cases}
\frac{25}{2}Ld^{\frac{1}{3}}T^{-\frac{1}{3}}\qquad &\text{if $I_j\cap [0,b_{\min})\neq\varnothing$;}\\
1-\frac{25}{2}Ld^{\frac{1}{3}}T^{-\frac{1}{3}}\qquad &\text{if $I_j\cap (b_{\max},1]\neq\varnothing$.}
\end{cases}$

Set $u_j \gets \frac{25}{2}Ld^{\frac{1}{3}}T^{-\frac{1}{3}}$
}
\uElse{
Compute empirical average $\widehat{g}_j \gets \frac{1}{\abs{\tau\in\Phi: b_\tau\in I_j}}\sum_{\tau\in \Phi}\indic[b_\tau\in I_j]\indic[M_\tau\le b_\tau]$.

Truncate by setting $\widehat{g}_j\gets \min\bra{\max\bra{Ld^{\frac{1}{3}}T^{-\frac{1}{3}}, \widehat{g}_j}, 1-Ld^{\frac{1}{3}}T^{-\frac{1}{3}}}$.\label{line:hob_est_truncation}

Set $u_j \gets 12\delta_j\sqrt{\widehat{g}_j(1-\widehat{g}_j)} + \delta_j^2 + 2Ld^{\frac{1}{3}}T^{-\frac{1}{3}}$ where $\delta_j = \delta_t(b)$ for any $b\in I_j$ as defined in \cref{lem:hob_est_error_bound}.
}
}

Invoke \cref{alg:enforce_monotone} with $\bra{I_j}_{j\in[J]}$, $\bra{\widehat{g}_j}_{j\in[J]}$, and $\bra{u_j}_{j\in[J]}$ to compute $\widehat{G}_t$.
\end{algorithm}

\begin{lemma}[Monotone CDF Estimator]\label{lem:hob_est_error_bound_monotone}
Let $h=d^{\frac{1}{3}}T^{-\frac{1}{3}}$ and $\bra{I_j}_{j\in[J]}$ be a partition of $[0,1]$ with $|I_j|\in[h,2h]$. Suppose $\Bcal_0$ is given by \cref{alg:joint_three_init}. Let $\widehat{G}_t$ be the CDF estimator given by \cref{alg:hob_est} at time $t$. Then with probability at least $1-2T^{-2}$,
\begin{equation*}
\abs*{G(b) - \widehat{G}_t(b)} \le 
c\delta_t(b)\sqrt{G(b)(1-G(b))} + c\delta_t(b)^2 + cLh, \qquad \text{for every $b\in [0,1]$}
\end{equation*}
where $c>0$ is an absolute constant and $\delta_t(b)$ is defined as in \cref{lem:hob_est_error_bound}.
\end{lemma}

\subsubsection{End-to-end Regret Bound}\label{sec:binary_iid_regret_analysis}

To conclude this section on i.i.d. HOB, we prove the following explicit end-to-end regret bound by verifying \cref{thm:main-lte-binary} with our implementation of \cref{ass:est_oracle_weaker_bern}. For the reader's convenience, the bound is restated below:

\begin{repeattheorem}[\cref{thm:main-lte-binary-iid}]
Consider the binary HOB feedback and i.i.d. HOB. Under Assumptions \ref{assump:linear}--\ref{ass:Gt}, there is an algorithm $\pi_{\mathrm{bin}}$ that achieves the expected regret
\begin{equation*}
\Reglte(\pi_{\mathrm{bin}}) = \widetilde{O}\parr*{d^{\frac{1}{3}}T^{\frac{2}{3}}}. 
\end{equation*}
\end{repeattheorem}

Recall that in the implementation throughout \cref{sec:implementation_iid_hob}, we discretize the bid space $[0,1]$ into $J\le \lfloor d^{-\frac{1}{3}}T^{\frac{1}{3}}\rfloor$ intervals with $h=d^{\frac{1}{3}}T^{-\frac{1}{3}}$ and each interval $|I_j|\in[h,2h]$. To apply \cref{thm:main-lte-binary}, first note that the initialization period takes $T_0 = Jn_0 = O(T^{\frac{2}{3}}\log T)$. The bias parameter by \cref{lem:hob_est_error_bound_monotone} is $\xi=O(h) = O(d^{\frac{1}{3}}T^{-\frac{1}{3}})$. To apply \cref{thm:main-lte-binary}, the key lies in upper bounding the quantity $\Delta$ with respect to any bidding trajectory.

Recall that the master algorithm runs in $S=\log\lceil \sqrt{T}\rceil$ stages. Let $\Phi^{(s)}\subseteq [T]\backslash[T_0]$ denote the set of time indices belonging to stage $s\in [S]$. We first have
\begin{equation*}
\Delta = 1+ \sum_{t=T_0+1}^T\delta_t(b_t;b^{t-1})^2 = 1+\sum_{s=1}^S\sum_{t\in\Phi^{(s)}}\delta_t(b_t;b^{t-1})^2.
\end{equation*}
For each stage $s$, 
\begin{align*}
\sum_{t\in\Phi^{(s)}}\delta_t(b_t;b^{t-1})^2 &= \sum_{j=1}^J\sum_{t\in\Phi^{(s)}}\indic[b_t\in I_j]\delta_t(b_t;b^{t-1})^2 
\stepa{\le} 64\log(2T)\sum_{j=1}^J\sum_{t\in\Phi^{(s)}_j}\frac{1}{T_0 + \abs{\bra{\tau\in\Phi^{(s)}_j:\tau<t}}}\\
&\le 64\log(2T)\sum_{j=1}^J\log|\Phi^{(s)}_j|
\le 64J\log^2(2T)
\le 64\log^2(2T)d^{-\frac{1}{3}}T^{\frac{1}{3}}
\end{align*}
where $\Phi^{(s)}_j = \bra{t\in\Phi^{(s)}: b_t\in I_j}$, and (a) follows from the definition of $\delta_t(b;b^{t-1})$ in \cref{lem:hob_est_error_bound} and that the initial $T_0$ samples are used in later HOB estimation in \cref{alg:hob_est}. The last line simply plugs in $J\le d^{-\frac{1}{3}}T^{\frac{1}{3}}$. Consequently,
\begin{equation*}\label{eq:reg_decomp_iid_1}
\Delta \le 1+ S\cdot 64\log^2(2T)d^{-\frac{1}{3}}T^{\frac{1}{3}} = O(d^{-\frac{1}{3}}T^{\frac{1}{3}}\log^3 T).
\end{equation*}
Plugging \eqref{eq:reg_decomp_iid_1} back in \cref{thm:main-lte-binary} concludes the proof of \cref{thm:main-lte-binary-iid}.

\section{Numerical Experiments}\label{sec:numericals}

To validate our theoretical guarantees and illustrate why causal inference is essential for bidding, we present numerical evaluations of our algorithms in this section. The experiments focus on synthetic data because the baseline outcomes $(v_{t,0})_t$ for the bidders are typically overlooked and thus absent from existing real-time bidding datasets, making offline assessments infeasible. Since the master-base algorithm scheme is primarily a theoretical device for handling dependence, we omit the master-level hierarchical elimination in our implementations and focus on the base algorithms, which align with industry practices.

To highlight the importance of capturing the baseline outcomes and estimating treatment effects properly, we will examine the numerical performance of four algorithms: (1) \texttt{LinUCB} that only updates the reward model when the bidder wins and observes the winning outcome $v_{t,1}$ \citep{li2010contextual}; (2) \texttt{Base-NoVR} that implements Algorithm \ref{alg:base_alg_te} but removes all variance reduction components (replacing the weights in linear regression by $1$ and sticking to a single UCB in \eqref{eq:UCB0}); (3) \texttt{Base-1UCB} that implements the weighted ridge regression in Algorithm \ref{alg:base_alg_te} but sticks to a single UCB in \eqref{eq:UCB0}; and (4) \texttt{Ours} that implements Algorithm \ref{alg:base_alg_te} in full.

\begin{figure}[t]
  \centering
  \begin{subfigure}[t]{0.47\linewidth}
    \centering
    \includegraphics[width=\linewidth]{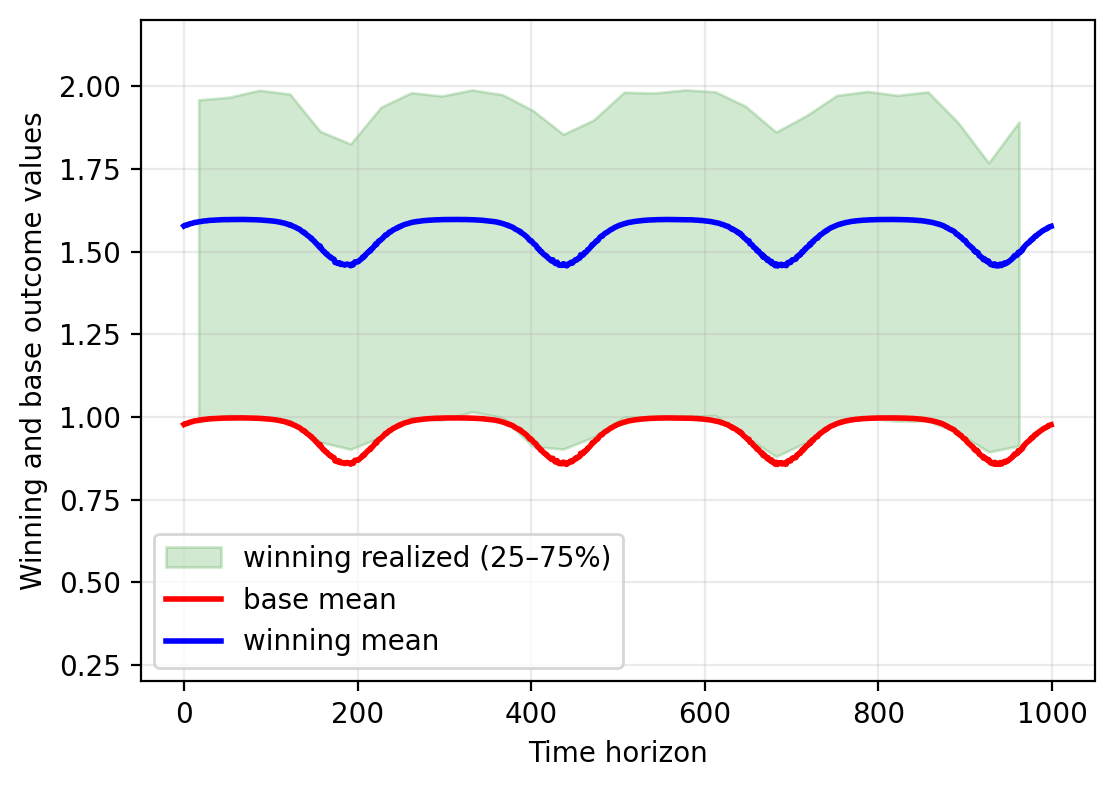}
    \caption{\small Pattern of periodic outcomes}
    \label{fig:periodic_outcomes}
  \end{subfigure}\hfill
  \begin{subfigure}[t]{0.5\linewidth}
    \centering
    \includegraphics[width=\linewidth]{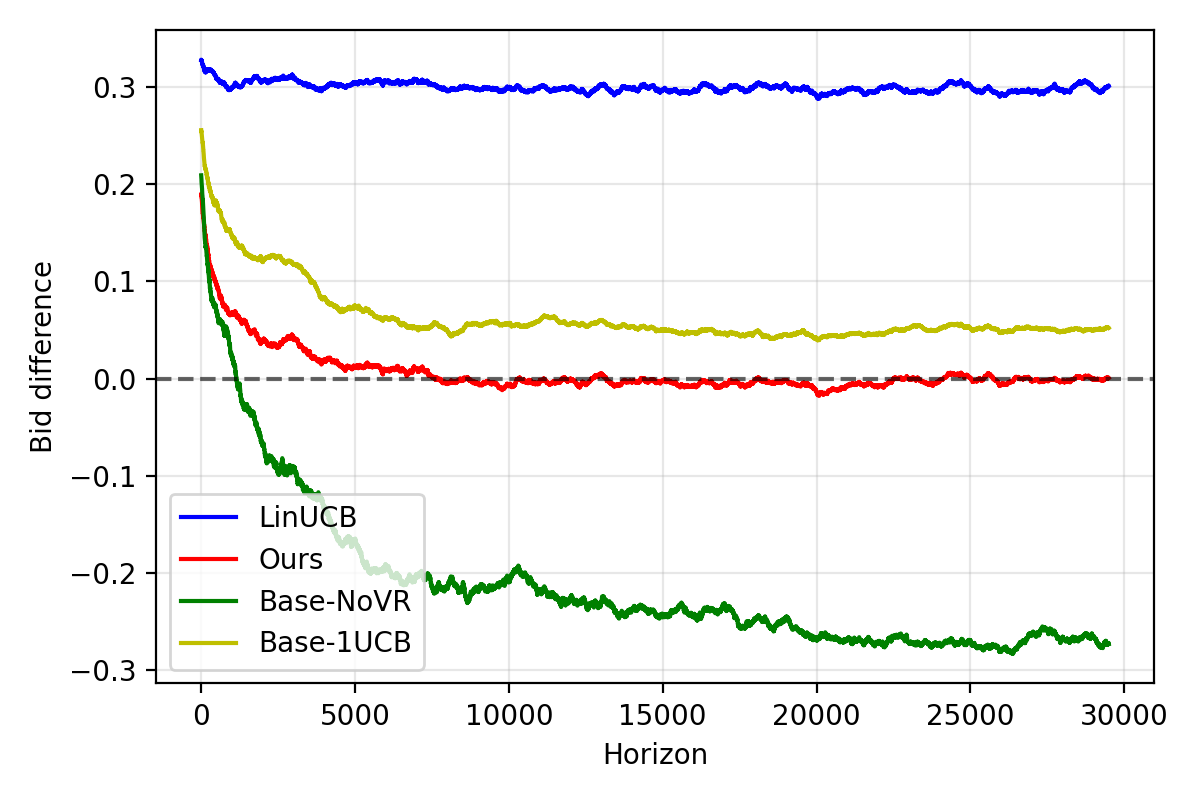}
    \caption{\small Pattern of bidding trajectories}
    \label{fig:bid_traj}
  \end{subfigure}
  \vspace{0.5em}
  \caption{\raggedright Plot (a) presents the periodic pattern of the means of the baseline and the winning outcomes, and the latter's variance, respectively. Plot (b) presents the trajectories of the bid difference $b_t - b_t^*$ of the four algorithms compared to the optimal bid, smoothed over a sliding window of width $500$ for visualization.}
  \label{fig:twoplots}
\end{figure}

For the purpose of demonstration and to showcase the root of failure of existing methods, we consider a simple setup. The underlying parameter $\btheta_* = [\btheta_*(1), \btheta_*(2:d)]$ and $\bx_t = [1, \bx_t(2:d)]$ are constructed by first drawing $\btheta_*(2:d)$ and $\bx_t(2:d)$ from the $(d-1)$-dimensional isotropic Gaussian $\mathcal{N}_{d-1}(\boldsymbol{0}, \bI)$ and then normalizing to guarantee a unit norm, with $\btheta_*(1)=0.6$ and $d=11$. The baseline outcome $v_{t,0} = \sigma(2+\sin(bt) + \cos(\boldsymbol\beta^\top \bx_t))$ is a deterministic periodic (in time) quantity, where $\sigma$ is the sigmoid function, $b = \frac{\pi}{125}$ indicates a length-$250$ period, and $\boldsymbol{\beta}\sim\mathcal{N}_{d}(\boldsymbol{0}, \bI)$ and normalized; its mean $\E_{\bx_t}[v_{t,0}]$ is illustrated in Figure \ref{fig:periodic_outcomes}. The periodic pattern is set to reflect the potentially periodic market behaviors. The winning outcome is $v_{t,1} = v_{t,0} + \varepsilon_t$ where $\varepsilon_t\sim\mathrm{Bern}(\btheta_*^\top \bx_t)$ is a Bernoulli variable. Finally, the HOB is drawn from an i.i.d. Beta distribution $M_t\sim\mathrm{Beta}(5,7)$ and is always revealed to the bidders.

Figure \ref{fig:bid_traj} compares the bidding trajectories of one run across the four algorithm variants. In particular, it plots the smoothed bid difference $b_t-b_t^*$ over time over a sliding window, for the sake of visualization. It is clear that existing algorithms like \texttt{LinUCB} can consistently overbid when a nonzero baseline outcome exists and yet is overlooked, because they mistake the winning outcome $v_{t,1}$ for the marginal gain $v_{t,1}-v_{t,0}$ of an ad opportunity. In comparison, our algorithm is able to identify the treatment effect. In addition, our variance reduction ingredients make our algorithm quickly converge to bidding near the optimal bid $b_t^*$.

\begin{figure}[t]
  \centering
  \begin{subfigure}[t]{0.35\linewidth}
    \centering
    \includegraphics[width=\linewidth]{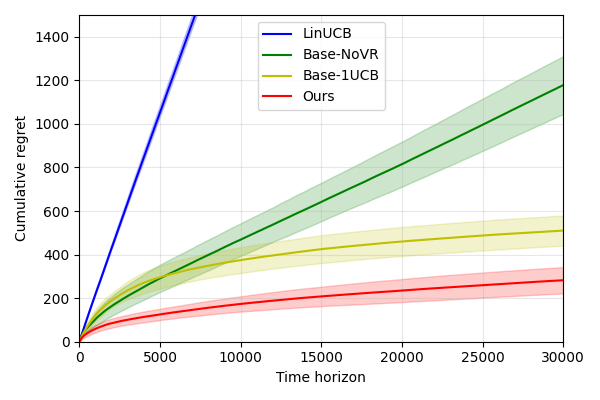}
    \caption{\small Pseudo regret over time}
    \label{fig:regrets}
  \end{subfigure}\hfill
  \begin{subfigure}[t]{0.64\linewidth}
    \centering
    \includegraphics[width=\linewidth]{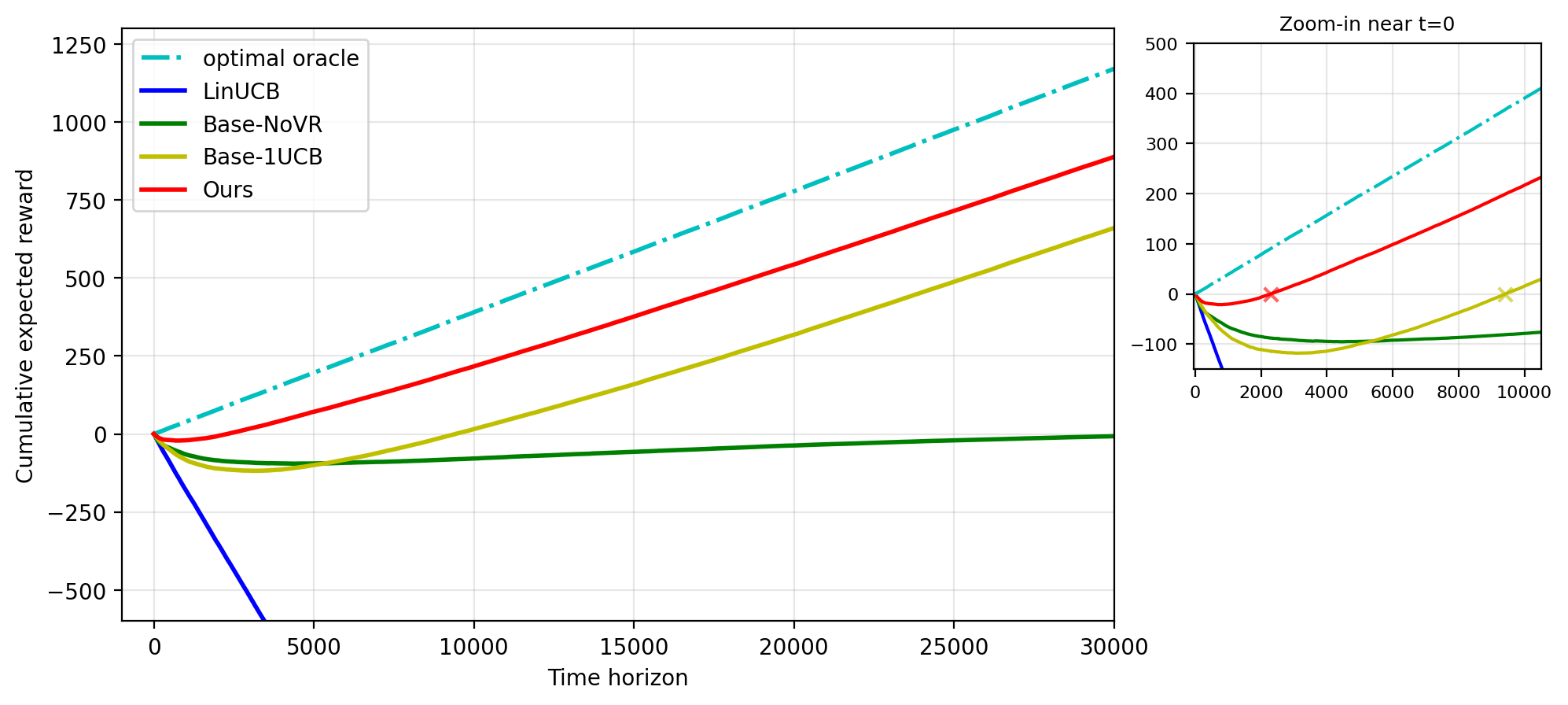}
    \caption{\small Reward over time}
    \label{fig:reward}
  \end{subfigure}
  \vspace{0.5em}
  \caption{\raggedright Plot (a) presents the pseudo regret of each of the four algorithms over a horizon $T=30,000$, averaged over five independent runs. The shaded region stands for one standard error. Plot (b) presents the corresponding cumulative reward and that achieved by the optimal oracle $(b_t^*)_t$, averaged over five runs. It zooms in around the initial time $t=0$ and marks when the reward becomes positive (i.e. breaks even).}
  \label{fig:performance_plots}
\end{figure}

To evaluate the numerical performances of these algorithms, Figure \ref{fig:performance_plots} compares the regrets and the cumulative rewards over a time horizon $T=30,000$. The results are averaged across five independent runs for the sake of robustness. \texttt{LinUCB} suffers a linear regret due to consistent overbidding. On the other hand, while every other algorithm achieves a sublinear convergence, \texttt{Ours} demonstrates a superior learning performance due to successful variance reduction with the weighted least squares (WLS) and the ``better of two UCB'' idea. The zoom-in in Figure \ref{fig:reward} indicates that the break-even points for the three algorithms are approximately $t=2,000, 9,000,$ and $30,000$, respectively.

\section{Conclusion and Future Directions}\label{sec:conclusion}
In this paper, we formulate the problem of joint value estimation and bidding in repeated FPAs into a special instance of causal inference, and establish regret bounds for the challenging scenario without overlap conditions or smoothness assumptions on the potential outcomes. Under both full-information and binary HOB feedback, our results establish near-optimal regret rates. Our work can be generalized in various directions. First, one can study explicit models for $M_t$ other than the i.i.d. and linear models in \cref{sec:main_results_full}, such as parametric families, and models beyond i.i.d. under the binary HOB feedback. Second, it would be interesting to generalize this framework to the semi-transparent model in \citep{cesa2024role}, where the bidder only observes the HOB when she loses. Another direction is to consider the delayed feedback, since the auction outcome may take time to manifest. For example, the bidder's win-loss of an ad opportunity can change the number of customers visiting her store per week, which is only revealed to the bidder after an entire week.

\section*{Acknowledgements}
This work is generously supported by the NSF award 2106508.


\clearpage 

\bibliographystyle{alpha}
\bibliography{Preprint.bib}
\clearpage 
\appendix


The appendices are organized as follows. In \cref{app:adv_model}, we present a near-optimal algorithm under a static regret where no causal inference is involved. Then \cref{app:proofs} lists the omitted proofs for the results in this work. The complementary $\Omega(\sqrt{dT})$ lower bound (i.e. \cref{thm:lower-bound}) is provided in \cref{app:lin_lower_bound}. Finally, auxiliary lemmata and relevant results on FPAs are shown in \cref{app:aux_lem}.

\section{Bypassing Estimations under Static Regret}\label{app:adv_model}

To contrast the difficulty of estimations under the dynamic regret defined in \eqref{eq:reg_def_lte}, in this section, we consider a weaker benchmark and compete against the best \textit{fixed} bid under the \textit{full-information} HOB feedback. This benchmark and this notion of regret do not capture the dynamic nature of contextual information. Consequently, it is possible to bypass the need for estimating the treatment effect $\E[v_{t,1}-v_{t,0}]$ and estimating the HOB CDF $G_t$. Before we proceed, note that even the best fixed bid can incur linear regret when competing against the context-dependent oracle in \eqref{eq:reg_def_lte}.

To start, we formally define the static regret as:
\begin{equation*}\label{eq:def_regret_adv}
\begin{split}
\Reglte_{\text{static}}(\pi) &= \sup_{\{(v_{t,1},v_{t,0},G_t)\}_{t=1}^T }\sup_{b_*\in[0,1]}\sum_{t=1}^T \left(\E[r_t(b_*)] - \E[r_t(b_t)]\right) \\
&= \sup_{\{(v_{t,1},v_{t,0},G_t)\}_{t=1}^T }\sup_{b_*\in[0,1]}\sum_{t=1}^T \left(G_t(b_*)(v_{1,t}-v_{0,t}-b_*) - \E[G_t(b_t)(v_{1,t}-v_{0,t}-b_t)] \right).
\end{split}
\end{equation*}
As a first step, we consider a finite bid set $\Bcal\subseteq [0,1]$ that comes from some sort of discretization of the bid space. We construct EXP3-FPA-TE (Algorithm~\ref{alg:EXP3_FPA}), an extension of the EXP3-FPA algorithm in \cite{cesa2024role}, that achieves $\widetildeO(\sqrt{T})$ expected regret when competing against the best fixed bid within this finite subset $\Bcal$. Then EXP3-FPA-TE can be coupled with a simple discretization to achieve near-optimal regret on the continuous bid space $[0,1]$. When there is indeed no context ($\bx_t=\varnothing$ for all $t$) and the benchmark is static, a $\Omega(\sqrt{T})$ lower bound is implied from the case with no treatment effect, i.e. $v_{t,0}\equiv 0$ for all $t$ \citep{cesa2024role}.

\begin{algorithm}[h!]\caption{EXP3-FPA-TE (EXP3 for first-price auctions with treatment effects)}
\label{alg:EXP3_FPA}
\textbf{Input:} time horizon $T$, finite bid set $\Bcal\subseteq [0,1]$, learning rate $\eta\in\parr{0,\frac{1}{2}}$.

\textbf{Initialize:} set exponential weight $w_{t,b}=1$ for each $b\in\Bcal$.

\For{$t=1$ \KwTo $T$}
{   
Set probability $p_{t}(b) = (1-2\eta)\frac{w_{t,b}}{\sum_{b'\in\Bcal}w_{t,b'}} + \eta\indic[b=b_{\min}] + \eta\indic[b=b_{\max}]$.

Draw bid $b_t\sim p_{t}$.

Make bid $b_t$, observe $M_t$, and build payoff estimator for every $b\in\Bcal$:
\[
\widehat{r}_t(b) = \frac{\indic[b_t\geq M_t]}{\sum_{b'\geq M_t}p_t(b')}\indic[b\geq M_t](v_{1,t} - b) + \frac{\indic[b_t< M_t]}{\sum_{b'< M_t}p_t(b')}\indic[b< M_t]v_{t,0}.
\]

For every $b\in\Bcal$, update $w_{t+1,b} = w_{t,b}\exp(\eta \hr_t(b))$.
}

\end{algorithm}

EXP3-FPA-TE maintains an exponential weight for each bid in the finite subset. At each time, it draws a bid based on these weights and an $\eta$-exploration on the largest and the smallest bid. Then the algorithm builds an unbiased estimator for the payoff for every candidate bid $b\in\Bcal$ and updates the weights accordingly.

\begin{theorem}\label{thm:exp3_fpa_regret_upper_bound}
When $\log|\Bcal|< (e-1)T$, EXP3-FPA-TE with $\eta = \sqrt{\frac{\log|\Bcal|}{(4e-2)T}}\in\parr*{0,\frac{1}{2}}$ achieves expected regret
\[
\E\parq*{\max_{b_*\in\Bcal}\sum_{t=1}^T r_t(b_*) - \sum_{t=1}^T\sum_{b\in\Bcal}p_t(b) r_t(b)} \leq \sqrt{(4e-2)\log|\Bcal|T}.
\]
\end{theorem}
\begin{proof}{Proof.}
Note that $\eta\hr_t(b)\leq 1$ for every $t\in[T]$ and $b\in\Bcal$, thanks to the $\eta$ amount of exploration on $b_{\min}$ and $b_{\max}$. Let $W_t = \sum_{b\in\Bcal}w_{t,b}$. Following the standard analysis for EXP3, for each $t\in[T]$,
\begin{align*}
\frac{W_{t+1}}{W_t} &= \sum_{b\in\Bcal}\frac{w_{t,b}}{W_t}\exp(\eta\hr_t(b))\\
&\leq \sum_{b\in\Bcal}\frac{w_{t,b}}{W_t}\parr*{1 + \eta\hr_t(b) + (e-2)\eta^2\hr_t(b)^2}\\
&= 1 + \eta\sum_{b\in\Bcal}\frac{w_{t,b}}{W_t}\parr*{\hr_t(b) + (e-2)\eta\hr_t(b)^2}\\
&\leq 1 + \frac{\eta}{1-2\eta}\sum_{b\in\Bcal}p_t(b)\parr*{\hr_t(b) + (e-2)\eta\hr_t(b)^2}
\end{align*}
Summing over the time horizon, we have
\begin{align*}
\log W_{T+1} - \log W_1 = \sum_{t=1}^T\log \frac{W_{t+1}}{W_t} \leq \frac{\eta}{1-2\eta}\sum_{t=1}^T\sum_{b\in\Bcal}p_t(b)\parr*{\hr_t(b) + (e-2)\eta\hr_t(b)^2}.
\end{align*}
Fix any $b_*\in\Bcal$. Since $W_1 = |\Bcal|$ and $W_{T+1} = \sum_{b\in\Bcal}\exp\parr*{\eta\sum_{t=1}^T\hr_t(b)}\geq \exp\parr*{\eta\sum_{t=1}^T\hr_t(b_*)}$, 
\begin{equation}\label{eq:exp3_ineq}
\sum_{t=1}^T\hr_t(b_*) - \frac{\log|\Bcal|}{\eta} \leq \frac{1}{1-2\eta}\sum_{t=1}^T\sum_{b\in\Bcal}p_t(b)\parr*{\hr_t(b) + (e-2)\eta\hr_t(b)^2}.
\end{equation}
Note that the quadratic payoff term can be bounded as follows.
\begin{align*}
\frac{1}{2}\hr_t(b)^2 &\leq \frac{\indic[b_t\geq M_t]}{\parr*{\sum_{b'\geq M_t}p_t(b')}^2}\indic[b\geq M_t](v_{1,t} - b)^2 + \frac{\indic[b_t< M_t]}{\parr*{\sum_{b'< M_t}p_t(b')}^2}\indic[b< M_t]v_{t,0}^2\\
&\leq \frac{\indic[b_t\geq M_t]}{\parr*{\sum_{b'\geq M_t}p_t(b')}^2}\indic[b\geq M_t] + \frac{\indic[b_t< M_t]}{\parr*{\sum_{b'< M_t}p_t(b')}^2}\indic[b< M_t].
\end{align*}
Recall $\E_t[\cdot] = \E[\cdot|\mathcal{F}_t]$ w.r.t. the filtration at time $t$. For every $t\in[T]$ and $b\in\Bcal$, it holds that $\E_t[\hr_t(b)] = r_t(b)$ and 
\begin{align*}
\E_t\parq*{\sum_{b\in\Bcal}p_t(b)\hr_t(b)^2} &\leq \E_t\parq*{2\sum_{b\in\Bcal}p_t(b)\frac{\indic[b_t\geq M_t]\indic[b\geq M_t]}{\parr*{\sum_{b'\geq M_t}p_t(b')}^2} + p_t(b)\frac{\indic[b_t< M_t]\indic[b< M_t]}{\parr*{\sum_{b'< M_t}p_t(b')}^2}}\\
&= \E_t\parq*{2\sum_{b\in\Bcal}p_t(b)\frac{\indic[b\geq M_t]}{\sum_{b'\geq M_t}p_t(b')} + p_t(b)\frac{\indic[b< M_t]}{\sum_{b'< M_t}p_t(b')}}\\
&= 4.
\end{align*}
Now taking expectation of \eqref{eq:exp3_ineq} and rearranging give us, for any $b_*\in\Bcal$,
\begin{align*}
\E\parq*{\sum_{t=1}^T r_t(b_*) - \sum_{t=1}^T\sum_{b\in\Bcal}p_t(b) r_t(b)} &\leq \frac{\log|\Bcal|}{\eta} + (4e-2)\eta T.
\end{align*}
\end{proof}

To handle the continuous bid space $[0,1]$, we consider a simple discretization $\Bcal = \bra*{\frac{k}{T}:k\in[T]}$. Thanks to the Lipschitzness of the CDF $G_t$ in Assumption~\ref{ass:Gt}, we are able to approximate the optimal fixed bid in $[0,1]$ with the optimal bid in the discretization $\Bcal$, which leads to the following result.
\begin{corollary}
Let $\Bcal = \bra*{\frac{k}{T}:k\in[T]}\subseteq[0,1]$. Let $\pi$ be the resulting bidding strategy of \cref{alg:EXP3_FPA} running on this discretization with $\eta = \sqrt{\frac{\log(T)}{(4e-2)T}}$. Then under Assumptions \ref{ass:noconfounding}--\ref{ass:Gt}, the policy $\pi$ achieves expected regret
\[
\Reglte_{\text{static}}(\pi) =O\parr*{\sqrt{T\log T}}.
\]
\end{corollary}
\begin{proof}{Proof.}
Let $b_* = \argmax_{b\in[0,1]}\sum_{t=1}^T\E[r_t(b)]$ be the optimal fixed bid in the continuous space. By Assumption~\ref{ass:Gt}, for any bid $b\in[0,1]$ and any $\varepsilon\in(0,1)$, we have
\[
\abs*{\E[r_t(b+\varepsilon)] - \E[r_t(b)]} = \abs{G_t(b+\varepsilon)(v-b-\varepsilon) - G_t(b)(v-b)} \leq (2L+1)\varepsilon
\]
where $v=v_{t,1}-v_{t,0}\in[-1,1]$. Then by Theorem~\ref{thm:exp3_fpa_regret_upper_bound}, 
\begin{align*}
\Reglte_{\text{static}}(\pi) &= \E\parq*{\sum_{t=1}^Tr_t(b_*) - \sum_{t=1}^T\sum_{b\in\Bcal}p_t(b)r_t(b)}\\
&\leq 2L+1 + \E\parq*{\max_{b_*'\in\Bcal}\sum_{t=1}^T r_t(b') - \sum_{t=1}^T\sum_{b\in\Bcal}p_t(b) r_t(b)}\\
&\leq 2L+1+\sqrt{(4e-2)T\log T}.
\end{align*}
\end{proof}

\section{Proofs of Main Lemmas}\label{app:proofs}

\subsection{Proof of \cref{lem:bias-variance} and \ref{lem:bias-variance_weaker}}
We will present the proof of \cref{lem:bias-variance} below.

\begin{repeatlemma}[\cref{lem:bias-variance}.]\label{lem:bias-variance_app}
Suppose \cref{ass:est_oracle_bern} holds, i.e. $|\widehat{G}_t(b) - G_t(b)|\le \delta_t\sqrt{G_t(b)(1-G_t(b))} + \delta_t^2$. Then
\begin{align*}
\abs*{\E[\widetilde{e}_t(b)] - \btheta_*^\top \bx_t} &\le 20\delta_t\sigma_t(b), \\
\Var(\widetilde{e}_t(b)) &\le 16\sigma_t(b)^2, 
\end{align*}
where $\widetilde{e}_t(b)$ is defined in \eqref{eq:truncated_estimator_te} and the quantity $\sigma_t(b)$ is defined as
\begin{equation}
    \sigma_t(b) = \frac{1}{\sqrt{ \widehat{G}_t(b)(1-\widehat{G}_t(b)) }}. 
\end{equation}
\end{repeatlemma}
\begin{proof}{Proof.}
\textbf{(Bias)}
Denote the bias by $\zeta_t(b) = \E[\widetilde{e}_t(b)] - \btheta_*^{\top}\bx_t$. We decompose the bias as follows:
By taking expectation w.r.t. $M_t$, we have
\begin{align*}
\E_{M_t}\parq*{\widetilde{e}_t(b)} - \parr*{v_{t,1}-v_{t,0}} &= \parr*{\frac{G_t(b)}{\max\bra*{\delta_t^2, \widehat{G}_t(b)}}-1}v_{t,1} - \parr*{\frac{1-G_t(b)}{\max\bra*{\delta_t^2, 1-\widehat{G}_t(b)}} -1}v_{t,0}.
\end{align*}
Since $v_{t,1}\in[0,1]$, the absolute value of the first term is bounded by
\begin{align*}
\zeta_{t,1}(b) \equiv \abs*{\frac{G_t(b)}{\max\bra*{\delta_t^2, \widehat{G}_t(b)}}-1}.
\end{align*}
Similarly, the second term is bounded by
\begin{align*}
\zeta_{t,0}(b)\equiv \abs*{\frac{1-G_t(b)}{\max\bra*{\delta_t^2, 1-\widehat{G}_t(b)}} -1}
\end{align*}
Finally, by Jensen's inequality and triangle inequality, for any $b$,
\begin{align*}
\abs*{\E\parq*{\widetilde{e}_t(b)} - \btheta_*^{\top}\bx_t} &= \abs*{\E_{v_{t,1}, v_{t,0}}\parq*{\E_{M_t}\parq*{\widetilde{e}_t(b)} - (v_{t,1}-v_{t,0})}}\\
&\leq \E_{v_{t,1}, v_{t,0}}\abs*{\E_{M_t}\parq*{\widetilde{e}_t(b)} - (v_{t,1}-v_{t,0})}\\
&\leq \zeta_{t,1}(b) + \zeta_{t,0}(b).
\end{align*}
Next, we will bound $\zeta_t(b)^2$ to obtain the first inequality in the claim. Note that $\zeta_t(b)^2\leq 4\max\bra{\zeta_{t,1}(b)^2, \zeta_{t,0}(b)^2}$ by the decomposition above. We deal with each case separately. First consider $\zeta_{t,1}(b)^2$: when $\widehat{G}_t(b) \geq 6\delta_t^2$, by Assumption~\ref{ass:est_oracle_bern} and some algebra, we have $G_t(b) \geq 3\delta_t^2$, and thereby straightforward computation leads to
\begin{align*}
\sigma_t(b)^{-2}\zeta_{t,1}(b)^2 &= \widehat{G}_t(b)\parr*{1-\widehat{G}_t(b)}\frac{(G_t(b)-\widehat{G}_t(b))^2}{\widehat{G}_t(b)^2}\\
&\leq \widehat{G}_t(b)\frac{(G_t(b)-\widehat{G}_t(b))^2}{\widehat{G}_t(b)^2} = \frac{(G_t(b)-\widehat{G}_t(b))^2}{\widehat{G}_t(b)}\\
&\stepa{\leq} \frac{(\delta_t\sqrt{G_t(b)}+\delta_t^2)^2}{\widehat{G}_t(b)} \\
&\leq 2\delta_t^2\frac{G_t(b)}{\widehat{G}_t(b)} + 2\frac{\delta_t^4}{8\delta_t^2}\\
&\stepb{\leq} 2\delta_t^2\frac{G_t(b)}{G_t(b) - \delta_t\sqrt{G_t(b)}-\delta_t^2} + \frac{\delta_t^2}{4}\\
&\stepc{\leq} 40\delta_t^2 + \frac{\delta_t^2}{4}  \leq 41\delta_t^2 \numberthis\label{eq:bias_var_bound_eq1}
\end{align*}
where (a) and (b) are by Assumption~\ref{ass:est_oracle_bern} and (c) uses $\delta_t\sqrt{G_t(b)}+\delta_t^2 \leq \frac{19}{20}G_t(b)$ when $G_t(b)\geq 3\delta_t^2$. On the other hand, when $\widehat{G}_t(b) < 6\delta_t^2$, again by Assumption~\ref{ass:est_oracle_bern} we have $G_t(b) < 11\delta_t^2$, and
\begin{align*}
\sigma_t(b)^{-2}\zeta_{t,1}(b)^2 &\leq \widehat{G}_t(b)\parr*{1-\widehat{G}_t(b)}\frac{(G_t(b)-\max\bra{\delta_t^2,\widehat{G}_t(b)})^2}{\max\bra{\delta_t^4, \widehat{G}_t(b)^2}}.
\end{align*}
If $\widehat{G}_t(b)\le \delta_t^2$, it reduces to $\sigma_t(b)^{-2}\zeta_{t,1}(b)^2 \le 100\delta_t^2$. Otherwise, 
\begin{equation}\label{eq:bias_var_bound_eq2}
\sigma_t(b)^{-2}\zeta_{t,1}(b)^2 \le \frac{\parr*{G_t(b)-\widehat{G}_t(b)}^2}{\widehat{G}_t(b)} \le \frac{100\delta_t^4}{\delta_t^2} \le 100\delta_t^2.
\end{equation}
A symmetric argument applies to $\sigma_t(b)^{-2}\zeta_{t,0}(b)^2$. Together we obtain the inequality
\begin{align*}
|\zeta_t(b)| = |\E[\widetilde{e}_t(b)]-\btheta_*^{\top}\bx_t| &\leq 2\max\bra{\zeta_{t,1}(b), \zeta_{t,0}(b)}\\
&\leq 2\sqrt{100\delta_t^2\sigma_t(b)^{2}} =  20\sigma_t(b)\delta_t.
\end{align*}
\textbf{(Variance)} We will bound the variance of $\widetilde{e}_t(b)$ by its second moment. Toward this end, we look at the first term in the estimator:
\begin{align*}
\parr*{\frac{\indic[b\geq M_t]}{\max\bra*{\delta_t^2, \widehat{G}_t(b)}}v_{t,1}}^2 \leq \parr*{\frac{\indic[b\geq M_t]}{\max\bra*{\delta_t^2, \widehat{G}_t(b)}}}^2 \leq \widehat{G}_t(b)^{-1}\frac{\indic[b\geq M_t]}{\max\bra*{\delta_t^2, \widehat{G}_t(b)}}.
\end{align*}
Similarly, 
\begin{align*}
\parr*{\frac{\indic[b< M_t]}{\max\bra*{\delta_t^2, 1-\widehat{G}_t(b)}}v_{t,0}}^2 \leq \parr*{1-\widehat{G}_t(b)}^{-1}\frac{\indic[b< M_t]}{\max\bra*{\delta_t^2, 1-\widehat{G}_t(b)}}.
\end{align*}
Then we have the second moment
\begin{align*}
\Var\parr*{\widetilde{e}_t(b)} &\leq \E\parq*{\widetilde{e}_t(b)^2}\\
&\leq 2\E\parq*{\parr*{\frac{\indic[b\geq M_t]}{\max\bra*{\delta_t^2, \widehat{G}_t(b)}}v_{t,1}}^2} + 2 \E\parq*{\parr*{\frac{\indic[b< M_t]}{\max\bra*{\delta_t^2, 1-\widehat{G}_t(b)}}v_{t,0}}^2}\\
&\leq 2\widehat{G}_t(b)^{-1}\E\parq*{\frac{\indic[b\geq M_t]}{\max\bra*{\delta_t^2, \widehat{G}_t(b)}}} + 2\parr*{1-\widehat{G}_t(b)}^{-1}\E\parq*{\frac{\indic[b< M_t]}{\max\bra*{\delta_t^2, 1-\widehat{G}_t(b)}}}\\
&= 2\widehat{G}_t(b)^{-1}\frac{G_t(b)}{\max\bra*{\delta_t^2, \widehat{G}_t(b)}} + 2\parr*{1-\widehat{G}_t(b)}^{-1}\frac{1-G_t(b)}{\max\bra*{\delta_t^2, 1-\widehat{G}_t(b)}}\\
&\stepa{\leq} 8\widehat{G}_t(b)^{-1} + 8\parr*{1-\widehat{G}_t(b)}^{-1}\\
&\leq 16\max\bra*{\widehat{G}_t(b)^{-1}, \parr*{1-\widehat{G}_t(b)}^{-1}}
\end{align*}
where (a) bounds each fraction by $4$ due to Assumption~\ref{ass:est_oracle_bern} and some algebra. To see why (a) holds, consider the first fraction $\frac{G_t(b)}{\max\bra*{\delta_t^2, \widehat{G}_t(b)}}$ and when $G_t(b)>4\delta_t^2$ (otherwise it holds trivially). Then by \cref{ass:est_oracle_bern}, $\widehat{G}_t(b) \ge G_t(b) - \sqrt{G_t(b)}\delta_t - \delta_t^2 > \delta_t^2$ and
\begin{equation*}
\frac{G_t(b)}{\widehat{G}_t(b)} \le 1 + \frac{\sqrt{G_t(b)}\delta_t + \delta_t^2}{\widehat{G}_t(b)} \le 1 + \frac{\frac{3}{2}\sqrt{G_t(b)}\delta_t}{G_t(b) - \sqrt{G_t(b)}\delta_t - \delta_t^2} \le 1 + \frac{\frac{3}{2}\sqrt{G_t(b)}\delta_t}{\sqrt{G_t(b)}\parr{\sqrt{G_t(b)}-\frac{3}{2}\delta_t}} \le 4
\end{equation*}
where the last inequality follows $\sqrt{G_t(b)}-\frac{3}{2}\delta_t \ge 2\delta_t-\frac{3}{2}\delta_t = \frac{\delta_t}{2}$.
\end{proof}

Then we present the bias-variance trade-off of the IPW estimator under Abstraction \ref{ass:est_oracle_weaker_bern}.

\begin{repeatlemma}[\cref{lem:bias-variance_weaker}]
Suppose $|\widehat{G}_t(b) - G_t(b)|\le \delta_t(b)\sqrt{G_t(b)(1-G_t(b))} + \delta_t(b)^2 + \xi$ for every $b\in[0,1]$. Then there exist absolute constants $c, c'>0$ such that
\begin{equation*}
\begin{cases}
    \abs*{\E[\widetilde{e}_t(b)] - \btheta_*^\top \bx_t} \le c(\delta_t(b)\sigma_t(b) + \xi \sigma_t(b)^2)\le c\parr*{\delta_t(b) + \xi}\sigma_t(b)^2 \\
    \Var(\widetilde{e}_t(b)) \le c'\max\{\xi\sigma_t(b)^4, \sigma_t(b)^2\} \le c'\sigma_t(b)^4.
\end{cases} \qquad \text{where }\sigma_t(b) \coloneqq \parr*{ \widehat{G}_t(b)(1-\widehat{G}_t(b))}^{-\frac{1}{2}}.
\end{equation*}
In particular, $c= 15$ and $c'= 80$.
\end{repeatlemma}
\begin{proof}{Proof.}
\textbf{(Bias)}
Denote the bias by $\zeta_t(b) = \E_t[\widetilde{e}_t(b)] - \btheta_*^{\top}\bx_t$. We decompose the bias as follows:
By taking expectation w.r.t. $M_t$, we have
\begin{align*}
\E_{M_t}\parq*{\widetilde{e}_t(b)} - \parr*{v_{t,1}-v_{t,0}} &= \parr*{\frac{G_t(b)}{\max\bra*{\delta_t^2, \widehat{G}_t(b)}}-1}v_{t,1} - \parr*{\frac{1-G_t(b)}{\max\bra*{\delta_t^2, 1-\widehat{G}_t(b)}} -1}v_{t,0}.
\end{align*}
Since $v_{t,1}\in[0,1]$, the absolute value of the first term is bounded by
\begin{align*}
\zeta_{t,1}(b) \equiv \abs*{\frac{G_t(b)}{\max\bra*{\delta_t^2, \widehat{G}_t(b)}}-1}.
\end{align*}
Similarly, the second term is bounded by
\begin{align*}
\zeta_{t,0}(b)\equiv \abs*{\frac{1-G_t(b)}{\max\bra*{\delta_t^2, 1-\widehat{G}_t(b)}} -1}
\end{align*}
Finally, by Jensen's inequality and triangle inequality, for any $b$,
\begin{align*}
\abs*{\E\parq*{\widetilde{e}_t(b)} - \btheta_*^{\top}\bx_t} &= \abs*{\E_{v_{t,1}, v_{t,0}}\parq*{\E_{M_t}\parq*{\widetilde{e}_t(b)} - (v_{t,1}-v_{t,0})}}\\
&\leq \E_{v_{t,1}, v_{t,0}}\abs*{\E_{M_t}\parq*{\widetilde{e}_t(b)} - (v_{t,1}-v_{t,0})}\\
&\leq \zeta_{t,1}(b) + \zeta_{t,0}(b).
\end{align*}
Next, we will bound $\zeta_t(b)^2$ to obtain the first inequality in the claim. Note that $\zeta_t(b)^2\leq 4\max\bra{\zeta_{t,1}(b)^2, \zeta_{t,0}(b)^2}$ by the decomposition above. Consider $\zeta_{t,1}(b)^2$. When $\widehat{G}_t(b) \ge 6\delta_t(b)^2$, straightforward computation gives
\begin{align*}
\zeta_{t,1}(b)^2 &= \frac{(G_t(b)-\widehat{G}_t(b))^2}{\widehat{G}_t(b)^2}\\
&\stepa{\leq} \frac{(\delta_t\sqrt{G_t(b)}+\delta_t^2 + \xi)^2}{\widehat{G}_t(b)^2}\\
&\le 3\delta_t^2\frac{G_t(b)}{\widehat{G}_t(b)^2} + 3\frac{\delta_t^4}{\widehat{G}_t(b)^2} + 3\frac{\xi^2}{\widehat{G}_t(b)^2}\\
&\stepb{\le} 6\delta_t(b)^2\sigma_t(b)^2 + 3\xi^2\sigma_t(b)^4 \numberthis\label{eq:bias_var_weaker_eq1}
\end{align*}
where (a) is by Abstraction~\ref{ass:est_oracle_weaker_bern} and (b) applies \eqref{eq:bias_var_bound_eq1} on the first two terms and $\delta_t(b)^2 \le \delta_t(b)$ and $\widehat{G}_t(b)^{-1} \le \sigma_t(b)^2$ on the last term.

Now consider the case when $\widehat{G}_t(b) < 6\delta_t(b)^2$. First, suppose $G_t(b) \ge 5\xi$. By Assumption~\ref{ass:est_oracle_weaker_bern} and some algebra, we have $G_t(b) < 20\delta_t(b)^2$, and thus following the computation in \eqref{eq:bias_var_bound_eq2},
\begin{equation}\label{eq:bias_var_weaker_eq2}
\zeta_{t,1}(b)^2 \leq \frac{(G_t(b)-\max\bra{\delta_t(b)^2,\widehat{G}_t(b)})^2}{\max\bra{\delta_t(b)^4, \widehat{G}_t(b)^2}} \le 400\delta_t(b)^2\sigma_t(b)^2.
\end{equation}
If instead $G_t(b) < 5\xi$, then we do a case study on $\widehat{G}_t(b)$. When $\widehat{G}_t(b) \le \delta_t(b)^2$, we have
\begin{equation}\label{eq:bias_var_weaker_eq3}
\zeta_{t,1}(b)^2 \le \frac{(5\xi -\delta_t(b)^2)^2}{\delta_t(b)^4} \stepc{\le} 50\xi^2\sigma_t(b)^4 + 2 \le 50\xi\sigma_t(b)^4 + 2\delta_t(b)^2\sigma_t(b)^2
\end{equation}
where (c) applies $(a+b)^2 \le 2a^2+2b^2$.
Otherwise $G_t(b) < 5\xi$ and $\delta_t(b)^2< \widehat{G}_t(b) < 6\delta_t(b)^2$, and we have
\begin{equation}\label{eq:bias_var_weaker_eq4}
\zeta_{t,1}(b)^2 \stepc{\le} \frac{50\xi^2 + 2\widehat{G}_t(b)^2}{\widehat{G}_t(b)^2} \le 50\xi^2\sigma_t(b)^4 + 2 \le 50\xi\sigma_t(b)^4 + 12\delta_t(b)^2\sigma_t(b)^2.
\end{equation}
By (\ref{eq:bias_var_weaker_eq1}--\ref{eq:bias_var_weaker_eq4}) and a symmetric argument on $\zeta_{t,0}(b)^2$, we obtain the inequality
\begin{align*}
|\zeta_t(b)| = |\E[\widetilde{e}_t(b)]-\btheta_*^{\top}\bx_t| &\leq 2\max\bra{\zeta_{t,1}(b), \zeta_{t,0}(b)} 
\le 2\sqrt{50}\xi\sigma_t(b)^2 + 2\sqrt{12}\delta_t(b)\sigma_t(b).
\end{align*}


\textbf{(Variance)} 
We will bound the variance of $\widetilde{e}_t(b)$ by its second moment. Toward this end, we look at the first term in the estimator:
\begin{align*}
\parr*{\frac{\indic[b\geq M_t]}{\max\bra*{\delta_t^2, \widehat{G}_t(b)}}v_{t,1}}^2 \leq \parr*{\frac{\indic[b\geq M_t]}{\max\bra*{\delta_t^2, \widehat{G}_t(b)}}}^2 \leq \widehat{G}_t(b)^{-1}\frac{\indic[b\geq M_t]}{\max\bra*{\delta_t^2, \widehat{G}_t(b)}}.
\end{align*}
Similarly, 
\begin{align*}
\parr*{\frac{\indic[b< M_t]}{\max\bra*{\delta_t^2, 1-\widehat{G}_t(b)}}v_{t,0}}^2 \leq \parr*{1-\widehat{G}_t(b)}^{-1}\frac{\indic[b< M_t]}{\max\bra*{\delta_t^2, 1-\widehat{G}_t(b)}}.
\end{align*}
Then we have the second moment
\begin{align*}
\Var\parr*{\widetilde{e}_t(b)} &\leq \E\parq*{\widetilde{e}_t(b)^2}\\
&\leq 2\E\parq*{\parr*{\frac{\indic[b\geq M_t]}{\max\bra*{\delta_t^2, \widehat{G}_t(b)}}v_{t,1}}^2} + 2 \E\parq*{\parr*{\frac{\indic[b< M_t]}{\max\bra*{\delta_t^2, 1-\widehat{G}_t(b)}}v_{t,0}}^2}\\
&\leq 2\widehat{G}_t(b)^{-1}\E\parq*{\frac{\indic[b\geq M_t]}{\max\bra*{\delta_t^2, \widehat{G}_t(b)}}} + 2\parr*{1-\widehat{G}_t(b)}^{-1}\E\parq*{\frac{\indic[b< M_t]}{\max\bra*{\delta_t^2, 1-\widehat{G}_t(b)}}}\\
&= 2\widehat{G}_t(b)^{-1}\frac{G_t(b)}{\max\bra*{\delta_t^2, \widehat{G}_t(b)}} + 2\parr*{1-\widehat{G}_t(b)}^{-1}\frac{1-G_t(b)}{\max\bra*{\delta_t^2, 1-\widehat{G}_t(b)}}
\numberthis\label{eq:truncated_ipw_variance_decompose}
\end{align*}
To proceed, we focus on the first term in \eqref{eq:truncated_ipw_variance_decompose}. 

First suppose $\widehat{G}_t(b) \ge 6\delta_t(b)^2$ and $G_t(b) \ge 5\xi$. By \cref{ass:est_oracle_weaker_bern} and some algebra, we have $G_t(b) \ge \frac{25}{9}\delta_t(b)^2$ and $\delta_t(b)\sqrt{G_t(b)} + \delta_t(b)^2 +\xi \le \frac{4}{5}G_t(b)$. Consequently, in this case
\begin{equation*}
\widehat{G}_t(b)^{-1}\frac{G_t(b)}{\max\bra*{\delta_t^2, \widehat{G}_t(b)}} = \widehat{G}_t(b)^{-1}\frac{G_t(b)}{\widehat{G}_t(b)} \le \widehat{G}_t(b)^{-1}\frac{G_t(b)}{G_t(b) - \delta_t(b)\sqrt{G_t(b)} - \delta_t(b)^2 -\xi} \le 5\widehat{G}_t(b)^{-1} \le 5\sigma_t(b)^2.
\end{equation*}
When $\widehat{G}_t(b) < 6\delta_t(b)^2$ and $G_t(b) \ge 5\xi$, \cref{ass:est_oracle_weaker_bern} and some computation gives again $G_t(b) < 20\delta_t(b)^2$. Then
\begin{equation*}\label{eq:bias_var_weaker_eq5}
\widehat{G}_t(b)^{-1}\frac{G_t(b)}{\max\bra*{\delta_t^2, \widehat{G}_t(b)}} \le \widehat{G}_t(b)^{-1}\frac{20\delta_t(b)^2}{\delta_t(b)^2} \le 20\sigma_t(b)^2.
\end{equation*}
Finally, if $G_t(b) < 5\xi$ in any case, we immediately have:
\begin{equation*}\label{eq:bias_var_weaker_eq6}
\widehat{G}_t(b)^{-1}\frac{G_t(b)}{\max\bra*{\delta_t^2, \widehat{G}_t(b)}} \le \widehat{G}_t(b)^{-1}\frac{5\xi}{\widehat{G}_t(b)} \le 5\xi\sigma_t(b)^4.
\end{equation*}
A symmetric argument on the second term in \eqref{eq:truncated_ipw_variance_decompose} gives
\[
\parr*{1-\widehat{G}_t(b)}^{-1}\frac{1-G_t(b)}{\max\bra*{\delta_t^2, 1-\widehat{G}_t(b)}} \le 20\sigma_t(b)^2 + 5\xi\sigma_t(b)^4.
\]
Then \eqref{eq:truncated_ipw_variance_decompose} is further bounded as follows:
\begin{align*}
\Var\parr*{\widetilde{e}_t(b)} 
\le 80\sigma_t(b)^2 + 20\xi\sigma_t(b)^4
\end{align*}
which completes the proof.
\end{proof}

\subsection{Proof of \cref{lem:WLS}}
\begin{repeatlemma}[\cref{lem:WLS}.]\label{lem:WLS_app}
Suppose $(v_{\tau,1},v_{\tau,0})_{\tau\in \Phi_t}$ are conditionally independent given $(\bx_\tau,b_\tau,M_\tau)_{\tau\in \Phi_t}$, then with probability at least $1-T^{-2}$, it holds that
\begin{align*}
\abs*{\bhtheta^{\top}_t\bx_t - \btheta_*^{\top}\bx_t} \leq \gamma \|\bx_t\|_{\bA_t^{-1}}, 
\end{align*}
where $\gamma$ is defined in Line 6 of \cref{alg:base_alg_te}. 
\end{repeatlemma}

\begin{proof}{Proof.}
Let the bias of the IPW estimator in \eqref{eq:truncated_estimator_te} at time $\tau$ be $\zeta_\tau = \E[\widetilde{e}_\tau(b_\tau)] - \btheta_*^{\top}\bx_\tau.$ Denote $\bD_t = [\sigma_\tau^{-1} \bx_\tau]_{\tau\in\Phi_t}\in\R^{d\times |\Phi_t|}$ the weighted contexts, $\bV_t = [\sigma_\tau^{-1} \widetilde{e}_\tau(b_\tau)]_{\tau\in\Phi_t}\in\R^{|\Phi_t|\times 1}$ the weighted estimators, and $\bZ_t = [\sigma_\tau^{-1}\zeta_\tau]_{\tau\in\Phi_t}\in\R^{|\Phi_t|\times 1}$ the weighted biases, where  $\sigma_\tau^{-1} = \sqrt{\widehat{G}_t(b_\tau)\parr*{1-\widehat{G}_t(b_\tau)}}$ as defined in \eqref{eq:sigma_t}. We have
\begin{align*}
\bhtheta_{t}^{\top}\bx_t - \btheta_{*}^{\top}\bx_t &= \bx_t^{\top}\bA^{-1}_{t}\bz_t - \bx_t^{\top}\bA^{-1}_{t}\parr*{\bI+\bD_{t}\bD_{t}^{\top}}\btheta_{*}\\
&= \bx_t^{\top}\bA_{t}^{-1}\bD_{t}\parr*{\bV_{t} - \bZ_t - \bD^{\top}_{t}\btheta_{*}} + \bx_t^{\top}\bA_{t}^{-1}\bD_{t}\bZ_t - \bx_t^{\top}\bA_{t}^{-1}\btheta_{*}.
\end{align*}
Given $\|\btheta_{*}\|_2\leq 1$, it follows that
\[
\abs*{\bhtheta_{t}^{\top}\bx_t - \btheta_{*}^{\top}\bx_t} \leq \abs*{\bx_t^{\top}\bA_{t}^{-1}\bD_{t}\parr*{\bV_{t} - \bZ_t - \bD^{\top}_{t}\btheta_{*}}} + \abs*{\bx_t^{\top}\bA_{t}^{-1}\bD_{t}\bZ_t} + \|\bx_t^{\top}\bA_{t}^{-1}\|_2.\numberthis\label{eq:WLS_app_width_decompose}
\]
Since $\bA_t\succeq \bI$, the last term is upper bounded by
\[
\|\bx_t^{\top}\bA_{t}^{-1}\|_2 = \sqrt{\bx_t^{\top}\bA_t^{-1}\bI\bA_t^{-1}\bx_t} \leq \sqrt{\bx_t^{\top}\bA_t^{-1}\bx_t} = \|\bx_t\|_{\bA_t^{-1}}.
\]
For the first term in \eqref{eq:WLS_app_width_decompose}, we can write
\[
\bx_t^{\top}\bA_{t}^{-1}\bD_{t}\parr*{\bV_{t} - \bZ_t -\bD^{\top}_{t}\btheta_{*}} = \sum_{\tau\in\Phi_t}\bx_t^{\top}\bA_t^{-1}\sigma_\tau^{-2}\bx_\tau \widetilde{\varepsilon}_\tau
\]
where $\widetilde{\varepsilon}_\tau = \widetilde{e}_\tau(b_\tau) - \btheta_*^{\top}\bx_\tau - \zeta_\tau$ satisfies the followings: 
\begin{itemize}
    \item $\widetilde{\varepsilon}_\tau$ are conditionally independent given $(\bx_\tau, b_\tau, M_\tau)_{\tau\in\Phi_t}$;

    \item $\E[\widetilde{\varepsilon}_\tau] = 0$;

    \item $\abs*{\widetilde{\varepsilon}_\tau} \leq |\widetilde{e}_\tau(b_\tau)| \leq 2\max\bra*{\widehat{G}_\tau(b_\tau)^{-1}, \parr*{1-\widehat{G}_\tau(b_\tau)}^{-1}}$;
    
    \item $\Var(\widetilde{\varepsilon}_\tau) = \Var(\widetilde{e}_\tau(b_\tau))  \leq c'\sigma_\tau^{2}$ by Lemma~\ref{lem:bias-variance}.
\end{itemize}
Since $\max\bra{\widehat{G}_\tau(b_\tau)^{-1}, \parr{1-\widehat{G}_\tau(b_\tau)}^{-1}} \leq 2 \widehat{G}_\tau(b_\tau)^{-1}\parr{1-\widehat{G}_\tau(b_\tau)}^{-1},$
it follows that for every $\tau\in\Phi_t$,
\begin{itemize}
    \item $\Var\parr*{\bx_t^{\top}\bA_t^{-1}\sigma_\tau^{-2}\bx_\tau\widetilde{\varepsilon}_\tau} \leq \abs*{\bx_t^{\top}\bA_t^{-1}\sigma_\tau^{-1} \bx_\tau}^2\sigma_\tau^{-2}\Var(\widetilde{\varepsilon}_\tau) \leq c'\abs*{\bx_t^{\top}\bA_t^{-1}\sigma_\tau^{-1} \bx_\tau}^2;$

    \item $\abs*{\bx_t^{\top}\bA_t^{-1}\sigma_\tau^{-2}\bx_\tau\widetilde{\varepsilon}_\tau} \leq 4\abs*{\bx_t^{\top}\bA_t^{-1}\bx_\tau} \leq 4\|\bx_t^{\top}\bA_t^{-1}\|_2 \leq 4\|\bx_t\|_{\bA_t^{-1}}$.
\end{itemize}
Then by Bernstein's inequality (Lemma~\ref{lem:bernstein}), with probability at least $1-T^{-2}$,
\begin{align*}
\abs*{\sum_{\tau\in\Phi_t}\bx_t^{\top}\bA_t^{-1}\sigma_\tau^{-2}\bx_\tau \widetilde{\varepsilon}_\tau} &\leq \sqrt{2c'\log\parr*{2T^2}\sum_{\tau\in\Phi_t}\abs*{\bx_t^{\top}\bA_t^{-1}\sigma_\tau^{-1} \bx_\tau}^2} + \frac{8}{3}\log\parr*{2T^2}\|\bx_t\|_{\bA_t^{-1}} \\
&= \sqrt{2c'\log\parr*{2T^2}\|\bx_t^{\top}\bA_t^{-1}\bD_t\|_2^2} + \frac{8}{3}\log\parr*{2T^2}\|\bx_t\|_{\bA_t^{-1}} \\
&\stepa{\leq} \sqrt{2c'\log\parr*{2T^2}\|\bx_t\|_{\bA_t^{-1}}^2} + \frac{8}{3}\log\parr*{2T^2}\|\bx_t\|_{\bA_t^{-1}} \\
&\leq 14\log\parr*{2T} \|\bx_t\|_{\bA_t^{-1}}
\end{align*}
where $c'=16$ in \cref{lem:bias-variance} and (a) follows from $\|\bx_t^{\top}\bA_t^{-1}\bD_t\|_2^2 = \bx_t^{\top}\bA_t^{-1}\bD_t\bD_t^{\top}\bA_t^{-1}\bx_t \leq \bx_t\bA_t^{-1}\bx_t = \|\bx_t\|_{\bA_t^{-1}}$ by $\bA_t=\bI + \bD_t\bD_t^{\top}$. Finally, to bound the middle term in \eqref{eq:WLS_app_width_decompose}, we write
\begin{align*}
\abs*{\bx_t^{\top}\bA_{t}^{-1}\bD_{t}\bZ_t} &\stepb{\leq} \|\bx_t\|_{\bA_t^{-1}} \|\bD_t\bZ_t\|_{\bA_t^{-1}}\\
&= \|\bx_t\|_{\bA_t^{-1}} \sqrt{\bZ_t^{\top}\bD_t^{\top}\bA_t^{-1}\bD_t\bZ_t}\\
&\stepc{\leq} \|\bx_t\|_{\bA_t^{-1}}\|\bZ_t\|_2
\end{align*}
where (b) applies Cauchy-Schwartz inequality and (c) is due to $\bD_t^{\top}(\bI+\bD_t\bD_t^{\top})^{-1}\bD_t\preceq \bI$ by Woodbury matrix identity. To obtain the desired claim, we can further upper bound the following by Lemma~\ref{lem:bias-variance}:
\begin{align*}
\|\bZ_t\|_2 = \sqrt{\sum_{\tau\in\Phi_t}\sigma_\tau^{-2}\zeta_\tau^2}\leq c\sqrt{\sum_{\tau\in\Phi_t}\delta_\tau^2}.
\end{align*}
\end{proof}

\subsection{Proof of \cref{lem:adapted_WLS}}
\begin{repeatlemma}[\cref{lem:adapted_WLS}.]\label{lem:adapted_WLS_app}
Suppose $(v_{\tau,1},v_{\tau,0})_{\tau\in \Phi_t}$ are conditionally independent given $(\bx_\tau,b_\tau,M_\tau)_{\tau\in \Phi_t}$, then with probability at least $1-T^{-2}$, it holds that
\begin{align*}
\abs*{\bhtheta^{\top}_t\bx_t - \btheta_*^{\top}\bx_t} \leq \gamma \|\bx_t\|_{\bA_t^{-1}}, 
\end{align*}
where $\bA_t$ and $\gamma$ are defined in Line 3 and Line 6 of \cref{alg:base_alg_three} respectively. 
\end{repeatlemma}

\begin{proof}{Proof.}
This proof largely follows that of \cref{lem:WLS}, except that the weights in the WLS is $\sigma_\tau^{-4}$ instead of $\sigma_\tau^{-2}$. Formally speaking, let the bias of the IPW estimator in \eqref{eq:truncated_estimator_te_weaker} at time $\tau$ be $\zeta_\tau = \E[\widetilde{e}_\tau(b_\tau)] - \btheta_*^{\top}\bx_\tau.$ Denote $\bD_t = [\sigma_\tau^{-2} \bx_\tau]_{\tau\in\Phi_t}\in\R^{d\times |\Phi_t|}$ the weighted contexts, $\bV_t = [\sigma_\tau^{-2} \widetilde{e}_\tau(b_\tau)]_{\tau\in\Phi_t}\in\R^{|\Phi_t|\times 1}$ the weighted estimators, and $\bZ_t = [\sigma_\tau^{-2}\zeta_\tau]_{\tau\in\Phi_t}\in\R^{|\Phi_t|\times 1}$ the weighted biases, where  $\sigma_\tau^{-2} = \widehat{G}_t(b_\tau)\parr*{1-\widehat{G}_t(b_\tau)}$. We have
\begin{align*}
\bhtheta_{t}^{\top}\bx_t - \btheta_{*}^{\top}\bx_t &= \bx_t^{\top}\bA^{-1}_{t}\bz_t - \bx_t^{\top}\bA^{-1}_{t}\parr*{\bI+\bD_{t}\bD_{t}^{\top}}\btheta_{*}\\
&= \bx_t^{\top}\bA_{t}^{-1}\bD_{t}\parr*{\bV_{t} - \bZ_t - \bD^{\top}_{t}\btheta_{*}} + \bx_t^{\top}\bA_{t}^{-1}\bD_{t}\bZ_t - \bx_t^{\top}\bA_{t}^{-1}\btheta_{*}.
\end{align*}
Given $\|\btheta_{*}\|_2\leq 1$, it follows that
\[
\abs*{\bhtheta_{t}^{\top}\bx_t - \btheta_{*}^{\top}\bx_t} \leq \abs*{\bx_t^{\top}\bA_{t}^{-1}\bD_{t}\parr*{\bV_{t} - \bZ_t - \bD^{\top}_{t}\btheta_{*}}} + \abs*{\bx_t^{\top}\bA_{t}^{-1}\bD_{t}\bZ_t} + \|\bx_t^{\top}\bA_{t}^{-1}\|_2.\numberthis\label{eq:WLS_app_width_decompose2}
\]
Since $\bA_t\succeq \bI$, the last term is upper bounded by
\[
\|\bx_t^{\top}\bA_{t}^{-1}\|_2 = \sqrt{\bx_t^{\top}\bA_t^{-1}\bI\bA_t^{-1}\bx_t} \leq \sqrt{\bx_t^{\top}\bA_t^{-1}\bx_t} = \|\bx_t\|_{\bA_t^{-1}}.
\]
For the first term in \eqref{eq:WLS_app_width_decompose2}, we can write
\[
\bx_t^{\top}\bA_{t}^{-1}\bD_{t}\parr*{\bV_{t} - \bZ_t -\bD^{\top}_{t}\btheta_{*}} = \sum_{\tau\in\Phi_t}\bx_t^{\top}\bA_t^{-1}\sigma_\tau^{-4}\bx_\tau \widetilde{\varepsilon}_\tau
\]
where $\widetilde{\varepsilon}_\tau = \widetilde{e}_\tau(b_\tau) - \btheta_*^{\top}\bx_\tau - \zeta_\tau$ satisfies the followings: 
\begin{itemize}
    \item $\widetilde{\varepsilon}_\tau$ are conditionally independent given $(\bx_\tau, b_\tau, M_\tau)_{\tau\in\Phi_t}$;

    \item $\E[\widetilde{\varepsilon}_\tau] = 0$;

    \item $\abs*{\widetilde{\varepsilon}_\tau} \leq |\widetilde{e}_\tau(b_\tau)| \leq 4\sigma_\tau^2$;
    
    \item $\Var(\widetilde{\varepsilon}_\tau) = \Var(\widetilde{e}_\tau(b_\tau))  \leq c'\sigma_\tau^4$ by Lemma~\ref{lem:bias-variance_weaker}.
\end{itemize}
It follows that for every $\tau\in\Phi_t$,
\begin{itemize}
    \item $\Var\parr*{\bx_t^{\top}\bA_t^{-1}\sigma_\tau^{-4}\bx_\tau\widetilde{\varepsilon}_\tau} \leq \abs*{\bx_t^{\top}\bA_t^{-1}\sigma_\tau^{-2} \bx_\tau}^2\sigma_\tau^{-4}\Var(\widetilde{\varepsilon}_\tau) \leq 2c'\abs*{\bx_t^{\top}\bA_t^{-1}\sigma_\tau^{-2} \bx_\tau}^2;$

    \item $\abs*{\bx_t^{\top}\bA_t^{-1}\sigma_\tau^{-4}\bx_\tau\widetilde{\varepsilon}_\tau} \leq 4\abs*{\bx_t^{\top}\bA_t^{-1}\bx_\tau} \leq 4\|\bx_t^{\top}\bA_t^{-1}\|_2 \leq 4\|\bx_t\|_{\bA_t^{-1}}$.
\end{itemize}
Then by Bernstein's inequality (Lemma~\ref{lem:bernstein}), with probability at least $1-T^{-2}$,
\begin{align*}
\abs*{\sum_{\tau\in\Phi_t}\bx_t^{\top}\bA_t^{-1}\sigma_\tau^{-4}\bx_\tau \widetilde{\varepsilon}_\tau} &\leq \sqrt{2c'\log\parr*{2T^2}\sum_{\tau\in\Phi_t}\abs*{\bx_t^{\top}\bA_t^{-1}\sigma_\tau^{-2} \bx_\tau}^2} + \frac{8}{3}\log\parr*{2T^2}\|\bx_t\|_{\bA_t^{-1}} \\
&= \sqrt{2c'\log\parr*{2T^2}\|\bx_t^{\top}\bA_t^{-1}\bD_t\|_2^2} + \frac{8}{3}\log\parr*{2T^2}\|\bx_t\|_{\bA_t^{-1}} \\
&\stepa{\leq} \sqrt{2c'\log\parr*{2T^2}\|\bx_t\|_{\bA_t^{-1}}^2} + \frac{8}{3}\log\parr*{2T^2}\|\bx_t\|_{\bA_t^{-1}} \\
&\leq 24\log\parr*{2T} \|\bx_t\|_{\bA_t^{-1}}
\end{align*}
where $c'=80$ in \cref{lem:bias-variance_weaker} and (a) follows from $\|\bx_t^{\top}\bA_t^{-1}\bD_t\|_2^2 = \bx_t^{\top}\bA_t^{-1}\bD_t\bD_t^{\top}\bA_t^{-1}\bx_t \leq \bx_t\bA_t^{-1}\bx_t = \|\bx_t\|_{\bA_t^{-1}}$ by $\bA_t=\bI + \bD_t\bD_t^{\top}$. Finally, to bound the middle term in \eqref{eq:WLS_app_width_decompose2}, we write
\begin{align*}
\abs*{\bx_t^{\top}\bA_{t}^{-1}\bD_{t}\bZ_t} &\stepb{\leq} \|\bx_t\|_{\bA_t^{-1}} \|\bD_t\bZ_t\|_{\bA_t^{-1}}\\
&= \|\bx_t\|_{\bA_t^{-1}} \sqrt{\bZ_t^{\top}\bD_t^{\top}\bA_t^{-1}\bD_t\bZ_t}\\
&\stepc{\leq} \|\bx_t\|_{\bA_t^{-1}}\|\bZ_t\|_2
\end{align*}
where (b) applies Cauchy-Schwartz inequality and (c) is due to $\bD_t^{\top}(\bI+\bD_t\bD_t^{\top})^{-1}\bD_t\preceq \bI$ by Woodbury matrix identity. To obtain the desired claim, we further upper bound the following by Lemma~\ref{lem:bias-variance_weaker}:
\begin{align*}
\|\bZ_t\|_2 = \sqrt{\sum_{\tau\in\Phi_t}\sigma_\tau^{-4}\zeta_\tau^2}\leq c\sqrt{\sum_{\tau\in\Phi_t}(\delta_\tau^2+\xi^2)}.
\end{align*}
\end{proof}

\subsection{Proof of \cref{lem:good_cdf_interval}}

\begin{repeatlemma}[\cref{lem:good_cdf_interval}.]
Let $c=\frac{1}{60L}$ and $[b_{\mathrm{left}},b_{\mathrm{right}}]$ be the interval found in \cref{alg:ucb_selection}. It holds that
\begin{enumerate}
    \item[(1)] $\argmax_{b^*\in \Bcal}\br_t(b^*) \in [b_{\mathrm{left}},b_{\mathrm{right}}]$;

    \item[(2)] $\widehat{G}_t(b_{\mathrm{right}}) - \widehat{G}_t(b_{\mathrm{left}}) \le 1 - 2c$ when $c^2 \ge \abs*{\widehat{v}-\btheta_*^\top \bx_t} + \|\widehat{G}_t - G_t\|_\infty$.
\end{enumerate}
\end{repeatlemma}
\begin{proof}{Proof.}
For the ease of notation, write $\eta = \abs*{\widehat{v}-\btheta_*^\top \bx_t} + \|\widehat{G}_t - G_t\|_\infty$. Recall the definitions in \cref{alg:ucb_selection} that $b_{+} = \argmax_{b\in \Bcal} \widehat{G}_t(b)(\widehat{v} + \frac{1}{30L} - b)$, $b_{-} = \argmax_{b\in \Bcal} \widehat{G}_t(b)(\widehat{v} - \frac{1}{30L} - b)$, $U = \bra{b\in \Bcal: \widehat{G}_t(b_-) - c \le \widehat{G}_t(b) \le \widehat{G}_t(b_+) + c}$, $b_{\mathrm{left}}= \min U$, and $b_{\mathrm{right}}= \max U$. Write $\hb_t^* = \argmax_{b^*\in\Bcal}\br_t(b^*)$ and $v_t=\btheta_*^\top \bx_t$ for simplicity. We have
\begin{align*}
G_t(\hb_t^*)(v_t - \hb_t^*) &\le \widehat{G}_t(\hb_t^*)(v_t - \hb_t^*) + \|\widehat{G}_t - G_t\|_\infty\\
&\le \widehat{G}_t(\hb_t^*)(\widehat{v} - \frac{1}{30L} - \hb_t^*) + \abs*{\widehat{v}-v_t} + \|\widehat{G}_t - G_t\|_\infty + \frac{\widehat{G}_t(\hb_t^*)}{30L}\\
&\le \widehat{G}_t(b_-)(\widehat{v} - \frac{1}{30L} - b_-) + \abs*{\widehat{v}-v_t} + \|\widehat{G}_t - G_t\|_\infty + \frac{\widehat{G}_t(\hb_t^*)}{30L}\\
&\le \widehat{G}_t(b_-)(\widehat{v} - b_-) + \abs*{\widehat{v}-v_t} + \|\widehat{G}_t - G_t\|_\infty + \frac{\widehat{G}_t(\hb_t^*)-\widehat{G}_t(b_-)}{30L}\\
&\le G_t(b_-)(v_t - b_-) + 2\abs*{\widehat{v}-v_t} + 2\|\widehat{G}_t - G_t\|_\infty + \frac{\widehat{G}_t(\hb_t^*)-\widehat{G}_t(b_-)}{30L}.
\end{align*}
Since $G_t(b_-)(v_t - b_-)\le G_t(\hb_t^*)(v_t - \hb_t^*)$, this implies
\begin{equation*}
\widehat{G}_t(\hb_t^*)\ge \widehat{G}_t(b_-) - 60L\parr*{\abs*{\widehat{v}-v_t} +\|\widehat{G}_t - G_t\|_\infty} \ge \widehat{G}_t(b_-)-\frac{\eta}{c} \ge \widehat{G}_t(b_-) - c
\end{equation*}
where the last inequality comes from the lemma assumption $\eta \le c^2$. The other direction holds under a symmetric argument, giving $\widehat{G}_t(\hb_t^*)\le \widehat{G}_t(b_+)+c$. Therefore, $\hb_t^*\in[b_{\mathrm{left}},b_{\mathrm{right}}]$ by definition. For the second claim, since $\|\widehat{G}_t - G_t\|_\infty \le \eta < 0.1$ and $\Bcal\subseteq[0,1]$ is a $\frac{1}{\sqrt{T}}-$discretization, \cref{lem:approx_CDF_UCB_Lip} implies $\widehat{G}_t(b_+) - \widehat{G}_t(b_-)\le 1-\frac{1}{15L}$. Finally, we have $\widehat{G}_t(b_{\mathrm{right}}) - \widehat{G}_t(b_{\mathrm{left}}) \le 1-\frac{1}{15L}-120L\eta \le 1-\frac{1}{30L} = 1-2c$ as desired.
\end{proof}

\subsection{Proof of \cref{lem:TE_width}}
\begin{repeatlemma}[\cref{lem:TE_width}.]
Suppose the assumption on $\Bcal$ in \cref{lem:good_cdf_interval} holds and the events in Lemma~\ref{lem:WLS} and \cref{ass:est_oracle_bern} hold. Let $\br_{t,0}$ and $\br_{t,1}$ be defined as in \eqref{eq:UCB0} and \eqref{eq:UCB1}. For the index $i\in\{0,1\}$ selected in \cref{alg:base_alg_te}, it holds that 
\begin{equation*}
u_{t,i}(b) - 2\min\bra{w_{t,0}(b), w_{t,1}(b)}\le \br_{t,i}(b)  \le u_{t,i}(b) 
\end{equation*}
for all $b\in [b_{\mathrm{left}}, b_{\mathrm{right}}]$. Here the interval $[b_{\mathrm{left}}, b_{\mathrm{right}}]$ is defined in \cref{alg:ucb_selection} and the widths $w_{t,1},w_{t,0}$ are defined on Line 8 of \cref{alg:base_alg_te}.
\end{repeatlemma}
\begin{proof}{Proof.}
Denote $\gamma_t=1+c_1\log(2T)+c_2\sqrt{\sum_{\tau\in\Phi_t}\delta_\tau^2}$ as defined in \cref{alg:base_alg_te}. By Lemma~\ref{lem:good_cdf_interval}, when $\gamma_t\|\bx_t\|_{\bA_t^{-1}}+2\delta_t\leq \frac{1}{(60L)^2}$, we have $\widehat{G}_t(b_{\mathrm{right}})-\widehat{G}_t(b_{\mathrm{left}})\leq 1-\frac{1}{30L}$.

If $\widehat{G}_t(b_{\mathrm{left}})> \frac{1}{60L}$, then $\widehat{G}_t(b)\geq \frac{1}{60L}$ and $1-\widehat{G}_t(b)\leq 60L \min\bra{\widehat{G}_t(b),1-\widehat{G}_t(b)}$ for all $b\in [b_{\mathrm{left}}, b_{\mathrm{right}}]$ by monotonicity of $\widehat{G}_t$. Recall that $\br_{t,1}(b) = (1-G_t(b))\parr*{b-\btheta_*^{\top}\bx_t}-b$ in \eqref{eq:UCB1}. By Lemma~\ref{lem:WLS} and \cref{ass:est_oracle_bern},
\[
\br_{t,1}(b) \leq u_{t,1}(b) = \widehat{G}_t(b)\parr*{\bhtheta_t^{\top}\bx_t -b - \gamma\|\bx_t\|_{\bA_t^{-1}}} - \parr*{\bhtheta_t^{\top}\bx_t  - \gamma\|\bx_t\|_{\bA_t^{-1}}} + 2\delta_t
\]
and
\begin{align*}
\abs*{u_{t,1}(b) - \br_{t,1}(b)} &= \abs*{ \widehat{G}_t(b)\parr*{\bhtheta_t^{\top}\bx_t -b - \gamma\|\bx_t\|_{\bA_t^{-1}}} - \parr*{\bhtheta_t^{\top}\bx_t  - \gamma\|\bx_t\|_{\bA_t^{-1}}} + 2\delta_t - \br_{t,1}(b)}\\
&\le (1-\widehat{G}_t(b))\abs*{\btheta_*^{\top}\bx_t - \bhtheta_t^{\top}\bx_t+\gamma\|\bx_t\|_{\bA_t^{-1}}} + 4\delta_t\\
&\leq 2(1-\widehat{G}_t(b))\gamma_t\|\bx_t\|_{\bA_t^{-1}}+ 4\delta_t \\
&\leq 120L\gamma_t \min\bra{\widehat{G}_t(b),1-\widehat{G}_t(b)} \|\bx_t\|_{\bA_t^{-1}} + 4\delta_t \\
&\leq \min\bra{w_{t,0}(b),w_{t,1}(b)}.
\end{align*}

If $\widehat{G}_t(b_{\mathrm{left}})\leq \frac{1}{60L}$, since $\widehat{G}_t(b_{\mathrm{right}})-\widehat{G}_t(b_{\mathrm{left}})\leq 1-\frac{1}{30L}$, we have $\widehat{G}_t(b_{\mathrm{right}})\leq 1-\frac{1}{60L}$ when $60L(\gamma_t\|\bx_t\|_{\bA_t^{-1}}+\delta_t)\leq \frac{1}{60L}$. Then for all $b\in [b_{\mathrm{left}}, b_{\mathrm{right}}]$,
\[
\widehat{G}_t(b)\leq 60L\min\bra{\widehat{G}_t(b),1-\widehat{G}_t(b)} + (60L)^2(\gamma_t\|\bx_t\|_{\bA_t^{-1}}+2\delta_t).
\]
Again, by definition in \eqref{eq:UCB0} and Lemma~\ref{lem:WLS}, 
$$
G_t(b)\parr*{\btheta_*^{\top}\bx_t -b} = \br_{t,0}(b)\leq u_{t,0}(b) = \widehat{G}_t(b)\parr*{\bhtheta_t^{\top}\bx_t -b + \gamma\|\bx_t\|_{\bA_t^{-1}}} + 2\delta_t.
$$
For all $b\in [b_{\mathrm{left}}, b_{\mathrm{right}}]$ we have
\begin{align*}
\abs*{u_{t,0}(b) - \br_{t,0}(b)} &\leq 2\widehat{G}_t(b)\gamma_t\|\bx_t\|_{\bA_t^{-1}} + 4\delta_t\\
&\leq 120L\min\bra{\widehat{G}_t(b),1-\widehat{G}_t(b)}\gamma_t\|\bx_t\|_{\bA_t^{-1}} + 2(60L)^2\gamma_t\|\bx_t\|_{\bA_t^{-1}}(\gamma_t\|\bx_t\|_{\bA_t^{-1}}+2\delta_t) + 4\delta_t\\
&\leq \min\bra{w_{t,0}(b),w_{t,1}(b)}.
\end{align*}
\end{proof}

\subsection{Proof of \cref{lem:master_confidence_bound_te}}
\begin{repeatlemma}[\cref{lem:master_confidence_bound_te}.]\label{lem:master_confidence_bound_te_app}
In \cref{alg:master_alg_te}, for every $t\in[T]$ and $s\in[S]$, it holds that $\br_t(b_t^*)-\br_t(b) \leq 17\cdot 2^{-s}$ for all $b\in B_s$, where $\hb_t^*=\argmax_{b^*\in\Bcal}\br_t(b^*)$ in the discretization $\Bcal$.
\end{repeatlemma}
\begin{proof}{Proof.}
We prove this result via an induction argument on $s\in[S]$ with the following induction hypothesis: for each $s\in[S]$, the optimal bid remains uneliminated, i.e. $\hb^*_t\in B_s$, and $\hr_t(b^*_t)- \hr_t(b) \leq 17\cdot 2^{-s}$ for all $b\in B_s$. For $s=1$ it clearly holds.

Suppose the hypothesis holds for $s-1$. The elimination step (Line~\ref{line:master_te_elim_step}) in \cref{alg:master_alg_te} implies that for any $b\in B_s$, we have $w^{(s-1)}_t(b) \leq 2^{-(s-1)}$, $w^{(s-1)}_t(\hb^*_t)\leq 2^{-(s-1)}$ and
\begin{equation*}
u_{t,i_{s-1}}^{(s-1)}(b) \ge u_{t,i_{s-1}}^{(s-1)}(\hb^*_t) - 2\cdot 2^{-(s-1)}
\end{equation*}
for the criterion $i_{s-1}\in\{0,1\}$ selected in stage $s-1$. Since $\hb^*_t\in B_{s-1}$ by the induction hypothesis and $w^{(s-1)}_t(b) = \min\bra{w^{(s-1)}_{t,0}(b), w^{(s-1)}_{t,1}(b)}$ by definition, by \cref{lem:TE_width},
\[
u_{t,i_{s-1}}^{(s-1)}(\hb^*_t) \geq \br_{t,i_{s-1}}\parr{\hb^*_t} \ge \br_{t,i_{s-1}}(b) \ge u_{t,i_{s-1}}^{(s-1)}(b) - 2w^{(s-1)}_t(b) \ge \br_{t,i_{s-1}}(b) - 2\cdot 2^{-(s-1)}
\]
for all $b\in B_{s-1}$. Therefore, together with the first claim of Lemma~\ref{lem:good_cdf_interval}, $\hb^*_t\in B_s$ after the elimination step at stage $s-1$. Also, for any $b\in B_s$,
\begin{align*}
\br_t(\hb^*_t) - \br_t(b) &= \br_{t,i_{s-1}}(\hb^*_t) -\br_{t,i_{s-1}}(b)\\
&\stepa{\leq} u_{t,i_{s-1}}^{(s-1)}(\hb^*_t) - u^{(s-1)}_{t,i_{s-1}}(b) + 2w^{(s-1)}_t(b)\\
&\stepb{\le} u_{t,i_{s-1}}^{(s-1)}(\hb^*_t) - \max_{b\in B_{s-1}}u^{(s-1)}_{t,i_{s-1}}(b) + 2\cdot 2^{-(s-1)} + 2w^{(s-1)}_t(b)\\
&\le 8\cdot 2^{-(s-1)} = 16\cdot 2^{-s}
\end{align*}
where (a) uses \cref{lem:TE_width} and (b) is by the elimination criterion in \cref{alg:master_alg_te}. Finally, since $s\le S\le \log\sqrt{T}$, we also have $\br_t(b_t^*) - \br_t(\hb^*_t)\le \frac{1}{\sqrt{T}} \le 2^{-s}$ by \cref{lem:bounded_regret_by_bid_discrete}. This concludes the induction and hence the proof.
\end{proof}

\subsection{Proof of \cref{lem:elliptical-potential-TE}}

Then we are ready for the proof of \cref{lem:elliptical-potential-TE}.
\begin{repeatlemma}[\cref{lem:elliptical-potential-TE}]\label{lem:elliptical-potential-TE_app}
For any $s\in[S]$, suppose $\delta_t<0.1$ for $t\in\Phi^{(s)}$ and we have
\[
\sum_{t\in\Phi^{(s)}}w^{(s)}_t(b_t) = O\parr*{\sqrt{\Delta dT}\log^2T + d\Delta\log T + \sum_{t\in\Phi^{(s)}} \delta_t}.
\]
\end{repeatlemma}
\begin{proof}{Proof.}
Fix any $s\in[S]$. Write $\bA_t=\bA_t^{(s)}$ and $\gamma = 1+ c_1\log(2T) + c_2\sqrt{\sum_{t\in\Phi^{(s)}}\delta_t^2}$ for simplicity. For each time $t$, we write the parameter $\gamma_t = 1+ c_1\log(2T) +  c_2\sqrt{\sum_{\tau\in\Phi_t^{(s)}}\delta_\tau^2}$ defined in Algorithm~\ref{alg:base_alg_te} for clarity.
Summing over the widths in this stage gives:
\begin{align*}
    \sum_{t\in\Phi^{(s)}}w^{(s)}_t(b_t) &= \sum_{t\in\Phi^{(s)}} c_3L\gamma_t\|\bx_t\|_{\bA_t^{-1}}\parr*{\min\bra{\widehat{G}_t(b),1-\widehat{G}_t(b)} + c_3L\parr*{\gamma_t\|\bx_t\|_{\bA_t^{-1}} + 4\delta_t}} + 4\delta_t \\
    &\stepa{\leq} 2c_3\underbrace{\sum_{t\in\Phi^{(s)}} L\gamma_t\|\sigma_t^{-1} \bx_t\|_{\bA_t^{-1}}}_{\text{(A)}} + (c_3L)^2\underbrace{\sum_{t\in\Phi^{(s)}} \gamma_t^2\| \bx_t\|^2_{\bA_t^{-1}}}_{\text{(B)}} + 4(c_3L)^2\underbrace{\sum_{t\in\Phi^{(s)}} \gamma_t\| \bx_t\|_{\bA_t^{-1}}\delta_t}_{\text{(C)}} + 4\sum_{t\in\Phi^{(s)}} \delta_t\numberthis \label{eq:elliptical_te_eq0}
\end{align*}
where (a) applies $\min\bra{\widehat{G}_t(b),1-\widehat{G}_t(b)} \le 2 \widehat{G}_t(b)\parr{1-\widehat{G}_t(b)}$. First, by \cref{lem:sum_width_bound2_FPA2} and $\gamma_t\leq \gamma$,
\begin{equation}\label{eq:elliptical_te_eq1}
\text{(A)} = O\parr*{\gamma\sqrt{d|\Phi^{(s)}|\log\abs*{\Phi^{(s)}}}} = O\parr*{\gamma\sqrt{dT\log T}}.
\end{equation}
To bound (B), recall that $\|\widehat{G}_t-G_t\|_{\infty}\leq 2\delta_t$ by \cref{ass:est_oracle_bern}. Thus by the truncation step in Line 20 of Algorithm~\ref{alg:master_alg_te} and \cref{lem:truncation}, it holds that
\[
\min\bra*{\gamma_t\sqrt{\frac{d}{T}}, \frac{1}{2}} \leq \widehat{G}_t(b_t) \leq \max\bra*{1-\gamma_t\sqrt{\frac{d}{T}}, \frac{1}{2} + 4\delta_t}
\]
Recall $\sigma_t^{-2} = \widehat{G}_t(b_t)(1-\widehat{G}_t(b_t)) \le \min\bra{\widehat{G}_t(b_t), 1-\widehat{G}_t(b_t)}.$ It holds that
\[
\|\bx_t\|_{\bA_t^{-1}}^2 \leq \max\bra*{\gamma_t^{-1}\sqrt{\frac{T}{d}}, 10}\|\sigma_t^{-1} \bx_t\|_{\bA_t^{-1}}^2 \leq  \parr*{\gamma_t^{-1}\sqrt{\frac{T}{d}} + 10}\|\sigma_t^{-1} \bx_t\|_{\bA_t^{-1}}^2
\]
since $\delta_t<0.1$ and $\frac{1}{2}+4\delta_t<0.9$. Then we can proceed to
\begin{align*}
\text{(B)} &\le \sum_{t\in\Phi^{(s)}} \gamma_t\sqrt{\frac{T}{d}}\| \sigma_t^{-1}\bx_t\|^2_{\bA_t^{-1}} + 10\gamma_t^2\sqrt{\frac{T}{d}}\| \sigma_t^{-1}\bx_t\|^2_{\bA_t^{-1}}= O\parr*{\gamma\sqrt{dT}\log T + \gamma^2 d\log(T)}\\
&= O\parr*{\sqrt{\Delta dT}\log T + \sqrt{dT}\log^2 T + d\Delta \log T + d\log^2 T} \numberthis\label{eq:elliptical_te_eq2}
\end{align*}
by \cref{lem:sum_width_bound2_FPA2} again. Finally, by Cauchy-Schwartz inequality and AM-GM inequality,
\begin{equation}\label{eq:elliptical_te_eq3}
    \text{(C)} \le \sqrt{\text{(B)}}\cdot\sqrt{\Delta} \le O\parr*{\gamma\sqrt{dT}\log T + \gamma^2 d\log T + \Delta} = O\parr*{\sqrt{\Delta dT}\log T + d\Delta \log T + d\log^2 T}
\end{equation}
since $\gamma^2 = O(\log T + \Delta)$.
Plugging (\ref{eq:elliptical_te_eq1}--\ref{eq:elliptical_te_eq3}) into \eqref{eq:elliptical_te_eq0} yields the claim.
\end{proof}

\subsection{Proof of \cref{lem:adapted_master_confidence_bound_te}}

\begin{repeatlemma}[\cref{lem:adapted_master_confidence_bound_te}]
In \cref{alg:joint_three_master}, for every $t\ge T_0$ and $s\in[S]$, it holds that $\br_t(b_t^*)-\br_t(b) \leq 17\cdot 2^{-s}$ for all $b\in B_s$.
\end{repeatlemma}
\begin{proof}{Proof.}
The proof of this result is a combination of \cref{lem:TE_width} and \ref{lem:master_confidence_bound_te}. The difference mainly lies in the additional exploration step in Line 10 of \cref{alg:joint_three_master}. Fix any $t\ge T_0$. For $s=1$ the claim trivially holds. Now we apply an inductive argument over $s$. Suppose the claim holds up to stage $s-1$.

If the algorithm enters the exploration condition in Line 10 at stage $s-1$, then it will not move to stage $s$ and the claim is vacuously true. Suppose now it moves past Line 10 in stage $s-1$, which implies 
\[
\abs*{\bhtheta_t^\top \bx_t - \btheta_*^\top \bx_t} + \|\widehat{G}_t - G_t\|_\infty \le \min_{b\in \Bcal}\max\bra{w^{(s-1)}_{t,0}(b), w^{(s-1)}_{t,1}(b)} \le c^2
\]
by \cref{lem:adapted_WLS} and \cref{ass:est_oracle_weaker_bern} for the constant $c=\frac{1}{60L}$. Then the condition in \cref{lem:good_cdf_interval} is satisfied and the UCB selection routine \cref{alg:ucb_selection} works as intended. Following verbatim the proof of \cref{lem:TE_width}, we have for the selected UCB index $i_{s-1}\in\bra{0,1}$,
\[
u_{t,i_{s-1}}(b) - 2\min\bra{w^{(s-1)}_{t,0}(b), w^{(s-1)}_{t,1}(b)} \le \br_{t,i_{s-1}}(b) \le u_{t,i_{s-1}}(b).
\]
Then the elimination step (Line 19) of \cref{alg:joint_three_master} implies that $w^{(s-1)}_t(b) = \min\bra{w^{(s-1)}_{t,0}(b), w^{(s-1)}_{t,1}(b)} \le 2^{-(s-1)}$ for all $b\in B_{s-1}$. Then following verbatim the proof of \cref{lem:master_confidence_bound_te}, we have $\hb_t^*\in B_s$ and for every $b\in B_s$,
\[
\br_t(b_t^*) - \br_t(b) \le \br_t(\hb_t^*) - \br_t(b) + \frac{1}{\sqrt{T}} \le 16\cdot 2^{-s} + \frac{1}{\sqrt{T}} \le 17\cdot 2^{-s}
\]
since $s\le S \le \log\sqrt{T}$ and by \cref{lem:bounded_regret_by_bid_discrete}.
\end{proof}

\subsection{Proof of \cref{lem:elliptical-potential-binary}}
\begin{repeatlemma}[\cref{lem:elliptical-potential-binary}]
For any $s\in[S]$ in \cref{alg:joint_three_master}, we have
\[
\sum_{t\in\Phi^{(s)}_1}w^{(s)}_t(b_t) = O\parr*{\sqrt{\Delta dT}\log^2 T + d\Delta\log T + \sum_{t\in\Phi^{(s)}_1} \delta_t + \xi|\Phi^{(s)}_1|}
\]
where we recall that $\Phi^{(s)}_1\subseteq [T]\backslash[T_0]$ denotes the set of time indices belonging to stage $s\in[S]$ and selected when $w^{(s)}_t(b_t)>2^{-s}$.
\end{repeatlemma}
\begin{proof}{Proof.}
This proof is a simpler version of \cref{lem:elliptical-potential-TE}. Let $\gamma_t = 1 + c_1\log(2T) + c_2\sqrt{\sum_{\tau\in\Phi_t^{(s)}}(\delta_\tau^2 + \xi^2)}$ as defined in \cref{alg:base_alg_three}. We have
\begin{align*}
\sum_{t\in\Phi^{(s)}_1}w^{(s)}_t(b_t) &= \sum_{t\in\Phi^{(s)}_1} 120L\gamma_t\|\bx_t\|_{\bA_t^{-1}}\min\bra{\widehat{G}_t(b),1-\widehat{G}_t(b)} + 4\delta_t + 2\xi\\
&\le \sum_{t\in\Phi^{(s)}_1} 240L\gamma_t\|\sigma_t^{-2}\bx_t\|_{\bA_t^{-1}} + 4\delta_t + 2\xi.
\end{align*}
By \eqref{eq:elliptical_te_eq1}, we have
\begin{align*}
\sum_{t\in\Phi^{(s)}_1}w^{(s)}_t(b_t) &= O\parr*{\sqrt{d|\Phi^{(s)}_1|\log T\sum_{t\in\Phi^{(s)}_1}\parr*{\delta_t^2 + \xi^2}} + \sum_{t\in\Phi^{(s)}_1}\delta_t + \xi|\Phi^{(s)}_1|}\\
&= O\parr*{\sqrt{\Delta dT\log T} + \sum_{t\in\Phi^{(s)}_1}\delta_t + \xi|\Phi^{(s)}_1|}
\end{align*}
where the second line follows from the subadditivity of square root.
\end{proof}

\subsection{Proof of \cref{lem:bound_exploration_binary}}
\begin{repeatlemma}[\cref{lem:bound_exploration_binary}]
In \cref{alg:joint_three_master}, it holds that
\[
\sum_{s\in[S]}|\Phi^{(s)}_0| = O\parr*{d\Delta\log^2 T + \xi T}.
\]
\end{repeatlemma}
\begin{proof}{Proof.}
To bound the number of exploration times, recall that we are considering time $t\ge T_0$ where $\|\delta_t\|_\infty \le \frac{c}{8}$ for constant $c=\frac{1}{60L}$ as in \cref{thm:main-lte-binary}. Fix any stage $s\in[S]$. A time index $t\in \Phi^{(s)}_0$ if
\[
\min_{b\in\Bcal}120L\gamma_t\|\bx_t\|_{\bA_t^{-1}}\max\bra{\widehat{G}_t(b),1-\widehat{G}_t(b)} + 4\delta_t + 2\xi > c
\]
where $\gamma_t = 1 + c_1\log(2T) + c_2\sqrt{\sum_{\tau\in\Phi_t^{(s)}}(\delta_\tau^2 + \xi^2)}$ as defined in \cref{alg:base_alg_three}, which implies
\[
240L\gamma\|\bx_t\|_{\bA_t^{-1}} + 4\delta_t + 2\xi > c.
\]
where $\gamma = 1 + c_1\log(2T) + c_2\sqrt{\sum_{t\in\Phi^{(s)}}(\delta_t^2 + \xi^2)} \ge \gamma_t$. Since $\|\delta_t\|_\infty \le \frac{c}{8}$, this implies
\[
240L\gamma\|\bx_t\|_{\bA_t^{-1}} + 2\xi > \frac{c}{2}.
\]
Since $t\in\Phi^{(s)}_0$ belongs to exploration, $\sigma_t^2 = 4$ and 
\[
960L\gamma\|\sigma_t^{-2}\bx_t\|_{\bA_t^{-1}} + 2\xi > \frac{c}{2}.
\]
By the elliptical potential lemma \cref{lem:sum_width_bound2_FPA2}, we have
\[
\frac{c}{2}\abs{\Phi^{(s)}_0} < \sum_{t\in\Phi^{(s)}_0}\parr*{960L\gamma\|\sigma_t^{-2}\bx_t\|_{\bA_t^{-1}} + 2\xi} \le 960L\gamma\sqrt{d\abs{\Phi^{(s)}_0}\log T} + 2\xi \abs{\Phi^{(s)}_0}.
\]
By some algebra and $\gamma^2 = O(\log T + \Delta + \xi^2\abs{\Phi^{(s)}_0})$, it holds that
\[
\abs{\Phi^{(s)}_0} = O\parr*{d\gamma^2 + \xi \abs{\Phi^{(s)}_0}} = O\parr*{d\Delta\log T +\xi \abs{\Phi^{(s)}_0} + d\xi^2 \abs{\Phi^{(s)}_0}\log T + d\log^2 T}.
\]
\end{proof}

\subsection{Proof of \cref{lem:linear_hob_error}}\label{app:linear_hob}
The proof of \cref{lem:linear_hob_error} is a direct consequence of \cref{lem:sum_width_bound2_FPA2} given the following explicit expression of $\{\delta_t\}_t$:

\begin{lemma}\label{lem:bernstein_for_gaussian_noise}
Under the conditions of \cref{lem:linear_hob_error}, there is an estimator $\widehat{G}_t$ for $G_t$ that satisfies \cref{ass:est_oracle_bern} with 
\begin{align*}
\delta_t=O\parr*{\log^{3/2}(T)\sqrt{\frac{d}{t}} + \log(T) \|\bx_t\|_{\Sigma_t^{-1}} }.
\end{align*}
Here $\Sigma_t = 18I + \sum_{s<t}\bx_s\bx_s^{\top}$ is the regularized Gram matrix at time $t$.
\end{lemma}
\begin{proof}{Proof.}
Given the Lipschitzness in \cref{ass:Gt}, we consider the discretized bid space $\widehat{\Bcal} = \bra{\frac{i}{T}: i\in\{0\}\cup[T]}$. At each time $t$, we estimate the parameter $\varphi_*$ via linear regression with a careful data splitting scheme: Denote $\Sigma_t = 18I + \sum_{s<t}\bx_s \bx_s^{\top}$. By \cref{lem:KS_data_splitting}, there is a subset $S_t\subseteq [t-1]$ such that $|S_t|\leq \frac{t-1}{2}$ and $18I + \sum_{s\in S_t}\bx_s \bx_s^{\top}\succeq \frac{\Sigma_t}{9}$. Define $\bA_t = 18I + \sum_{s\in S_t}\bx_s \bx_s^{\top}$ and $\bz_t = \sum_{s\in S_t}M_s \bx_s$, and set the estimator $\widehat{\varphi}_t = \bA_t^{-1}\bz_t$. Following the same proof of \cref{lem:WLS}, with probability at least $1-T^{-2}/2$, we have a confidence bound that
\begin{equation}\label{eq:temp_HOB_width}
\abs*{\widehat{\varphi}_t^{\top}x - \varphi_*^{\top}x} \leq \beta\|x\|_{\bA_t^{-1}}
\end{equation}
for $\beta = O(\sigma\sqrt{\log T})$ and all $x\in \{\bx_t - \bx_{s}: s\in [t-1]\}$, where $\sigma$ is the sub-Gaussian parameter in \cref{ass:flat_noise_tail}. Construct the empirical CDF as follows: let $X_s(b) = \indic[M_s - \widehat{\varphi}_t^{\top}\bx_s + \widehat{\varphi}_t^{\top}\bx_t\leq b]$ and $\widehat{G}_t(b) = \frac{1}{|S_t^c|}\sum_{s\in S_t^c}X_s(b)$ using the remaining data $S_t^c=[t-1]\backslash S_t$. Note
\begin{align*}
\E[X_s(b)] &= \mathbb{P}\parr*{M_s - \widehat{\varphi}_t^{\top}\bx_s + \widehat{\varphi}_t^{\top}\bx_t\leq b}\\
&= Q\parr*{b-\varphi_*^{\top}\bx_s + \widehat{\varphi}_t^{\top}\bx_s - \widehat{\varphi}_t^{\top}\bx_t}\\
&= Q\parr*{b-\varphi_*^{\top}\bx_t} + Q\parr*{b-\varphi_*^{\top}\bx_t +\varepsilon_s} - Q\parr*{b-\varphi_*^{\top}\bx_t} \numberthis\label{eq:indicator_bias_Mt}
\end{align*}
where $Q(\cdot)$ denotes the CDF of the noise and $\varepsilon_s = (\widehat{\varphi}_t-\varphi_*)^{\top}(\bx_s-\bx_t)$. By \eqref{eq:temp_HOB_width}, the error term $\varepsilon_s$ satisfies 
\begin{equation}\label{eq:est_error_in_Q}
|\varepsilon_s| \leq \beta\|\bx_s-\bx_t\|_{\bA_t^{-1}}.
\end{equation}
Note $Q\parr*{b-\varphi_*^{\top}\bx_t}=\mathbb{P}(M_t\leq b)=G_t(b)$, so the bias of this indicator variable is in the difference of the last two terms in \eqref{eq:indicator_bias_Mt}. 

Write the target value as $v=\varphi_*^{\top}\bx_t-b$ for simplicity. 
Since $\E[X_s(b)] = Q(v+\varepsilon_s)$, its variance is $\Var(X_s(b)) = Q(v+\varepsilon_s)(1-Q(v+\varepsilon_s))$. Now we apply Bernstein's inequality (cf. Lemma~\ref{lem:bernstein}) to $\widehat{G}_t(b)$ and a union bound to obtain, with probability at least $1-T^{-2}/2$, for all $b\in\widehat{\Bcal}$, 
\begin{equation*}
\abs*{\widehat{G}_t(b) - \frac{1}{|S_t^c|}\sum_{s\in S_t^c}Q\parr*{v +\varepsilon_s}} \leq \frac{2\sqrt{2\log(2T)}}{|S_t^c|}\sqrt{\sum_{s\in S_t^c}Q(v+\varepsilon_s)(1-Q(v+\varepsilon_s))} + \frac{8\log(2T)}{3|S_t^c|}.
\end{equation*}
For each $s<t$, we can decompose
\begin{align*}
Q(v+\varepsilon_s)(1-Q(v+\varepsilon_s)) &= Q(v)(1-Q(v)) + Q(v+\varepsilon_s)- Q(v+\varepsilon_s)^2 - Q(v) + Q(v)^2\\
&= Q(v)(1-Q(v)) + \parr*{Q(v+\varepsilon_s) - Q(v)}\parr*{1-Q(v+\varepsilon_s)-Q(v)}\\
&\leq Q(v)(1-Q(v)) + \abs*{Q(v+\varepsilon_s) - Q(v)}.
\end{align*}
Together with the bias term in \eqref{eq:indicator_bias_Mt} and that $\sqrt{a+b}\leq \sqrt{a}+\sqrt{b}$ for $a,b>0$, we obtain an error bound for our CDF estimator:
\begin{align*}
\abs*{\widehat{G}_t(b) - G_t(b)} &\leq 2\sqrt{\frac{2\log(2T)Q(v)(1-Q(v))}{|S_t^c|}} + \frac{2\sqrt{2\log(2T)}}{|S_t^c|}\sqrt{\sum_{s\in S_t^c}\abs*{Q(v+\varepsilon_s) - Q(v)}}\\
&\quad + \frac{8\log(2T)}{3|S_t^c|} + \frac{1}{|S_t^c|}\sum_{s\in S_t^c}\abs*{Q(v+\varepsilon_s) - Q(v)}.
\end{align*}
If $\sum_{s\in S_t^c}\abs*{Q(v+\varepsilon_s) - Q(v)}\leq 1$, since $|S_t^c|\geq \frac{t-1}{2}$, it becomes
\begin{align*}
\abs*{\widehat{G}_t(b) - G_t(b)} &\leq 2\sqrt{\frac{2\log(2T)Q(v)(1-Q(v))}{|S_t^c|}} + \frac{2\sqrt{2\log(2T)}}{|S_t^c|} + \frac{11\log(2T)}{3|S_t^c|}\\
&\leq 4\sqrt{\frac{2\log(2T)Q(v)(1-Q(v))}{t}} + \frac{8\sqrt{2\log(2T)}}{t} + \frac{15\log(2T)}{t}
\end{align*}
and we are done.
When $\sum_{s\in S_t^c}\abs*{Q(v+\varepsilon_s) - Q(v)}>1$, the bound can be simplified to
\begin{equation}\label{eq:bernstein_error_Gt}
\abs*{\widehat{G}_t(b) - G_t(b)} \leq 2\sqrt{\frac{2\log(2T)Q(v)(1-Q(v))}{|S_t^c|}} + \frac{4\sqrt{\log(2T)}}{|S_t^c|}\sum_{s\in S_t^c}\abs*{Q(v+\varepsilon_s) - Q(v)} + \frac{3\log(2T)}{|S_t^c|}.
\end{equation}
The key lies in upper bounding the second term. Note we have 
\[
\abs*{Q(v+\varepsilon_s) - Q(v)} \leq Q'(v)|\varepsilon_s| + \|Q''\|_\infty\varepsilon_s^2.
\]
Let $a_0,a_1,a_2>0$ be the constants in \cref{ass:flat_noise_tail}. When $Q(v)\in(a_0,1-a_0)$, we bound the second term as
\begin{align*}
\sum_{s\in S_t^c}\abs*{Q(v+\varepsilon_s) - Q(v)} &\leq L\beta\sum_{s\in S_t^c}\|\bx_s-\bx_t\|_{\bA_t^{-1}} + a_2\beta^2\sum_{s\in S_t^c}\|\bx_s-\bx_t\|_{\bA_t^{-1}}^2\\
&\leq a_0^{-1}L\beta \sqrt{Q(v)(1-Q(v))}\sum_{s\in S_t^c}\parr*{\|\bx_s\|_{\bA_t^{-1}} + \|\bx_t\|_{\bA_t^{-1}}} + 2a_2\beta^2\sum_{s\in S_t^c}\parr*{\|\bx_s\|_{\bA_t^{-1}}^2 + \|\bx_t\|_{\bA_t^{-1}}^2}\\
&\stepa{\leq} a_0^{-1}L\beta\sqrt{Q(v)(1-Q(v))}\parr*{\sqrt{2d|S_t^c|\log(T)} + 3|S_t^c|\|\bx_t\|_{\Sigma_t^{-1}}}\\
&\quad + 2a_2\beta^2\parr*{2d\log(T)+ 9|S_t^c|\|\bx_t\|_{\Sigma_t^{-1}}^2} + c\beta^2\numberthis\label{eq:sum_bias_bound1}
\end{align*}
for a constant $c>0$, where (a) uses \cref{lem:sum_width_bound2_FPA2} and the properties of $S_t$. On the other hand, when $Q(v)\in[0,a_0]\cup[1-a_0,a]$, we have
\begin{align*}
\sum_{s\in S_t^c}\abs*{Q(v+\varepsilon_s) - Q(v)} &\leq a_1\beta\sqrt{Q(v)(1-Q(v))}\sum_{s\in S_t^c}\|\bx_s-\bx_t\|_{\bA_t^{-1}} + a_2\beta^2\sum_{s\in S_t^c}\|\bx_s-\bx_t\|_{\bA_t^{-1}}^2\\
&\stepa{\leq} a_1\beta\sqrt{Q(v)(1-Q(v))}\parr*{\sqrt{2d|S_t^c|\log(T)} + 3|S_t^c|\|\bx_t\|_{\Sigma_t^{-1}}}\\
&\quad + 2a_2\beta^2\parr*{2d\log(T)+ 9|S_t^c|\|\bx_t\|_{\Sigma_t^{-1}}^2} + c\beta^2. \numberthis\label{eq:sum_bias_bound2}
\end{align*}
Plugging \eqref{eq:sum_bias_bound1} and \eqref{eq:sum_bias_bound2} into \eqref{eq:bernstein_error_Gt} and substituting in $|S_t^c|\geq \frac{t-1}{2} \geq \frac{t}{4}$, the CDF estimation error is bounded as desired: for every $b\in\widehat{\Bcal}$, with probability at least $1-T^{-2}/2$,
\begin{equation*}
\abs*{\widehat{G}_t(b) - G_t(b)} \leq \delta_t\sqrt{Q(v)(1-Q(v))} + \delta_t^2
\end{equation*}
where $\delta_t = O\parr*{\log^{3/2}(T)\sqrt{d/t} + \log(T)\|\bx_t\|_{\Sigma_t^{-1}}}$, and we recall that $Q(v)=G_t(b)$.

For any $b\in\Bcal$ in the continuous bid space, let its discretized neighbors be $h_1=h_2-T^{-1}\leq b \leq h_2 $ for $h_1,h_2\in\widehat{\Bcal}$. Thanks to the Lipschitzness of $Q$, for both $h=h_1,h_2$, it holds that $G_t(h)(1-G_t(h))\leq G_t(b)(1-G_t(b)) + \frac{L}{T}$. Hence by monotonicity of $\widehat{G}_t$ and the triangular inequality, with probability at least $1-T^{-2}/2$,
\begin{align}\label{eq:HOB_error_bound_cts}
\abs*{\widehat{G}_t(b)-G_t(b)} &\leq \max_{h\in \{h_1,h_2\}} \abs*{\widehat{G}_t(h)-G_t(h)} + \frac{L}{T}\nonumber \\
&\le \delta_t\left(\sqrt{G_t(b)(1-G_t(b))} + \sqrt{\frac{L}{T}}\right) + \delta_t^2 + \frac{L}{T} \nonumber \\
&= O(\delta_t)\sqrt{G_t(b)(1-G_t(b))} + O(\delta_t^2).
\end{align}
The proof is completed by taking a union bound over \eqref{eq:temp_HOB_width}, \eqref{eq:HOB_error_bound_cts}, and all time indices $t\in[T]$.

\end{proof}

\subsection{Proof of \cref{lem:hob_est_error_bound}}
We first present the error bound for the CDF estimator \textbf{without} the monotonicity enforcement step in \cref{alg:hob_est}.
Recall the definition of the CDF estimator from \eqref{eq:offline_G_est}: Let $h=T^{-\frac{1}{3}}$ and $I_1,\dots, I_J$ be an interval partition of $[0,1]$ such that $|I_j|\in [h,2h]$ and $J\le \lfloor h^{-1}\rfloor$. The estimator is defined as
\begin{equation*}
g_j \coloneqq \frac{1}{\abs{\bra{\tau\in\Phi_t:b_\tau\in I_j}}} \sum_{\tau\in\Phi_t:b_\tau\in I_j}\indic[M_\tau\le b_\tau],\qquad 
\widehat{G}^{(0)}_{t}(b) \coloneqq \sum_{j=1}^J \indic[b\in I_j]g_j.
\end{equation*}

\begin{repeatlemma}[\cref{lem:hob_est_error_bound}.]
Suppose $h\in(0,1)$ and the intervals $\bra{I_j}_{j\in[J]}$ partition $[0,1]$ and $|I_j|\le 2h$ for every $j$. Also suppose $(\indic[M_\tau\le b_\tau])_{\tau\in\Phi_t}$ are conditionally independent given $(\bx_\tau, b_\tau)_{\tau\in\Phi_t}$. With probability $1-T^{-2}$, for every $b\in[0,1]$
\begin{equation*}
\abs*{\widehat{G}^{(0)}_t(b) - G(b)} \le 
\sqrt{G(b)(1-G(b))}\delta_t(b) + \delta_t(b)^2 + 3Lh,
\end{equation*}
where $\delta_t(b) = c\sqrt{\frac{\log(2T/h)}{N_t(b)}}$ for constant $c=8$. Here $I(b)$ denotes the interval $I_j\ni b$ and $N_t(b) = |\bra{\tau\in \Phi_t: b_\tau\in I(b)}|$ denotes the number of neighbor observations in $\Phi_t$.
\end{repeatlemma}
\begin{proof}{Proof.}
Since $\widehat{G}^{(0)}_t$ is a piece-wise constant local average over the intervals $\bra{I_j}_{j\in[J]}$ as defined in \cref{alg:hob_est}, it suffices to show that the desired bound holds over each interval. Let $I_+=\bra{b\in[0,1]: G(b)\in[Lh,1-Lh]}$ denote the interval where the true CDF is large. First, fix any $I_j\subseteq I_+$ and $b_j$ the center of $I_j$. Then we have the estimation error
\begin{align*}
\abs*{\widehat{G}^{(0)}_t(b_j) - G(b_j)} &= \abs*{\frac{1}{N_t(b_j)}\sum_{\tau\in \Phi_t: b_\tau\in I_j}\indic[M_\tau\le b_\tau] - G(b_j)}
\end{align*}
By Freedman's inequality (\cref{lem:freedman}), with probability at least $1-T^{-\frac{7}{3}}$,
\begin{equation}\label{eq:cdf_est_error_proof_eq1}
\abs*{\frac{1}{N_t(b_j)}\sum_{\tau\in \Phi_t:b_\tau\in I_j}\indic[M_\tau\le b_\tau] - G(b_j)} \le \sqrt{\frac{10G(b_j)(1-G(b_j))\log(2T)}{N_t(b_j)}} + \frac{5\log(2T)}{3N_t(b_j)} + Lh
\end{equation}
where we use $\Var(\indic[M_\tau\le b_\tau])= G(b_\tau)(1-G(b_\tau)) \le 2G(b_j)(1-G(b_j))$ when $b_\tau\in I_j\subseteq I_+$, the Lipschitzness of $G$, and $|b_j-b_\tau|\le h$ when $b_\tau\in I_j$.
Consequently, for every such index $j$, it holds that
\begin{equation*}
\abs*{\widehat{G}^{(0)}_t(b_j) - G(b_j)} \le \sqrt{G(b_j)(1-G(b_j))}\delta_t(b_j) + \delta_t(b_j)^2 + Lh
\end{equation*}
with probability $1-T^{-\frac{7}{3}}$, where $\delta_t(b_j) = 4\sqrt{\frac{\log(2T)}{N_t(b_j)}}$. For any other $b\in I_j$, we have $\widehat{G}^{(0)}_t(b)=\widehat{G}^{(0)}_t(b_j)$, $\delta_t(b)=\delta_t(b_j)$ by definition, and $G(b)(1-G(b))\le 2G(b_j)(1-G(b_j))$ when $I_j\subseteq I_+$ (because $|I_j|\le 2h$ and $G$ is Lipschitz). Consequently, it holds that
\begin{align*}
\abs*{\widehat{G}^{(0)}_t(b) - G(b)} &\le \abs*{\widehat{G}^{(0)}_t(b_j) - G(b_j)} + \abs*{G(b_j) - G(b)}\\
&\le 2\sqrt{G(b)(1-G(b))}\delta_t(b) + \delta_t(b)^2 + 2Lh
\end{align*}
with probability $1-T^{-\frac{7}{3}}$.

Now consider any interval $I_j$ such that $b_j\notin I_+$. WLOG we assume $I_j\subseteq [0,1]\backslash I_+$. Following the same argument, except that we trivially upper bound $G(b_\tau)(1-G(b_\tau)) \le Lh$ for any $b_\tau\in I_j\backslash I_+$, we get
\begin{equation*}
\abs*{\widehat{G}_t(b) - G(b)} \le \sqrt{Lh}\delta_t(b) + \delta_t(b)^2 + 2Lh \stepa{\le} 2\delta_t(b)^2 + 3Lh \le \sqrt{G(b)(1-G(b))}\delta_t(b) + 2\delta_t(b)^2 + 3Lh
\end{equation*}
for every $b\in I_j$ with probability $1-T^{-\frac{7}{3}}$ as desired. Step (a) follows from the AM-GM inequality. We will simply absorb the constant $2$ in $\delta_t$. Taking a union bound over the $J\le T^{\frac{1}{3}}$ intervals completes the proof.
\end{proof}

\subsection{Enforcing Monotonicity on CDF Estimator}\label{sec:enforce_monotonicity_details}

In this section, we provide the details of how to construct a monotone CDF estimator based on \cref{lem:hob_est_error_bound}. Again, if the error width were known, we can apply the recursive procedure in Figure \ref{fig:cdf}. It provides an intuitive procedure to find a monotone adaptation by shifting the empirical average $g_j$ in each interval $I_j$ recursively. The resulting estimator would enjoy the same error bound up to a factor of $2$. 

\begin{figure}[h!]
\centering

\begin{tikzpicture}[scale=6,>=stealth]
\foreach \xx in {0.03,0.07,0.11}{
  \fill[gray] (\xx,0.26) circle (0.007);
}
\foreach \xx in {1.18,1.22,1.26}{
  \fill[gray] (\xx,0.63) circle (0.007);
}
  \foreach \xa/\xb/\band/\y/\cdf in {%
    0.14/0.34/0.08/0.27/0.31,
    0.34/0.54/0.14/0.36/0.34,
    0.54/0.74/0.22/0.38/0.34,
    0.74/0.94/0.14/0.5/0.43,
    0.94/1.14/0.1/0.6/0.55
  }{
    \pgfmathsetmacro{\halfband}{0.5*\band}
    \pgfmathsetmacro{\ymin}{max(0,\y-\halfband)}
    \pgfmathsetmacro{\ymax}{min(1,\y+\halfband)}

    \fill[blue!12] (\xa,\ymin) rectangle (\xb,\ymax);
    \draw[blue,dashed]    (\xa,\ymin) rectangle (\xb,\ymax);
    \draw[very thick] (\xa,\y) -- (\xb,\y);
    \draw[very thick,red] (\xa,\cdf) -- (\xb,\cdf);
  }

  \draw[->] (0,0) -- (1.4,0) node[below] {$x$};
  \draw[->] (0,0) -- (0,0.7) node[left] {$y$};
  \draw (0,0) -- (0,-0.02) node[below=2pt] {\small 0};
  \draw (1.32,0) -- (1.32,-0.02) node[below=2pt] {\small 1};
  \draw (0.54,0) -- (0.54,-0.02) node[below=2pt] {};
  \draw (0.74,0) -- (0.74,-0.02) node[below=2pt] {};
  \node[below] at (0.64,0) {\small $j_0$};
  \draw (0,0.34) -- (-0.02,0.34) node[left=2pt] {\small $\frac{1}{2}$};
  \draw[->, red, thick] (0.56,-0.05) -- (0.36,-0.05);
  \draw[->, red, thick] (0.72,-0.05) -- (0.92,-0.05);
  
  \foreach \x in {0.14,0.34,0.54,0.74,0.94,1.14}{\draw[dashed,gray!60] (\x,0) -- (\x,0.7);}

  \begin{scope}
    \draw[fill=white,draw=black] (0.72,0.04) rectangle (1.4,0.25);
    \draw[very thick] (0.76,0.2) -- (0.96,0.2);
    \node[anchor=west] at (0.98,0.2) {\small empirical mean};
    \draw[very thick,red] (0.76,0.14) -- (0.96,0.14);
    \node[anchor=west] at (0.98,0.14) {\small monotone value};
    \fill[blue!12] (0.76,0.06) rectangle (0.96,0.1);
    \draw[blue]    (0.76,0.06) rectangle (0.96,0.1);
    \node[anchor=west] at (0.98,0.08) {\small width};
  \end{scope}
\end{tikzpicture}

\caption{\raggedright \textbf{Monotone CDF Adjustment.} This figure illustrates a procedure to enforce monotonicity on a piece-wise constant estimator, when the confidence width in each interval is known. The procedure begins with a middle interval $j_0$ and then iterates through the intervals $j=j_0-1,\dots,1$ (resp. $j=j_0+1,\dots,J$) and sets the estimated value to be either the upper confidence bound (resp. the lower confidence bound) or the value of the previous interval. Its validity is guaranteed when these confidence widths hold, i.e. the true CDF $G$ lies within the shaded regions.}
\label{fig:cdf}
\end{figure}

To be more precise, we first identify the ``middle'' region $j_0$ where we can set $g_{j_0}$ to $\frac{1}{2}$, then (1) loop through the intervals $j=j_0-1,\dots, 1$ and recursively set the estimated value to be either the \textit{upper} confidence bound or the value of the previous interval, and (2) loop through the intervals $j=j_0+1,\dots,J$ and recursively set the estimated value to be either the \textit{lower} confidence bound or the value of the previous interval. This procedure is summarized in \cref{alg:enforce_monotone}. At a high level, it finds a monotone variant of the HOB CDF estimator that enjoys the same estimation error bound and ``concentrates'' more toward the value $\frac{1}{2}$. The ultimate goal of such concentration is to guarantee that
\[
\widehat{G}^{(0)}_t(b)(1-\widehat{G}^{(0)}_t(b)) \le \widehat{G}_t(b)(1-\widehat{G}_t(b))
\]
which would appear in the error width of our final estimator $\widehat{G}_t$.

The procedure is formalized in \cref{alg:enforce_monotone}, and its properties are summarized by \cref{lem:monotone_cdf_adjustment}.

\begin{algorithm}[h!]\caption{HOB Adjustment for Monotonicity}
\label{alg:enforce_monotone}
\textbf{Input:} Interval partition $\bra{I_j}_{j\in[J]}$, empirical averages $\bra{\widehat{g}_j}_{j\in[J]}$, confidence widths $\bra{u_j}_{j\in[J]}$.

Find index $j_0\gets \min\bra{j\in[J]: \frac{1}{2}\in [\widehat{g}_j-u_j, \widehat{g}_j + u_j]}$.

Set $\widehat{g}_{j_0}'\gets \frac{1}{2}$.

\For{interval index $j=j_0-1$ \KwTo $1$}{
Set $\widehat{g}_j'\gets \min\bra{\widehat{g}_{j+1}, \widehat{g}_j + u_j}$.
}

\For{interval index $j=j_0+1$ \KwTo $J$}{
Set $\widehat{g}_j'\gets \max\bra{\widehat{g}_{j-1}, \widehat{g}_j - u_j}$.
}

Output estimator $\widehat{G}$ with $\widehat{G}(b) = \sum_{j=1}^J \indic[b\in I_j] \widehat{g}_j'$.
\end{algorithm}

\begin{lemma}\label{lem:monotone_cdf_adjustment}
Let $\bra{I_j}_{j\in[J]}$ be any partition of $[0,1]$ and $G$ be a Lipschitz CDF. Fix any values $\bra{(\widehat{g}_j, u_j)}_{j\in[J]}$. Suppose for each $j\in[J]$, $G(b)\in[\widehat{g}_j-u_j, \widehat{g}_j + u_j]$ for every $b\in I_j$. Then the output $\widehat{G}$ of \cref{alg:enforce_monotone} is monotone and satisfies: For every $j\in[J]$ and $b\in I_j$,
\begin{enumerate}
    \item $\abs*{G(b) - \widehat{G}(b)} \le 2u_j$;

    \item $\sqrt{\widehat{g}_j(1-\widehat{g}_j)} \le \sqrt{\widehat{G}(b)(1-\widehat{G}(b))}$.
\end{enumerate}
\end{lemma}
\begin{proof}{Proof.}
When $G(b)\in [\widehat{g}_j-u_j, \widehat{g}_j+u_j]$ for every interval index $j\in[J]$ and every $b\in I_j$, the first claim is trivially true since $\widehat{G}(b)\in [\widehat{g}_j-u_j, \widehat{g}_j+u_j]$ as well. The second claim follows from the fact that, by construction of $\widehat{G}$ in \cref{alg:enforce_monotone}, $\abs{\widehat{G}(b) - \frac{1}{2}} \le \abs{\widehat{g}_j - \frac{1}{2}}$ for every $j\in[J]$ and $b\in I_j$.
\end{proof}

The procedure in \cref{alg:enforce_monotone} crucially relies on the knowledge of the confidence widths, which is not known from \cref{lem:hob_est_error_bound} due to the presence of $G(b)$. Fortunately, we can find a good approximation of it. In particular, over the regime where $G(b)$ is bounded away from $0$ and $1$ by $h$, the next result allows us to use $\widehat{G}^{(0)}_t(b)$ in computing the width.

\begin{lemma}\label{lem:hatG_for_G_conf_width}
Fix any $h,\delta\in[0,\frac{1}{2}]$ and scalars $\widehat{g},g\in[h,1-h]$. Suppose $\abs*{\widehat{g} - g} \le \delta\sqrt{g(1-g)} + \delta^2 + ch$ such that $\delta \le \sqrt{h}$ for some factors $c\ge 1$. Then it holds that
\[
\abs*{\widehat{g} - g} \le c_0\cdot \delta\sqrt{\widehat{g}(1-\widehat{g})} + \delta^2 + ch
\]
with an extra factor $c_0 = 4(c+2)$.
\end{lemma}
\begin{proof}{Proof.}
The proof follows from a case study. Since $\widehat{g}(1-\widehat{g}) = \Theta(\min\bra{\widehat{g}, 1-\widehat{g}})$, we first focus on the upper bound that only involves $\widehat{g}$. Suppose for now $g\le 2\widehat{g}$. Then
\begin{equation}\label{eq:knowing_width_eq1}
\abs*{\widehat{g} - g} \le \delta\sqrt{g} + \delta^2 + c'h \le 2\delta\sqrt{\widehat{g}} + \delta^2 + ch.
\end{equation}
Then when $g> 2\widehat{g}$, we have
\begin{equation}\label{eq:knowing_width_eq2}
\frac{1}{2}g < \abs*{\widehat{g} - g}\le \delta\sqrt{g} + \delta^2 + ch \le (2 + c)\sqrt{gh}.
\end{equation}
which implies $\frac{g^2}{4} < (c + 2)^2gh$ and thus $g < 4(c + 2)^2h \le 4(c + 2)^2\widehat{g}$. Consequently, combining \eqref{eq:knowing_width_eq1} and \eqref{eq:knowing_width_eq2} gives
\[
\abs*{\widehat{g} - g} \le 2(c +2)\delta\sqrt{\widehat{g}} + \delta^2 + ch.
\]
By the same argument applied to $1-g$ and $1-\widehat{g}$ and noting $\abs*{\widehat{g} - g} = \abs*{(1-\widehat{g}) - (1-g)}$, we also have
\[
\abs*{\widehat{g} - g} \le 2(c +2)\delta\sqrt{1-\widehat{g}} + \delta^2 + ch.
\]
Finally, since $\min\bra*{\sqrt{\widehat{g}}, \sqrt{1-\widehat{g}}} \le \sqrt{2\widehat{g}(1-\widehat{g})}$, the desired claim follows.
\end{proof}

\subsubsection{Proof of \cref{lem:monotone_cdf_adjustment_main}}
Finally, we can piece things together.

\begin{proof}{Proof of \cref{lem:monotone_cdf_adjustment_main}}
Based on \cref{lem:hatG_for_G_conf_width}, we have a \textit{computable} confidence width over the regime where $G(b), \widehat{G}^{(0)}_t(b)\in[0,h)\cup (1-h,1]$:
\begin{equation}\label{eq:computable_hob_width}
\abs*{G(b) - \widehat{G}^{(0)}_t(b)} \le c_0\cdot \sqrt{\widehat{G}^{(0)}_t(b)(1-\widehat{G}^{(0)}_t(b))}\delta_t(b) + \delta_t(b)^2 + 3Lh
\end{equation}
with a factor $c_0=O(1)$ when $\delta_t(b) \le \sqrt{h}$, i.e. the error parameter is relatively small. Then we can run \cref{alg:enforce_monotone} to find a monotone estimated CDF $\widehat{G}_t$ that enjoys the same error bound \eqref{eq:computable_hob_width} up to an additional factor of $2$. By \cref{lem:monotone_cdf_adjustment},
\begin{align*}
\abs*{G(b) - \widehat{G}_t(b)} &\le 2c_0\cdot \sqrt{\widehat{G}^{(0)}_t(b)(1-\widehat{G}^{(0)}_t(b))}\delta_t(b) + \delta_t(b)^2 + 3Lh\\
&\le 2c_0\cdot \sqrt{\widehat{G}_t(b)(1-\widehat{G}_t(b))}\delta_t(b) + \delta_t(b)^2 + 3Lh
\end{align*}
where the second line follows from \cref{lem:monotone_cdf_adjustment}. Finally, to recover an error bound in terms of $G(b)$ instead of $\widehat{G}_t(b)$, we apply \cref{lem:hatG_for_G_conf_width} again with $g=\widehat{G}_t(b)$ and $\widehat{g} = G(b)$ to have
\begin{equation}
\abs*{G(b) - \widehat{G}_t(b)} \le 2c_0^2\cdot \sqrt{G(b)(1-G(b))}\delta_t(b) + \delta_t(b)^2 + 3Lh.
\end{equation}
\end{proof}

\subsubsection{Bid Space Truncation to Know Widths}

To satisfy the conditions in \cref{lem:hatG_for_G_conf_width} to compute the error bound, we have the following result.

\begin{repeatlemma}[\cref{lem:truncated_bid_space}]
Consider the bid space $\Bcal_0= [b_{\min}, b_{\max}]$ in \cref{alg:joint_three_init}. With probability at least $1-T^{-2}$, it holds that
\begin{enumerate}
    \item $G(b)\in[Lh, 1-Lh]$ for every $b\in\Bcal_0$;

    \item $G(b)\in[0,25Lh]\cup [1-25Lh, 1]$ for $b\notin \Bcal_0$.
\end{enumerate}
\end{repeatlemma}
\begin{proof}{Proof.}
Suppose the event in \cref{lem:hob_est_error_bound} hold at time $t=JT_0$ i.e. the end of the initialization stage. Recall that in \cref{alg:joint_three_init}, we set $b_{\min} \gets \inf\bra{b\in[0,1]: \widehat{G}_0(b) \ge 9Lh}$. For the sake of contradiction, suppose $G(b_{\min}) < Lh$. Due to the sampling scheme in \cref{alg:joint_three_init}, for every $b\in [0,1]$, $N_t(b) = |\bra{\tau\le t: b_\tau \in I(b)}|\ge 16\log(2T)L^{-1}h^{-1}$, where $I(b)\in \bra{I_j}_{j\in[J]}$ is the interval containing $b$, and hence $\delta_t(b_{\min}) = 8\sqrt{\frac{\log(2T)}{N_t(b_{\min})}}\le 2\sqrt{Lh}$. By \cref{lem:hob_est_error_bound},
\begin{equation*}
\widehat{G}_0(b_{\min}) - G(b_{\min}) \le 2Lh + 4Lh + 2Lh = 8Lh.
\end{equation*}
Meanwhile, we have
\begin{equation*}
8Lh \le \widehat{G}_0(b_{\min}) - Lh < \widehat{G}_0(b_{\min}) - G(b_{\min}) \le 8Lh
\end{equation*}
which gives a contradiction. The other direction follows the same argument.

To prove the second half of the claim, we now upper bound $G(b_{\min})$. By construction, $\widehat{G}_0(b_{\min}) \le 9Lh$. Suppose now $G(b_{\min}) > 9Lh$. By \cref{lem:hob_est_error_bound}, 
\[
G(b_{\min}) - 9Lh \le 2\sqrt{LhG(b_{\min})} + 4Lh + 2Lh.
\]
Straightforward algebra gives $G(b_{\min}) \le 25Lh$. The other direction follows similarly.
\end{proof}

\subsection{Proof of \cref{lem:hob_est_error_bound_monotone}}

\begin{repeatlemma}[\cref{lem:hob_est_error_bound_monotone}]
Let $h=d^{\frac{1}{3}}T^{-\frac{1}{3}}$ and $\bra{I_j}_{j\in[J]}$ be a partition of $[0,1]$ with $|I_j|\in[h,2h]$. Suppose $\Bcal_0$ is given by \cref{alg:joint_three_init}. Let $\widehat{G}_t$ be the CDF estimator given by \cref{alg:hob_est} at time $t$. Then with probability at least $1-2T^{-2}$,
\begin{equation*}
\abs*{G(b) - \widehat{G}_t(b)} \le 
288\delta_t(b)\sqrt{G(b)(1-G(b))} + 2\delta_t(b)^2 + 25Lh, \qquad \text{for every $b\in [0,1]$}
\end{equation*}
where $\delta_t(b)$ is defined as in \cref{lem:hob_est_error_bound}.
\end{repeatlemma}
\begin{proof}{Proof.}
Let $\bra{\widehat{g}_j}$ be defined as in \cref{alg:hob_est} and $\widehat{G}^{(0)}_t(b) = \sum_{j=1}^J\indic[b\in I_j]\widehat{g}_j$. Thanks to the initialization samples and $\Phi\supseteq [Jn_0]$, we have $\delta_t(b) \le \sqrt{2Lh}$ for every $b\in [0,1]$. By a union bound over \cref{lem:hob_est_error_bound} and \ref{lem:truncated_bid_space} and by \cref{lem:hatG_for_G_conf_width}, with probability at least $1-2T^{-2}$, 
\[
G(b) \in [\widehat{g}_j - u_j, \widehat{g}_j+u_j],\qquad \text{for every $j\in[J]$ and $b\in I_j$.}
\]
Then by \cref{lem:monotone_cdf_adjustment}, for $b\in I_j$
\begin{align*}
\abs*{G(b) - \widehat{G}_t(b)} \le 2u_j &\le 24\delta_t(b)\sqrt{\widehat{g}_j(1-\widehat{g}_j)} + 2\delta_t(b)^2 + 25Lh\\
&\le 24\delta_t(b)\sqrt{\widehat{G}_t(b)(1-\widehat{G}_t(b))} + 2\delta_t(b)^2 + 25Lh\\
&\le 288\delta_t(b)\sqrt{G(b)(1-G(b))} + 2\delta_t(b)^2 + 25Lh
\end{align*}
where the last line applies \cref{lem:hatG_for_G_conf_width} again.
\end{proof}

\section{Proof of lower bounds}\label{app:lin_lower_bound}
In this section we prove the regret lower bounds in \cref{thm:lower-bound}. We begin with the analysis of the case when $\bx_t=\varnothing$, or when $\bx_t\equiv x$ for all $t$. 

\begin{theorem}\label{thm:lower_bound_adv}
Consider a special instance with i.i.d. $v_{t,1}\sim \mathrm{Bern}(\mu)$ for an unknown $\mu\in [0,1]$, $v_{t,0}\equiv 0$, and i.i.d. $M_t$ following a known distribution
\begin{align*}
M_t \sim \frac{1}{2}\mathrm{Unif}\left([0,\frac{1}{100}]\right) + \frac{1}{2}\mathrm{Unif}\left([\frac{1}{8}-\frac{1}{200},\frac{1}{8} + \frac{1}{200}]\right). 
\end{align*}
Then for some absolute constant $c>0$, it holds that
\begin{align*}
\inf_{\pi}\sup_{\mu\in [0,1]} \sum_{t=1}^T \left( \max_{b_*\in [0,1]} \E[r_t(b_*)] - \E[r_t(b_t)] \right) \ge c\sqrt{T}. 
\end{align*}
\end{theorem}
\begin{proof}
We use Le Cam's two-point lower bound. Set $\mu_1 = \frac{1}{4}-\Delta$ and $\mu_2 = \frac{1}{4}+\Delta$, with $\Delta := \frac{1}{4\sqrt{T}}$. Let $\E[r_i(b)] = G(b)(\mu_i -b)$ be the expected payoff under the choice $\mu_i$, for $i=1,2$. By straightforward computation, we have
\begin{align*}
\max_{b\in[0,1]}\E[r_1(b)] &= \E[r_1(b_1^*)] = \frac{1}{8}-\frac{1}{200} - \frac{\Delta}{2}, \\
\max_{b\in[0,1]}\E[r_2(b)] &= \E[r_2(b_2^*)] = \frac{1}{8} + \Delta - \frac{1}{200},\\
\max_{b\in[0,1]}\E[r_1(b)] + \E[r_2(b)] &= \E[r_1(b_2^*)] + \E[r_2(b_2^*)] = \frac{1}{4} - \frac{1}{100}, 
\end{align*}
where $b_1^* := \frac{1}{200}$, and $b_2^* = \frac{1}{8}+\frac{1}{100}$. Consequently, for any $b_t\in[0,1]$, we have
\begin{align*}
&\max_{b\in[0,1]}\E[r_1(b) - r_1(b_t)] + \max_{b\in[0,1]}\E[r_2(b) - r_2(b_t)]\\
&\geq \max_{b\in[0,1]}\E[r_1(b)] + \max_{b\in[0,1]}\E[r_2(b)] - \max_{b\in[0,1]}\E[r_1(b)] + \E[r_2(b)]\\
&\geq \frac{\Delta}{2}. \numberthis\label{eq:lecam_separation}
\end{align*}
This shows that no single bid can perform well under both of the environments $\mu_1$ and $\mu_2$. For any fixed policy $\pi$, the environment $\mu_i$ gives rise to a distribution $\mathbb{P}_{i}^{\otimes T}$ over the time horizon $T$ for $i=1,2$. Denote the environment-specific regret by
\[
\Reglte_i(\pi) = \E_{\mathbb{P}_{i}^{\otimes T}}\parq*{\sum_{t=1}^T \max_{b\in[0,1]} \E[r_i(b) - r_i(b_t)]}
\]
for $i=1,2$, where $b_t$ is the bid chosen by policy $\pi$. Let $\nor{\cdot}_{\mathrm{TV}}$ denote the total variation distance. It follows that
\begin{align*}
\Reglte_1(\pi) + \Reglte_2(\pi) &= \sum_{t=1}^{T}\mathbb{P}_{1}^{\otimes t}\parr*{\max_{b\in[0,1]}\E[r_1(b) - r_1(b_t)]} + \mathbb{P}_{2}^{\otimes t}\parr*{\max_{b\in[0,1]}\E[r_2(b) - r_2(b_t)]}\\
&\overset{\eqref{eq:lecam_separation}}{\geq} \frac{\Delta}{2}\sum_{t=1}^{T}\int\min\bra*{\mathrm{d}\mathbb{P}_{1}^{\otimes t}, \mathrm{d}\mathbb{P}_{2}^{\otimes t}}\\
&\stepa{=} \frac{\Delta}{2}\sum_{t=1}^{T} \left(1-\|\mathbb{P}_{1}^{\otimes t}-\mathbb{P}_{2}^{\otimes t}\|_{\mathrm{TV}}\right)\\
&\stepb{\geq} \frac{\Delta T}{2} \parr*{1-\nor*{\mathbb{P}_{1}^{\otimes T}-\mathbb{P}_{2}^{\otimes T}}_{\mathrm{TV}}}
\end{align*}
where (a) uses $\int\min\bra*{\mathrm{d}P, \mathrm{d}Q} = 1-\nor{P-Q}_{\mathrm{TV}}$ and (b) follows from the data processing inequality for the total variation distance. By Lemma~\ref{lem:KL_bounds_TV} and the fact that 
\[
D_{\mathrm{KL}}\parr*{\mathrm{Bern}(p)^{\otimes T}\|\mathrm{Bern}(q)^{\otimes T}} = TD_{\mathrm{KL}}\parr*{\mathrm{Bern}(p)\|\mathrm{Bern}(q)} \leq \frac{T(p-q)^2}{q(1-q)},
\]
we arrive at
\begin{align*}
\Reglte_1(\pi) + \Reglte_2(\pi) &\geq \frac{\Delta T}{4}\exp\parr*{-O(T\Delta^2)}. 
\end{align*}
Finally, plugging in the choice of $\Delta=\frac{1}{4\sqrt{T}}$ leads to
\[
\max_{i=1,2}\Reglte_i(\pi) \geq \frac{\Reglte_1(\pi) + \Reglte_2(\pi)}{2} = \Omega(\sqrt{T}),
\]
establishing the theorem. 
\end{proof}

\subsection{Proof of \cref{thm:lower-bound}}
Note that \cref{thm:lower_bound_adv} holds with $L = O(1)$ in \cref{ass:Gt}. We proceed to show that the special lower bound instance in \cref{thm:lower_bound_adv} can be embedded in the lower bounds of \cref{thm:lower-bound}. Under the special case $v_{t,0}\equiv 0$ and $d\ge 2$, let
\begin{align*}
\btheta_{*} = \left(\frac{1}{2},\mathrm{Unif}\parr*{ \Big\{-2\Delta, 2\Delta \Big\}^{d-1} }\right),
\end{align*}
with $\Delta = \frac{1}{4}\sqrt{\frac{d-1}{T}}$. Since $T\ge d^2$, it holds that $\|\btheta_{*}\|_2\le 1$ almost surely. Divide the time horizon $T$ into $d-1$ sub-horizons $T_i$ for $i=1,\dots,d-1$ with equal length $\frac{T}{d-1}$, and let 
\begin{align*}
    \bx_t = \left(\frac{1}{2}, 0, \dots, 0, \frac{1}{2}, 0, \dots, 0\right)
\end{align*}
during the sub-horizon $T_i$, where the second $\frac{1}{2}$ appears in the $(i+1)$-th entry. As for $M_t$, we again use the i.i.d. distribution in \cref{thm:lower_bound_adv}. Consequently, the regret $\Reglte(\pi)$ can be decomposed as the sum of regrets from $d-1$ independent sub-problems, each of time duration $\frac{T}{d-1}$. By \cref{thm:lower_bound_adv}, we have
\begin{align*}
\inf_{\pi} \Reglte(\pi) = (d-1) \cdot \Omega\parr*{\sqrt{\frac{T}{d-1}}} = \Omega(\sqrt{dT}), 
\end{align*}
as claimed. For the case of $d=1$, we simply use a different construction $\btheta_{*}\sim \mathrm{Unif}(\{\frac{1}{4}-\Delta,\frac{1}{4}+\Delta\})$ and $\bx_t\equiv 1$, so that \cref{thm:lower_bound_adv} again gives a regret lower bound $\Omega(\sqrt{T})$.

\section{Auxiliary Lemmata}
\label{app:aux_lem}

\begin{lemma}[Elliptical potential lemma]\label{lem:sum_width_bound2_FPA2}
For any given vectors $\{\bz_\tau\}_{\tau=1}^{t-1}$ in $\R^d$ with $\|\bz_\tau\|_2\leq 1$, let the Gram matrix be $\bA_s = cI + \sum_{\tau<s}\bz_\tau \bz_\tau^{\top}$ for some $c>0$ and for every $1\leq s \leq t$. It holds that
\[
\sum_{\tau<t}\|\bz_\tau\|^2_{\bA_{\tau}^{-1}} \leq 2d\log\parr*{c+\frac{t-1}{d}} - 2d\log(c).
\]
In particular, by Cauchy-Schwartz inequality,
\[
\sum_{\tau<t}\|\bz_\tau\|_{\bA_{\tau}^{-1}} \leq \sqrt{2d(t-1)\log\parr*{c+\frac{t-1}{d}}} + \sqrt{2(t-1)d|\log(c)|}.
\]
\end{lemma}
\begin{proof}
We decompose the Gram matrix as follows:
\begin{align*}
\det\parr*{\bA_t} &= \det\parr*{\bA_{t-1} + \bz_{t-1} \bz_{t-1}^{\top}} \\
&= \det\parr*{\bA_{t-1}^{1/2}\parr*{\bI + \bA_{t-1}^{-1/2}\bz_{t-1} \bz_{t-1}^{\top}\bA_{t-1}^{-1/2}}\bA_{t-1}^{1/2}}\\
&\overset{\text{(a)}}{=} \det\parr*{\bA_{t-1}}\det\parr*{\bI + \bA_{t-1}^{-1/2}\bz_{t-1} \bz_{t-1}^{\top}\bA_{t-1}^{-1/2}}\\
&\overset{\text{(b)}}{=} \det\parr*{\bA_{t-1}}\parr*{1+\|\bz_{t-1}\|_{\bA_{t-1}^{-1}}^2} \\
&= \det\parr*{\bA_1}\prod_{\tau<t}\parr*{1+\|\bz_\tau\|_{\bA_\tau^{-1}}^2} \numberthis\label{eq:cov_matrix_decomp_1}
\end{align*}
where (a) uses $\det(AB) = \det(A)\det(B)$ and (b) uses that $\bI+vv^T$ only has eigenvalues $1$ and $1+\|v\|_2^2$. Since $\bA_1 = cI$, by AM-GM inequality,
\begin{equation}\label{eq:cov_matrix_decomp_2}
\det\parr*{\bA_{t}} \leq \parr*{\frac{\Tr\parr*{\bA_{t}}}{d}}^d = \parr*{\frac{cd+\sum_{\tau<t} \Tr\parr*{\bz_\tau \bz_\tau^{\top}}}{d}}^d \leq \parr*{c + \frac{t-1}{d}}^d.
\end{equation}
Note that for all $\tau<t$, we have  $\|\bz_\tau\|^2_{\bA_{\tau}^{-1}} \leq \|\bz_\tau\|_2^2\lambda^{-1}_{\min}\parr*{\bA_{\tau}} \leq 1$, and hence $\|\bz_\tau\|^2_{\bA_{\tau}^{-1}} \leq 2\log\parr*{1+\|\bz_\tau\|^2_{\bA_{\tau}^{-1}}}$. Combining \eqref{eq:cov_matrix_decomp_1} and \eqref{eq:cov_matrix_decomp_2}, we have
\begin{align*}
\sum_{\tau<t} \|\bz_\tau\|^2_{\bA_{\tau}^{-1}} &\leq 2\sum_{\tau<t}\log\parr*{1+\|\bz_\tau\|^2_{\bA_{\tau}^{-1}}} \\
&= 2\log\prod_{\tau<t}\parr*{1+\|\bz_\tau\|^2_{\bA_{\tau}^{-1}}} \\
&= 2\log\frac{\det\parr*{\bA_{t}}}{\det\parr*{\bA_1}} \\
&\leq 2d\log\parr*{c+\frac{t-1}{d}} - 2d\log(c).
\end{align*}
\end{proof}

\begin{lemma}\label{lem:approx_CDF_UCB_Lip}
Let $G$ be a CDF on $[0,1]$ with density upper bounded by $L$, and $\widehat{G}$ be another CDF with $\|G-\widehat{G}\|_\infty \le 0.1$. Let $\Bcal\subseteq [0,1]$ and $b_0 = \inf\bra{b\in\Bcal: \widehat{G}(b)\ge \frac{1}{2}}$. Suppose $\exists b\in\Bcal$ with $b\in [b_0- \kappa,b_0)$ for some $\kappa< \frac{1}{30L}$. Denote $\hb_*(v) = \argmax_{b\in\Bcal}\widehat{G}(b)\parr*{v-b}$ with tie broken by taking the maximum. For any $v_1\le v_2$, if $v_2 - v_1\le \frac{1}{15L}$, then
\begin{align*}
\widehat{G}(\widehat{b}_*(v_2)) - \widehat{G}(\widehat{b}_*(v_1)) \le 1 - \frac{1}{15L}. 
\end{align*}
\end{lemma}
\begin{proof}{Proof.}
For simplicity, write $b_1 = \hb_*(v_1)$ and $b_2=\hb_*(v_2)$. By \cref{lem:monotone_ucb_maximizers}, we have $\widehat{G}(b_1)\leq \widehat{G}(b_2)$. For the sake of contradiction, assume $\widehat{G}(b_2)-\widehat{G}(b_1) > 1-\frac{1}{15L}$, which implies $\widehat{G}(b_1)\leq \frac{1}{15L}$ and $\widehat{G}(b_2)> 1-\frac{1}{15L}$.

Define $b_0 = \inf\bra{b\in\Bcal: \widehat{G}(b)\ge \frac{1}{2}}$ and $b'\in\Bcal\cap [b_0-\kappa, b_0)$. We have $\widehat{G}(b_2) > 1-\frac{1}{15L} \geq \widehat{G}(b') + \frac{1}{2}-\frac{1}{15L}$ since $\widehat{G}(b')\le \frac{1}{2}$. Then $b_2 - b_0 > \frac{1}{5L}$, since otherwise by Lipschitzness of $G$ and $\|\widehat{G}-G\|_\infty \le 0.1$, 
\[
\frac{1}{2}-\frac{1}{15L}< \widehat{G}(b_2) - \widehat{G}(b') \le 0.2 + L(b_2-b_0+\kappa) \le \frac{2}{5} + L\kappa
\]
for $\kappa<\frac{1}{30L}$, which is a contradiction.\footnote{Since $G$ is a CDF, it holds that $L\geq 1$.} Since $\widehat{G}(b_0)\ge \frac{1}{2}$, we have the following inequalities:
\begin{align*}
\frac{1}{15L} &\geq \widehat{G}(b_1)(v_1-b_1) 
\geq \widehat{G}(b_0)(v_1-b_0)
\geq \frac{1}{2}\parr*{v_1-b_0}\\
&> \frac{1}{2}\parr*{v_1-b_2+\frac{1}{5L}}
\ge \frac{1}{2}\parr*{v_2-\frac{1}{15L}-b_2+\frac{1}{5L}}
\ge \frac{1}{2}\parr*{\frac{1}{5L}-\frac{1}{15L}}\\
&= \frac{1}{15L}
\end{align*}
which yields a contradiction. Hence the proof is complete.
\end{proof}

\begin{lemma}[Optimal bid is monotone in value]\label{lem:monotone_ucb_maximizers}
Let $G:[0,1]\rightarrow[0,1]$ be any CDF and let $v_1,v_0\in[0,1]$ satisfy $v_1\geq v_0$. Fix any $B\subseteq [0,1]$. Denote $b_*(v) = \argmax_{b\in B}G(b)\parr*{v-b}$ with tie broken by taking the maximum (or minimum). Then $b_*$ is increasing in $v$.
\end{lemma}
\begin{proof}{Proof.}
Let $b_1=b_*(v_1)$ and $b_0=b_*(v_0)$. For the sake of contradiction, suppose $b_0>b_1$. Then
\begin{align*}
G(b_1)\parr*{v_1-b_1} &= G(b_1)\parr*{v_0-b_1} + G(b_1)(v_1-v_0)\\
&\stepa{\leq} G(b_0)\parr*{v_0-b_0} + G(b_0)(v_1-v_0)\\
&= G(b_0)\parr*{v_1-b_0}.
\end{align*}
If tie is broken by taking the maximum, then we have $b_1=b_*(v_1)\ge b_0>b_1$, a contradiction. If tie is broken by minimum, inequality (a) becomes strict because $b_*(v_0)=b_0>b_1$ and hence $G(b_1)\parr*{v_0-b_1} < G(b_0)\parr*{v_0-b_0}$; then we end up with a contradiction again.
\end{proof}


\begin{lemma}\label{lem:truncation}
Let $G$ be a continuous CDF on $[0,1]$, $v\in [0,1]$, and $t\in (0,\frac{1}{2}]$. Let $\widehat{G}$ be another CDF satisfying $\|G-\widehat{G}\|_{\infty} \le \varepsilon$. For $b\in [0,1]$, let 
$$b' = \min\{\max\{b, \widehat{G}^{-1}(t)\}, \widehat{G}^{-1}(1-t)\}.$$ 
Then $t\le \widehat{G}(b') \le 1-t+2\varepsilon$ and 
\begin{align*}
\widehat{G}(b)(v-b) - \widehat{G}(b')(v-b') \le t + 2\varepsilon. 
\end{align*}
\end{lemma}
\begin{proof}{Proof.}
If $\widehat{G}^{-1}(t)<b<\widehat{G}^{-1}(1-t)$, clearly the difference is zero. If $b\le \widehat{G}^{-1}(t)$, then by continuity of $G$ and the bounded error, we have $t \le \widehat{G}(b') \le t + 2\varepsilon$
and hence
\begin{align*}
\widehat{G}(b)(v-b) - \widehat{G}(b')(v-b') = v(\widehat{G}(b)-\widehat{G}(b')) + b'\widehat{G}(b')-b\widehat{G}(b) \le b'\widehat{G}(b') \le t + 2\varepsilon. 
\end{align*}
Finally, if $b\ge \widehat{G}^{-1}(1-t)$, then similarly we have $1-t\le \widehat{G}(b')\le 1-t+2\varepsilon$ and hence
\begin{align*}
\widehat{G}(b)(v-b) - \widehat{G}(b')(v-b') &= v(\widehat{G}(b)-\widehat{G}(b')) - (b\widehat{G}(b)-b'\widehat{G}(b')) \\
&\le \widehat{G}(b) - \widehat{G}(b')\le 1-\widehat{G}(b') \le t. 
\end{align*}
\end{proof}

\begin{lemma}[Regret by Bid Discretization]\label{lem:bounded_regret_by_bid_discrete}
Suppose \cref{ass:Gt} holds. Let $\Bcal=\bra*{\frac{j}{\sqrt{T}}: j\in [\lfloor\sqrt{T}\rfloor]}$ be a discretization of the bid interval $[0,1]$. Then $\br_t(b_t^*) - \max_{b\in\Bcal}\br_t(b) \le \frac{1}{\sqrt{T}}$.
\end{lemma}
\begin{proof}{Proof.}
The proof follows from a straightforward computation. Let $\hb_t^*\in\Bcal$ denote the bid closest to the optimal bid, i.e.
\[
\hb_t^* = \argmin_{b\in\Bcal} |b-b_t^*|.
\]
Then $\abs*{b_t^* - \hb_t^*} \le \frac{1}{\sqrt{T}}$ and 
\begin{align*}
\br_t(b_t^*) - \br_t(\hb_t^*) &= G_t(b_t^*)(\btheta_*^\top \bx_t - b_t^*) - G_t(\hb_t^*)(\btheta_*^\top \bx_t - \hb_t^*)\\
&\le \abs*{G_t(b_t^*) - G_t(\hb_t^*)} + G_t(\hb_t^*)\hb_t^* - G_t(b_t^*)b_t^*\\
&\le \frac{L}{\sqrt{T}} + G_t(\hb_t^*)\hb_t^* - G_t(\hb_t^*)b_t^* + G_t(\hb_t^*)b_t^* - G_t(b_t^*)b_t^*\\
&\le \frac{3L}{\sqrt{T}}
\end{align*}
where the first inequality uses $\abs*{\btheta_*^\top \bx_t}\le 1$ and the second and third inequalities uses the Lipschitzness of $G_t$ by \cref{ass:Gt}, $L\ge 1$, $b_t^*\in[0,1]$, and $\|G_t\|_\infty \le 1$.
\end{proof}

\begin{lemma}[Bernstein's inequality in \citep{boucheron2003concentration}]
\label{lem:bernstein}
Consider independent random variables $X_1,\dots, X_n\in[a,b]$. We have
\[
\mathbb{P}\parr*{\abs*{\sum_{i=1}^nX_i - \sum_{i=1}^n\E[X_i]}\geq \varepsilon} \leq 2\exp\parr*{-\frac{\varepsilon^2}{2(\sigma^2 + \varepsilon(b-a)/3)}}
\]
for any $\varepsilon>0$, where $\sigma^2 = \sum_{i=1}^n\Var(X_i)$.

In particular, it implies the following confidence bound: for any $\delta\in(0,1)$, with probability at least $1-\delta$, we have
    \[
    \frac{1}{n}\left|\sum_{i=1}^nX_i - \sum_{i=1}^n\E[X_i] \right|\leq \sqrt{\frac{2\sigma^2/n\log(2/\delta)}{n}} + \frac{2(b-a)\log(2/\delta)}{3n}.
    \]
\end{lemma}

\begin{lemma}[Freedman's inequality; \citep{freedman1975tail,tropp2011freedman}]\label{lem:freedman}
Let $\bra{Y_t}_{t=1,2,\dots}$ be a real-valued margingale sequence with difference sequence $\bra{X_t}_{t=1,2,\dots}$. Suppose $X_t \le R$ almost surely for all $t=1,2,\dots$. Denote the quadratic variation process as $V_t = \sum_{s\le t} \E_{s}[X_s^2]$ for $t=1,2,\dots$. Then for all $\epsilon>0$ and $\sigma^2\ge 0$, it hold that
\[
\mathbb{P}\parr*{\exists t\ge 0:\, Y_t \ge \epsilon\text{ and }V_t \le \sigma^2} \le \exp\parr*{-\frac{\epsilon^2}{2(\sigma^2 + R\epsilon/3)}}.
\]
\end{lemma}

\begin{lemma}[Bretagnolle--Huber inequality \citep{bretagnolle1978estimation}]\label{lem:KL_bounds_TV}
Let $P,Q$ be two probability measures on the same probability space. Then
\[
1-\nor{P-Q}_{\mathrm{TV}} \geq \frac{1}{2}\exp\parr*{-D_{\mathrm{KL}}(P\|Q) }
\]
where $\nor{\cdot}_{\mathrm{TV}}$ denotes the total variation distance, and $D_{\mathrm{KL}}$ denotes the KL divergence.
\end{lemma}

\begin{lemma}[Linear data splitting]\label{lem:KS_data_splitting}
Let $\bx_1,\dots,\bx_n\in \R^d$ be vectors with $\|\bx_i\|_2\le 1$ for all $i\in [n]$, and $\boldsymbol\Sigma := 18I + \sum_{i=1}^n \bx_i\bx_i^\top$. Then there exists $S\subseteq [n]$ with $|S|\le \frac{n}{2}$ such that
\begin{align*}
18\bI + \sum_{i\in S} \bx_i\bx_i^\top \succeq \frac{\boldsymbol\Sigma}{9}. 
\end{align*}
\end{lemma}
\begin{proof}{Proof.}
The proof uses a deep result in the Kadison--Singer problem. Let $\boldsymbol y_i := \boldsymbol\Sigma^{-1/2}\bx_i$ for every $i\in [n]$, so that $\sum_{i=1}^n \boldsymbol{y}_i\boldsymbol{y}_i^\top \preceq \bI$. In addition, 
\begin{align*}
\|\boldsymbol{y}_i\|_2^2 \le \frac{1}{18}\|\bx_i\|_2^2 \le \frac{1}{18}. 
\end{align*}
By \cite[Corollary 1.5]{marcus2015interlacing} (with $r = 2, \delta = \frac{1}{18}$), there exists a partition $\{S_1, S_2\}$ of $[n]$ such that for $j\in \{1,2\}$, 
\begin{align*}
\sum_{i\in S_j} \boldsymbol{y}_i\boldsymbol{y}_i^\top \preceq \parr*{\frac{1}{\sqrt{r}}+\sqrt{\delta}}^2 \bI = \frac{8}{9}\bI \Longrightarrow \sum_{i\in S_j} \bx_i\bx_i^\top \preceq \frac{8}{9}\boldsymbol\Sigma. 
\end{align*}
WLOG assume that $|S_1|\ge \frac{n}{2}$, then the choice of $S=S_2$ satisfies $|S|\le \frac{n}{2}$, and
\begin{align*}
18\bI + \sum_{i\in S} \bx_i\bx_i^\top = \boldsymbol\Sigma - \sum_{i\in S_1} \bx_i\bx_i^\top \succeq \frac{\Sigma}{9}. 
\end{align*}
This proves the lemma. 
\end{proof}

\end{document}